\documentclass[10pt,a4paper]{article}
\usepackage[utf8]{inputenc}
\usepackage[T1]{fontenc}
\usepackage{mathpazo} 
\usepackage[margin=1in]{geometry}
\usepackage{microtype}
\usepackage{xcolor}
\usepackage{amsmath}
\usepackage{amssymb}
\DeclareUnicodeCharacter{2713}{\checkmark}
\usepackage{amsfonts}
\usepackage{amsthm}

\theoremstyle{plain}
\newtheorem{theorem}{Theorem}[section]

\newtheorem{proposition}[theorem]{Proposition}
\newtheorem{corollary}[theorem]{Corollary}

\theoremstyle{definition}
\newtheorem{definition}[theorem]{Definition}
\newtheorem{example}[theorem]{Example}

\newtheorem{principle}[theorem]{Principle}

\theoremstyle{remark}
\newtheorem{remark}[theorem]{Remark}

\usepackage{mathtools}
\usepackage{stmaryrd}
\usepackage{booktabs}
\usepackage{longtable}
\usepackage{multirow}
\usepackage{enumitem}
\usepackage[numbers,sort&compress]{natbib}
\usepackage{tikz}
\usetikzlibrary{cd,arrows.meta,positioning,calc,decorations.pathmorphing,shapes.geometric,fit,backgrounds}
\usepackage{pgfplots}
\pgfplotsset{compat=1.18}
\usepackage{algorithm}
\usepackage{algpseudocode}

\tikzset{
  spidernode/.style={circle, fill=black, inner sep=2pt, minimum size=4pt},
  spiderleg/.style={thick},
  factornode/.style={rectangle, draw=black, thick, minimum size=8mm, fill=blue!10},
  varnode/.style={circle, draw=black, thick, minimum size=8mm, fill=green!10},
  halfedge/.style={-Stealth, thick, shorten >=1pt},
  message/.style={-Stealth, thick, blue, shorten >=1pt},
  gauge/.style={dashed, thick, red, ->},
  wire/.style={thick},
  cdmap/.style={-Stealth, thick},
  equiv/.style={double, double distance=2pt, -},
}
\setlist[itemize]{leftmargin=1.3em}
\setlist[enumerate]{leftmargin=1.6em}

\usepackage{titlesec}

\usepackage{hyperref}
\hypersetup{
  colorlinks=true,
  linkcolor=blue,
  citecolor=blue,
  urlcolor=blue,
  pdftitle={Categorical Belief Propagation: Sheaf-Theoretic Inference via Descent and Holonomy},
  pdfauthor={Sridhar Mahadevan, Enrique ter Horst, Juan Zambrano},
  pdfsubject={Belief Propagation, Categorical Probability, Effective Descent},
  pdfkeywords={belief propagation, categorical probability, effective descent, simplicial methods, factor graphs, holonomy}
}
\usepackage[capitalise,nameinlink]{cleveref}

\titleformat{\section}
  {\Large\bfseries\color{blue!30!black}}
  {\thesection}{1em}{}
\titleformat{\subsection}
  {\large\bfseries\color{blue!50!black}}
  {\thesubsection}{1em}{}
\titleformat{\subsubsection}
  {\normalsize\bfseries\color{blue!70!black}}
  {\thesubsubsection}{1em}{}

\titlespacing{\section}{0pt}{12pt plus 4pt minus 2pt}{6pt plus 2pt minus 1pt}
\titlespacing{\subsection}{0pt}{10pt plus 3pt minus 2pt}{4pt plus 2pt minus 1pt}

\newcommand{\cat}[1]{\mathbf{#1}}

\newcommand{\Syn}{\cat{Syn}}

\newcommand{\Msg}{\mathsf{Msg}}

\title{\textbf{Categorical Belief Propagation: \\
Sheaf-Theoretic Inference via Descent and Holonomy}}
\author{Enrique ter Horst\thanks{Universidad de los Andes, School of Management, Bogot\'a, Colombia} \and Sridhar Mahadevan\thanks{Adobe Research and University of Massachusetts, Amherst} \and Juan Diego Zambrano\thanks{\url{https://www.colimits.com}}}

\date{\today}

\usepackage{amsfonts}
\newcommand{\Z}{\mathbb{Z}}

\begin{document}
\maketitle

\begin{abstract}
We develop a categorical foundation for belief propagation on factor graphs. We construct the free hypergraph category \(\Syn_\Sigma\) on a typed signature and prove its universal property, yielding compositional semantics via a unique functor to the matrix category \(\cat{Mat}_R\). Message-passing is formulated using a Grothendieck fibration \(\int\Msg \to \cat{FG}_\Sigma\) over polarized factor graphs, with schedule-indexed endomorphisms defining BP updates. We characterize exact inference as effective descent: local beliefs form a descent datum when compatibility conditions hold on overlaps. This framework unifies tree exactness, junction tree algorithms, and loopy BP failures under sheaf-theoretic obstructions. We introduce HATCC (Holonomy-Aware Tree Compilation), an algorithm that detects descent obstructions via holonomy computation on the factor nerve, compiles non-trivial holonomy into mode variables, and reduces to tree BP on an augmented graph. Complexity is \(O(n^2 d_{\max} + c \cdot k_{\max} \cdot \delta_{\max}^3 + n \cdot \delta_{\max}^2)\) for \(n\) factors and \(c\) fundamental cycles. Experimental results demonstrate exact inference with significant speedup over junction trees on grid MRFs and random graphs, along with UNSAT detection on satisfiability instances.

\medskip
\noindent
\textbf{Keywords:} Belief propagation, categorical probability, effective descent, simplicial methods, probabilistic inference, graphical models
\end{abstract}

\section{Introduction}
\label{sec:introduction}

Belief propagation (BP) is routinely deployed as a \emph{numerical inference engine} for graphical models, yet its mathematical status is bifurcated:
(i) the \emph{semantic} object one intends to compute (partition function, marginals, MAP), and
(ii) the \emph{execution} mechanism one actually runs (an iterative, schedule-dependent message update).
In loopy graphs these layers diverge: BP may fail to converge, converge to an incorrect fixed point, or exhibit oscillatory and gauge-dependent behavior even when the target semantics is well-defined \citep{yedidia2005constructing,WainwrightJordan2008,mezard2002analytic}.

This paper develops a categorical foundation that separates \textbf{syntax}, \textbf{semantics}, and \textbf{execution} for BP, and then introduces a deterministic, topologically-informed mechanism that \emph{compiles} loop inconsistency into a finite decomposition of the inference task.
The core technical object is \textbf{holonomy}: the action induced on a chosen ``fiber'' of states by transporting constraints around fundamental cycles.
Holonomy is a global obstruction that is invisible to purely local message updates but can be computed explicitly from the factor graph structure.

\subsection*{Problem statement (Bayesian inference on a factor graph)}

Let $G=(V,F,E)$ be a (bipartite) factor graph, where $V$ is the set of variable nodes,
$F$ is the set of factor nodes, and $E\subseteq V\times F$ encodes incidences (an edge
$(i,f)\in E$ means that the variable $x_i$ participates in factor $f$).
For each variable node $i\in V$, let $\mathcal{X}_i$ denote its (finite) state space.
A global assignment is the tuple
\[
x \;=\; (x_i)_{i\in V}\in \mathcal{X}_V \;:=\; \prod_{i\in V}\mathcal{X}_i.
\]

Each factor node $f\in F$ is associated with a nonnegative potential function
\[
\psi_f:\mathcal{X}_{\mathrm{scope}(f)}\to \mathbb{R}_{\ge 0},
\qquad
\mathrm{scope}(f)\subseteq V,
\]
where $\mathrm{scope}(f)$ is the set of variables adjacent to $f$ in the graph:
$\mathrm{scope}(f)=\{\, i\in V : (i,f)\in E \,\}$, and
$\mathcal{X}_{\mathrm{scope}(f)}:=\prod_{i\in \mathrm{scope}(f)}\mathcal{X}_i$.
The (unnormalized) joint weight of an assignment $x$ is the product of local potentials,
\[
\Psi(x)\;:=\;\prod_{f\in F}\psi_f\!\bigl(x_{\mathrm{scope}(f)}\bigr),
\]
where $x_{\mathrm{scope}(f)}$ denotes the restriction (projection) of the global assignment
$x$ to the coordinates indexed by $\mathrm{scope}(f)$.

The \emph{semantic} inference task is to evaluate the global eliminative computations
\[
Z \;=\; \sum_{x\in \mathcal{X}_V} \Psi(x),
\qquad
p(x_i=a) \;=\; \frac{1}{Z}\sum_{x_{V\setminus\{i\}}\in \mathcal{X}_{V\setminus\{i\}}}
\Psi(x)\Big|_{x_i=a},
\]
for each variable $i\in V$ and each state $a\in \mathcal{X}_i$.
Here $Z$ is the \emph{partition function} (normalizing constant), and the marginal
$p(x_i=a)$ is obtained by summing out all variables except $x_i$ (i.e., eliminating
$x_{V\setminus\{i\}}$). The max--product/MAP analogue replaces the outer sums by maxima.

Standard loopy belief propagation (BP) does not carry out these global eliminations directly.
Instead, it defines an \emph{execution} procedure: a local dynamical system that iteratively
updates messages on the directed incidences of $G$ (typically functions
$m_{i\to f}:\mathcal{X}_i\to \mathbb{R}_{\ge 0}$ and $m_{f\to i}:\mathcal{X}_i\to \mathbb{R}_{\ge 0}$).
Because execution is only an approximation to the semantics on graphs with cycles, the
failure modes of interest are concrete and testable: (i) non-convergence (oscillation or
divergence of message updates), (ii) convergence to an initialization- or schedule-dependent
fixed point, and (iii) converged beliefs whose induced marginals are inconsistent with the
intended global semantics (e.g., violate global cycle constraints or disagree with exact
elimination on tractable subinstances).
\emph{In this work we focus on regimes where nontrivial cycle holonomy is the dominant obstruction
to semantic correctness, and we exploit the induced holonomy sectors (orbits) to decompose inference
into sector-wise exact computations.}

\subsection*{Contributions}

We make five contributions, each phrased to make the separation of layers explicit.

\begin{enumerate}[leftmargin=*]
\item \textbf{Universal syntax for factor-graph inference.}
We present a free, objectwise free hypergraph category $\Syn_{\Sigma}$ generated by a typed signature $\Sigma$, giving a compositional diagram language for factor graphs. This is a \emph{syntax-only} layer: diagrams are programs, not numbers.

\item \textbf{Semantics as semiring-parametric elimination.}
We evaluate $\Syn_{\Sigma}$ in a semantic target $\cat{Mat}_R$ (matrices over a commutative semiring $R$), where composition is elimination (finite ``integration''), and hypergraph structure is encoded by special commutative Frobenius algebras (SCFAs) \citep{fong2019categorical,fritz2020categorical}.

\item \textbf{Execution as a message fibration.}
We model message configurations as fibers of a Grothendieck fibration over a category of polarized factor graphs.
This formalizes (a) how message spaces vary with the graph and (b) the precise equivariance of BP under graph isomorphisms and reindexings.

\item \textbf{Gauge equivariance and simplicial structure.}
We organize standard message rescalings as a groupoid action (``gauge''), and study the resulting simplicial/Kan structure via nerves. This makes invariances and redundancies explicit, and clarifies which properties belong to semantics versus execution \citep{WainwrightJordan2008}.

\item \textbf{Holonomy-aware, deterministic compilation of loopy inference.}
We introduce \emph{Holonomy-Aware Tree Compilation} (HATCC): a deterministic procedure that computes holonomy generators from a cycle basis, extracts the induced orbit/quotient structure on a chosen interface variable, and reduces inference to a finite mixture of \emph{tree} problems (exact per sector). This yields (i) explicit failure certificates for vanilla BP, and (ii) exactness guarantees \emph{within each sector} together with a principled recombination rule.
\end{enumerate}

\subsection*{Positioning against prior work}

BP and its variational interpretations are classical; we rely on the characterization of BP fixed points as stationary points of the Bethe free energy and related formulations \citep{yedidia2005constructing,WainwrightJordan2008,mezard2002analytic}.
Loopy inference is not an edge case: outside of trees, exact marginalization is generally intractable (indeed, probabilistic inference is NP-hard in broad classes of graphical models) \citep{cooper1990complexity,roth1996hardness,KollerFriedman2009}.
As a result, message passing (sum--product / max--product) has become a default large-scale Bayesian workhorse in vision, coding, and statistical physics, with factor-graph formulations providing a unifying computational language \citep{pearl1988probabilistic,lauritzen1988local,kschischang2001factor}.
However, on loopy graphs, standard BP is a \emph{local dynamical system} whose behavior can be fragile: it may fail to converge, converge to initialization-dependent fixed points, or yield beliefs that do not reflect the intended global elimination semantics \citep{weiss2000correctness,ihler2005loopy,heskes2004uniqueness,mooij2007sufficient}.
A substantial literature provides partial remedies---e.g.\ provable convergence under convexified free energies or reweighted objectives \citep{heskes2006convexity,wainwright2003trw}, or post-hoc loop corrections \citep{chertkov2006loopcalculus}---but these approaches typically remain \emph{algorithmic} modifications of local updates.

Gauge invariances of message representations (rescalings that do not change normalized beliefs) are also well known in graphical models and variational inference \citep{WainwrightJordan2008}.
Our contribution is \emph{not} a new local update rule; it is a \emph{global, deterministic compilation step} that extracts topological transport data (cycle holonomy) from the factor graph and uses it to decompose loopy inference into holonomy-induced sectors on which inference is (provably) exact or markedly better behaved.
Conceptually, we replace ``iterate locally and hope'' with ``compute the loop obstruction explicitly, quotient by its action, and solve sector-wise.''

The holonomy mechanism is particularly natural for synchronization and alignment models (e.g.\ $\mathbb{Z}_k$ or $SO(d)$ synchronization), where cycle consistency is the fundamental global constraint \citep{singer2011angular,bandeira2017cheeger}.
It is also compatible with the goals of lifted inference---factorizing inference by exploiting symmetries---but focuses on \emph{topological} symmetries induced by cycle transport rather than purely relational symmetries \citep{poole2003firstorder,van2011lifted}.

\subsection*{Scope, limitations, and what we do \emph{not} claim}

We work primarily with finite domains and semiring-valued potentials, in order to keep semantics exact and algorithmic objects finite.
HATCC is exact for the model class where holonomy induces a finite orbit decomposition on the chosen interface; it is not a general-purpose replacement for junction tree.
We do not claim polynomial-time exact inference in general loopy models.
Instead, we claim that (i) holonomy exposes a concrete and computable global obstruction that explains specific BP failures, and (ii) when the obstruction has small orbit structure, inference can be \emph{deterministically decomposed} into exact subproblems.

\subsection*{Roadmap}

The next section fixes notation and strictness conventions.
We then develop (i) the syntax layer (free hypergraph category), (ii) the semantic layer (evaluation into $\cat{Mat}_R$), and (iii) the execution layer (message fibrations, BP endomorphisms, and gauge).
We subsequently reformulate exactness as a descent condition, and finally introduce holonomy-aware compilation (HATCC) and its sector-wise exactness statement.
The paper closes with experiments, ablations, and a discussion of limitations and open problems.

\section{Background and conventions (condensed)}
\label{sec:background}
We work with discrete factor graphs and a semiring--elimination view of inference.
The categorical and sheaf-theoretic framing in the longer version is used here only as \emph{organizing semantics}: it tells us which objects should be compared (limits/equalizers for consistency) and which computations are canonical (pullbacks/restrictions and pushforwards/eliminations).
All standard categorical preliminaries (hypergraph categories, rig structure, evaluation functors) are moved to Appendix~\ref{app:cat-prelims} with citations to \cite{MacLane1963,fong2019categorical,FongSpivak2019}. 
Similarly, basic sheaf/descent material is deferred to Appendix~\ref{app:sheaf-prelims} with pointers to \cite{GoerssJardine1999,robinson2017sheaf}.

\paragraph{Notation.}
Variables take values in finite sets; for a set of variables $S$ we write $\Omega(S)=\prod_{v\in S}\Omega(v)$.
We use $\odot$ for the semiring product (ordinary multiplication in probabilistic inference) and $\oplus$ for semiring sum (ordinary addition).
When we form \emph{holonomy kernels}, composition is matrix multiplication in the appropriate semiring (Boolean for support; $(+,\times)$ for weighted likelihoods).

\paragraph{What we keep in the main text.}
The novel path---\emph{holonomy} $\rightarrow$ \emph{orbit/sector decomposition} $\rightarrow$ \emph{sector-wise exactness}---is developed in 
\cref{sec:descent,sec:hatcc}.
Background that does not directly support these steps is either cited or moved to appendices.
\section{Belief Propagation as an Endomorphism on Message Fibers}\label{sec:bp}
\label{sec:bp-as-endomorphism}
\label{sec:bpoperator}
\label{sec:trees}

We now define the belief propagation update operator as an endomorphism \(T_G : \cat{Msg}(G) \to \cat{Msg}(G)\) on message spaces. This is the standard sum-product algorithm expressed categorically. The key result is that BP operates on the fibers of the message fibration, and the fixed points of \(T_G\) correspond to marginal beliefs.

\subsection{Factor potentials and the model}

Fix a state space assignment \(\Omega : \Lambda \to \cat{FinSet}\) and a polarized factor graph \(G \in \cat{FG}_{\Sigma}\).

\begin{definition}[Factor potential assignment]
\label{def:factor-potentials}
A \textbf{factor potential assignment} \(\Phi\) for \(G\) assigns to each factor vertex \(f \in F(G)\) with adjacent variables \(\mathsf{adj}(f) = \{v_1, \ldots, v_k\}\) a function:
\[
\phi_f : \Omega(\lambda(v_1)) \times \cdots \times \Omega(\lambda(v_k)) \to R
\]

\textbf{Notation}: Write \(x_{\mathsf{adj}(f)} = (x_{v_1}, \ldots, x_{v_k})\) for a configuration of states at the neighbors of \(f\).
\end{definition}

Interpretation: \(\phi_f\) is the local factor (potential function) associated with factor \(f\). It assigns a weight in \(R\) to each joint configuration of its neighboring variables.

\begin{example}[Pairwise MRF potentials]
For a pairwise factor \(f\) connecting binary variables \(v_1, v_2\) with \(\Omega(\lambda(v_i)) = \{0, 1\}\):
\[
\phi_f(0, 0) = 2.0, \quad \phi_f(0, 1) = 1.0, \quad \phi_f(1, 0) = 1.0, \quad \phi_f(1, 1) = 2.0
\]
This prefers agreement (diagonal entries have higher weight).
\end{example}

\subsection{Message indexing and neighborhoods}

\begin{definition}[Message indexing]
\label{def:message-indexing}
For a message configuration \(m \in \cat{Msg}(G)\) and a half-edge \(h \in \mathsf{HalfEdges}(G)\), write:
\[
m_h \in R^{\Omega(\lambda(\mathsf{var}(h)))}
\]
for the message assigned to half-edge \(h\).

For directed half-edges:
\begin{itemize}
\item If \(h\) is directed \(v \to f\), write \(m_{v \to f}\) for the message from variable \(v\) to factor \(f\)
\item If \(h\) is directed \(f \to v\), write \(m_{f \to v}\) for the message from factor \(f\) to variable \(v\)
\end{itemize}
\end{definition}

\begin{definition}[Neighborhoods]
\label{def:neighborhoods}
For a variable vertex \(v \in V(G)\), define:
\[
\mathsf{nbhd}(v) = \{f \in F(G) \mid f \text{ is adjacent to } v\}
\]

For a factor vertex \(f \in F(G)\), define:
\[
\mathsf{nbhd}(f) = \{v \in V(G) \mid v \text{ is adjacent to } f\}
\]
\end{definition}

\subsection{The belief propagation update rules}

\begin{definition}[Variable-to-factor update]
\label{def:var-to-fac-update}
For a half-edge \(v \to f\) and message configuration \(m \in \cat{Msg}(G)\), the \textbf{variable-to-factor update} is:
\[
(\mathsf{BP}_{v \to f}(m))(x) = \prod_{g \in \mathsf{nbhd}(v) \setminus \{f\}} m_{g \to v}(x)
\]
for each \(x \in \Omega(\lambda(v))\).

Interpretation: The message from variable \(v\) to factor \(f\) is the product of all incoming messages to \(v\) from other factors, evaluated pointwise at each state \(x\).
\end{definition}

Intuition: Variable \(v\) aggregates all information from its neighbors except \(f\), and sends this aggregated belief to \(f\).

\begin{definition}[Factor-to-variable update]
\label{def:fac-to-var-update}
For a half-edge \(f \to v\) and message configuration \(m \in \cat{Msg}(G)\), the \textbf{factor-to-variable update} is:
\[
(\mathsf{BP}_{f \to v}(m))(x_v) = \sum_{x_{\mathsf{nbhd}(f) \setminus \{v\}}} \phi_f(x_{\mathsf{nbhd}(f)}) \cdot \prod_{u \in \mathsf{nbhd}(f) \setminus \{v\}} m_{u \to f}(x_u)
\]
for each \(x_v \in \Omega(\lambda(v))\).

Interpretation: The message from factor \(f\) to variable \(v\) is computed by:
\begin{enumerate}
\item Taking the product of incoming messages from all other neighbors of \(f\)
\item Multiplying by the factor potential \(\phi_f\)
\item Marginalizing (summing) over all variables except \(v\)
\end{enumerate}
\end{definition}

Intuition: Factor \(f\) combines information from all neighbors except \(v\), applies its local constraint \(\phi_f\), and marginalizes out all other variables to send a belief about \(v\) alone.

\begin{example}[Factor update for pairwise MRF]
\label{ex:factor-update}
Consider a pairwise factor \(f\) connecting variables \(v_1, v_2\) with \(\Omega(\lambda(v_i)) = \{0, 1\}\). The message \(f \to v_1\) is:
\begin{align*}
(\mathsf{BP}_{f \to v_1}(m))(x_1) 
&= \sum_{x_2 \in \{0,1\}} \phi_f(x_1, x_2) \cdot m_{v_2 \to f}(x_2) \\
&= \phi_f(x_1, 0) \cdot m_{v_2 \to f}(0) + \phi_f(x_1, 1) \cdot m_{v_2 \to f}(1)
\end{align*}

For \(x_1 = 0\):
\[
(\mathsf{BP}_{f \to v_1}(m))(0) = 2.0 \cdot m_{v_2 \to f}(0) + 1.0 \cdot m_{v_2 \to f}(1)
\]

For \(x_1 = 1\):
\[
(\mathsf{BP}_{f \to v_1}(m))(1) = 1.0 \cdot m_{v_2 \to f}(0) + 2.0 \cdot m_{v_2 \to f}(1)
\]

This is exactly the standard BP update for pairwise factors.
\end{example}

\subsection{The belief propagation operator}

\begin{definition}[Unified local update]
\label{def:unified-update}
For a half-edge \(h \in \mathsf{HalfEdges}(G)\), define:
\[
\mathsf{BP}_h : \cat{Msg}(G) \to R^{\Omega(\lambda(\mathsf{var}(h)))}
\]
by:
\[
\mathsf{BP}_h(m) = \begin{cases}
\mathsf{BP}_{v \to f}(m) & \text{if } h \text{ is directed } v \to f \\
\mathsf{BP}_{f \to v}(m) & \text{if } h \text{ is directed } f \to v
\end{cases}
\]
\end{definition}

\begin{definition}[Parallel (synchronous) belief propagation]
\label{def:parallel-bp}
The \textbf{parallel belief propagation operator} is the endomorphism \(T_G : \cat{Msg}(G) \to \cat{Msg}(G)\) defined by:
\[
(T_G(m))_h = \mathsf{BP}_h(m) \quad \text{for all } h \in \mathsf{HalfEdges}(G)
\]

Interpretation: All messages are updated simultaneously based on the current configuration.
\end{definition}

\begin{definition}[Single-edge (asynchronous) update]
\label{def:async-update}
For a half-edge \(h \in \mathsf{HalfEdges}(G)\), the \textbf{single-edge update} is \(U_h : \cat{Msg}(G) \to \cat{Msg}(G)\) defined by:
\[
(U_h(m))_{h'} = \begin{cases}
\mathsf{BP}_h(m) & \text{if } h' = h \\
m_{h'} & \text{if } h' \neq h
\end{cases}
\]

Interpretation: Update only the message on edge \(h\), leaving all other messages unchanged.
\end{definition}

\subsubsection{Visual representation of belief propagation}

We illustrate the BP update rules and message flow using annotated factor graph diagrams with colored arrows representing messages.

\begin{figure}[ht]
\centering
\begin{tikzpicture}[scale=1.3]
  \node at (0,3.5) {Message space \(\cat{Msg}(G)\) for chain MRF};
  
  \node[varnode] (v1) at (0,1) {\(v_1\)};
  \node[varnode] (v2) at (3,1) {\(v_2\)};
  \node[varnode] (v3) at (6,1) {\(v_3\)};
  
  \node[factornode] (f12) at (1.5,1) {\(f_{12}\)};
  \node[factornode] (f23) at (4.5,1) {\(f_{23}\)};
  
  \draw[thick, gray] (v1) -- (f12);
  \draw[thick, gray] (v2) -- (f12);
  \draw[thick, gray] (v2) -- (f23);
  \draw[thick, gray] (v3) -- (f23);
  
  \draw[message, blue, bend left=15] (v1) to node[above, pos=0.5] {\scriptsize \(m_{v_1 \to f_{12}}\)} (f12);
  \draw[message, red, bend left=15] (f12) to node[below, pos=0.5] {\scriptsize \(m_{f_{12} \to v_1}\)} (v1);
  
  \draw[message, blue, bend left=15] (v2) to node[above, pos=0.3] {\scriptsize \(m_{v_2 \to f_{12}}\)} (f12);
  \draw[message, red, bend left=15] (f12) to node[below, pos=0.3] {\scriptsize \(m_{f_{12} \to v_2}\)} (v2);
  
  \draw[message, blue, bend left=15] (v2) to node[above, pos=0.7] {\scriptsize \(m_{v_2 \to f_{23}}\)} (f23);
  \draw[message, red, bend left=15] (f23) to node[below, pos=0.7] {\scriptsize \(m_{f_{23} \to v_2}\)} (v2);
  
  \draw[message, blue, bend left=15] (v3) to node[above, pos=0.5] {\scriptsize \(m_{v_3 \to f_{23}}\)} (f23);
  \draw[message, red, bend left=15] (f23) to node[below, pos=0.5] {\scriptsize \(m_{f_{23} \to v_3}\)} (v3);
  
  \node[draw, rectangle, fill=yellow!15, align=left] at (3,-0.8) {
    \(|E| = 4\) incidences \\
    \(|E^{\to}| = 8\) half-edges \\
    \(\cat{Msg}(G) = \prod_{h \in E^{\to}} R^{\Omega(\lambda(\mathsf{var}(h)))}\)
  };
\end{tikzpicture}
\caption{Message space \(\cat{Msg}(G)\) for a chain MRF. Each directed half-edge carries a message (function from state space to \(R\)). Blue arrows: variable-to-factor messages. Red arrows: factor-to-variable messages. Belief propagation iteratively updates all messages.}
\label{fig:message-space}
\end{figure}

\begin{figure}[ht]
\centering
\begin{tikzpicture}[scale=1.4]
  \node at (0,4) {\textbf{Variable-to-factor update}: \((T_G(m))_{v \to f}\)};
  
  \node[varnode] (v) at (0,1.5) {\(v\)};
  
  \node[factornode] (f1) at (-1.5,2.8) {\(g_1\)};
  \node[factornode] (f2) at (1.5,2.8) {\(g_2\)};
  \node[factornode] (f3) at (-1.5,0.2) {\(g_3\)};
  \node[factornode] (f_target) at (2,1.5) {\(f\)};
  
  \draw[message, red, -Stealth] (f1) -- node[left, pos=0.6] {\scriptsize \(m_{g_1 \to v}\)} (v);
  \draw[message, red, -Stealth] (f2) -- node[right, pos=0.6] {\scriptsize \(m_{g_2 \to v}\)} (v);
  \draw[message, red, -Stealth] (f3) -- node[left, pos=0.6] {\scriptsize \(m_{g_3 \to v}\)} (v);
  
  \draw[message, blue, very thick, -Stealth] (v) -- node[above] {\((T_G(m))_{v \to f}\)} (f_target);
  
  \node[draw, rectangle, fill=blue!10, align=left] at (0,-1.2) {
    \((T_G(m))_{v \to f}(x) = \prod_{g \in \mathsf{nbhd}(v) \setminus \{f\}} m_{g \to v}(x)\) \\
    \scriptsize Product of all incoming messages except from \(f\)
  };
\end{tikzpicture}
\caption{Variable-to-factor BP update. Variable \(v\) aggregates incoming messages (red arrows) from all neighboring factors except \(f\), multiplying them pointwise. This product becomes the outgoing message \(v \to f\) (blue arrow).}
\label{fig:bp-var-to-fac}
\end{figure}

\begin{figure}[ht]
\centering
\begin{tikzpicture}[scale=1.4]
  \node at (0,4.5) {\textbf{Factor-to-variable update}: \((T_G(m))_{f \to v}\)};
  
  \node[factornode] (f) at (0,2) {\(f\)};
  
  \node[varnode] (v1) at (-2,2.8) {\(u_1\)};
  \node[varnode] (v2) at (2,2.8) {\(u_2\)};
  \node[varnode] (v3) at (-2,1.2) {\(u_3\)};
  \node[varnode] (v_target) at (2.5,2) {\(v\)};
  
  \draw[message, blue, -Stealth] (v1) -- node[above left, pos=0.6] {\scriptsize \(m_{u_1 \to f}\)} (f);
  \draw[message, blue, -Stealth] (v2) -- node[above right, pos=0.6] {\scriptsize \(m_{u_2 \to f}\)} (f);
  \draw[message, blue, -Stealth] (v3) -- node[below left, pos=0.6] {\scriptsize \(m_{u_3 \to f}\)} (f);
  
  \draw[message, red, very thick, -Stealth] (f) -- node[above] {\((T_G(m))_{f \to v}\)} (v_target);
  
  \node[draw, rectangle, fill=red!10, align=left] at (0,-0.5) {
    \scriptsize \((T_G(m))_{f \to v}(x_v) = \sum_{x_{u_1}, x_{u_2}, x_{u_3}} \phi_f(x_{u_1}, x_{u_2}, x_{u_3}, x_v) \prod_{u \in \{u_1,u_2,u_3\}} m_{u \to f}(x_u)\) \\
    \scriptsize Sum-product: multiply factor potential \(\phi_f\) with incoming messages, marginalize over \(u_1, u_2, u_3\)
  };
\end{tikzpicture}
\caption{Factor-to-variable BP update. Factor \(f\) collects incoming messages (blue arrows) from all neighbors except \(v\), multiplies by its potential \(\phi_f\), and marginalizes (sums) over all variables except \(v\). Result is the outgoing message \(f \to v\) (red arrow).}
\label{fig:bp-fac-to-var}
\end{figure}

\begin{figure}[ht]
\centering
\begin{tikzpicture}[scale=1.2]
  \node at (0,5) {\textbf{BP iteration}: \(T_G : \cat{Msg}(G) \to \cat{Msg}(G)\)};
  
  \begin{scope}[shift={(-4,0)}]
    \node[draw, ellipse, minimum width=3cm, minimum height=4cm, fill=blue!5] (msg1) at (0,1.5) {};
    \node at (0,3.5) {\(\cat{Msg}(G)\)};
    \node[circle, fill=red, inner sep=2pt] (m) at (0,1.5) {};
    \node[right] at (m) {\(m\)};
    \node[below, align=center] at (0,0) {Initial \\ message config};
  \end{scope}
  
  \draw[-Stealth, very thick, blue] (-2,2) -- node[above] {\(T_G\)} (2,2);
  
  \begin{scope}[shift={(4,0)}]
    \node[draw, ellipse, minimum width=3cm, minimum height=4cm, fill=green!5] (msg2) at (0,1.5) {};
    \node at (0,3.5) {\(\cat{Msg}(G)\)};
    \node[circle, fill=green!70!black, inner sep=2pt] (m2) at (0,1.5) {};
    \node[right] at (m2) {\(T_G(m)\)};
    \node[below, align=center] at (0,0) {Updated \\ messages};
  \end{scope}
  
  \node[draw, rectangle, fill=yellow!20, align=center] at (0,-1.5) {
    Fixed point: \(T_G(m^*) = m^*\) \\
    \scriptsize Convergence \(\Rightarrow\) approximate marginals
  };
\end{tikzpicture}
\caption{Belief propagation as endomorphism \(T_G : \cat{Msg}(G) \to \cat{Msg}(G)\). Starting from initial messages \(m^{(0)}\), iterate \(m^{(t+1)} = T_G(m^{(t)})\). Fixed points \(m^* = T_G(m^*)\) encode (approximate) marginal beliefs. On trees, BP converges to exact marginals.}
\label{fig:bp-iteration}
\end{figure}
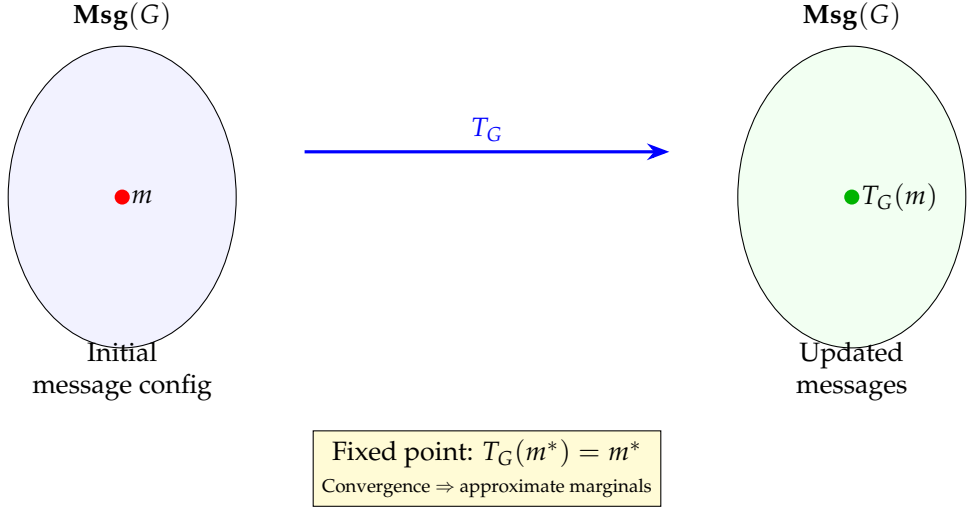

\begin{figure}[ht]
\centering
\begin{tikzpicture}[scale=1.1]
  \node at (0,5.5) {\textbf{Gauge group action} \(K_G \curvearrowright \cat{Msg}(G)\)};
  
  \node[draw, ellipse, minimum width=5cm, minimum height=4.5cm, fill=blue!5] at (0,2) {};
  \node at (0,4.3) {\(\cat{Msg}(G)\)};
  
  \draw[thick, red, dashed] (-1.5,2.5) ellipse (0.8 and 0.4);
  \node[circle, fill=red, inner sep=1.5pt] at (-1.8,2.5) {};
  \node[circle, fill=red, inner sep=1.5pt] at (-1.5,2.6) {};
  \node[circle, fill=red, inner sep=1.5pt] at (-1.2,2.4) {};
  \node[right, red] at (-0.5,2.5) {\scriptsize orbit \([m_1]\)};
  
  \draw[thick, green!70!black, dashed] (1.5,2) ellipse (0.8 and 0.4);
  \node[circle, fill=green!70!black, inner sep=1.5pt] at (1.2,2) {};
  \node[circle, fill=green!70!black, inner sep=1.5pt] at (1.5,2.1) {};
  \node[circle, fill=green!70!black, inner sep=1.5pt] at (1.8,1.9) {};
  \node[right, green!70!black] at (2.4,2) {\scriptsize orbit \([m_2]\)};
  
  \draw[thick, purple, dashed] (0,0.8) ellipse (0.7 and 0.35);
  \node[circle, fill=purple, inner sep=1.5pt] at (-0.2,0.8) {};
  \node[circle, fill=purple, inner sep=1.5pt] at (0,0.9) {};
  \node[circle, fill=purple, inner sep=1.5pt] at (0.2,0.7) {};
  
  \draw[gauge, thick, bend right=20] (-1.8,2.5) to node[below, pos=0.4] {\scriptsize \(k \cdot m\)} (-1.2,2.4);
  
  \draw[-Stealth, ultra thick, blue] (0,-0.8) -- (0,-2) node[midway, right] {quotient \(\pi\)};
  \node[draw, rectangle, fill=orange!20, minimum width=4cm, minimum height=1.2cm] at (0,-2.8) {\(\mathbb{P}\cat{Msg}(G) = \cat{Msg}(G)/K_G\)};
  
  \node[draw, rectangle, fill=yellow!15, align=left] at (5.5,2) {
    Gauge group: \\
    \(K_G = \prod_{h \in E^{\to}} R^\times\) \\
    Action: \((k \cdot m)_h = k_h \cdot m_h\) \\
    \scriptsize Rescales messages
  };
\end{tikzpicture}
\caption{Gauge group \(K_G\) acts on message space \(\cat{Msg}(G)\) by pointwise rescaling. Each orbit (dashed ellipse) represents messages equivalent up to gauge. Projective message space \(\mathbb{P}\cat{Msg}(G)\) is the quotient by this action, identifying physically equivalent message configurations.}
\label{fig:gauge-action-orbits}
\end{figure}
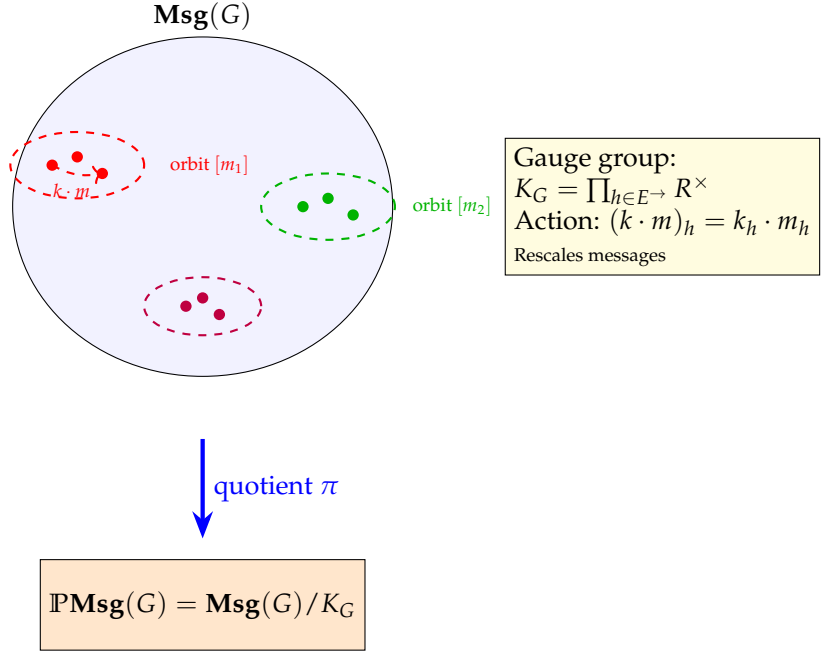

\begin{definition}[Scheduled belief propagation]
\label{def:scheduled-bp}
A \textbf{schedule} on \(G\) is a finite sequence \(s = (h_1, \ldots, h_k)\) of half-edges in \(G\).

The \textbf{scheduled belief propagation operator} is:
\[
T_G^{(s)} = U_{h_k} \circ \cdots \circ U_{h_2} \circ U_{h_1} : \cat{Msg}(G) \to \cat{Msg}(G)
\]

Interpretation: Update messages sequentially according to the schedule, with each update using the most recent message values.
\end{definition}

\begin{proposition}[BP operators are well-defined endomorphisms]
\label{prop:bp-endomorphisms}
The operators \(T_G\), \(U_h\), and \(T_G^{(s)}\) are well-defined functions \(\cat{Msg}(G) \to \cat{Msg}(G)\).

However, these operators are \textbf{not} \(R\)-linear (not semimodule homomorphisms). Instead, they are \textbf{multilinear/polynomial}.
\end{proposition}

\begin{proof}
\textbf{Well-definedness}: Each update is constructed using:
\begin{itemize}
\item Pointwise products in \(R\) (multiplication of messages)
\item Finite sums over state spaces (marginalization)
\item Evaluation of factor potentials
\end{itemize}
All these operations are well-defined in \(R\).

\textbf{Multilinearity (not linearity)}: The variable-to-factor update is:
\[
(\mathsf{BP}_{v \to f}(m))(x) = \prod_{g \in \mathsf{nbhd}(v) \setminus \{f\}} m_{g \to v}(x)
\]

This is \textbf{multilinear} (separately linear in each incoming message \(m_{g \to v}\)), but \textbf{not linear} in the full message vector \(m\).

\textbf{Example}: For degree 2 (two neighbors):
\[
\mathsf{BP}_{v \to f}(m) = m_{g_1 \to v} \cdot m_{g_2 \to v}
\]

This is \textbf{bilinear}: linear in \(m_{g_1 \to v}\) when \(m_{g_2 \to v}\) is fixed, and vice versa. But:
\[
\mathsf{BP}_{v \to f}(m + m') = (m_{g_1} + m'_{g_1}) \cdot (m_{g_2} + m'_{g_2}) \neq \mathsf{BP}_{v \to f}(m) + \mathsf{BP}_{v \to f}(m')
\]

Similarly, factor-to-variable updates involve products under sums, making them polynomial (degree \(\geq 2\)) in the incoming messages.

Conclusion: \(T_G\) is a \textbf{polynomial map} on message spaces, not a linear map. It's multilinear in incoming messages at each node, but this doesn't extend to global linearity.
\qed
\end{proof}

\begin{remark}[Polynomial maps vs linear maps]
BP is a well-defined \textbf{set-theoretic function} \(T_G : \cat{Msg}(G) \to \cat{Msg}(G)\). It's not a semimodule homomorphism (not \(R\)-linear), but it is a \textbf{polynomial map}: compositions of multilinear operations (products) and linear operations (sums).

This distinction is important:
\begin{itemize}
\item \textbf{Linear}: \(f(ax + by) = af(x) + bf(y)\)
\item \textbf{Multilinear}: \(f(x_1, \ldots, x_n)\) is linear in each \(x_i\) separately
\item \textbf{Polynomial}: Compositions of multilinear and linear operations
\end{itemize}

BP falls in the last category: products (multilinear) composed with sums (linear).
\end{remark}

\subsection{Fixed points and beliefs}

\begin{definition}[Fixed point messages]
\label{def:fixed-point}
A message configuration \(m \in \cat{Msg}(G)\) is a \textbf{fixed point} of parallel BP if:
\[
T_G(m) = m
\]

For scheduled BP, \(m\) is a fixed point of schedule \(s\) if:
\[
T_G^{(s)}(m) = m
\]
\end{definition}

\begin{definition}[Marginal beliefs]
\label{def:beliefs}
For a fixed point \(m^* \in \cat{Msg}(G)\) of \(T_G\) and a variable vertex \(v\), the \textbf{belief} at \(v\) is:
\[
b_v(x) = \prod_{f \in \mathsf{nbhd}(v)} m^*_{f \to v}(x) \quad \text{for each } x \in \Omega(\lambda(v))
\]

Interpretation: The belief is the product of all incoming messages—it represents the aggregate information about variable \(v\).
\end{definition}

Key question: When do fixed points exist? When do they correspond to correct marginals? This is answered for trees in Section \ref{sec:trees}.

\section{Gauge Equivariance: How BP Interacts with Message Rescaling}
\label{sec:gauge}
\label{sec:kanbp}

In Section \ref{sec:bpoperator}, we showed that beliefs are gauge-invariant (Proposition \ref{prop:gauge-invariance}), but left open whether the BP operator itself respects gauge rescaling. This section establishes the fundamental relationship between BP and gauge: the operator \(T_G\) is \textbf{semi-equivariant} under a gauge propagation map \(\Theta_G : K_G \to K_G\).

The key result is:
\begin{equation}
\label{eq:bp-gauge-equivariance}
T_G(k \cdot m) = \Theta_G(k) \cdot T_G(m)
\end{equation}

This says: if we rescale the input messages by \(k\), the output messages are rescaled by \(\Theta_G(k)\)—a \emph{different} rescaling determined by how gauge propagates through the BP update rules.

We prove this by explicit calculation, showing how gauge factors propagate through the product and sum operations in the BP formulas.

\subsection{The gauge propagation map}

\begin{definition}[Gauge propagation map]
\label{def:gauge-propagation}
For a polarized factor graph \(G\), define the \textbf{gauge propagation map} \(\Theta_G : K_G \to K_G\) by specifying its action on each half-edge:

\textbf{For variable-to-factor edges} \(v \to f\):
\[
(\Theta_G(k))_{v \to f} = \prod_{g \in \mathsf{nbhd}(v) \setminus \{f\}} k_{g \to v}
\]

\textbf{For factor-to-variable edges} \(f \to v\):
\[
(\Theta_G(k))_{f \to v} = \prod_{u \in \mathsf{nbhd}(f) \setminus \{v\}} k_{u \to f}
\]

Interpretation: When we rescale input messages by gauge \(k\), the BP operator produces output messages rescaled by \(\Theta_G(k)\). The rescaling on each output edge is the \emph{product} of rescalings on the input edges used to compute that output.
\end{definition}

\begin{example}[Gauge propagation on a simple edge]
Consider a simple two-variable tree:
\[
\begin{tikzcd}
v_1 \arrow[r, dash] & f \arrow[r, dash] & v_2
\end{tikzcd}
\]

Suppose we rescale the message \(m_{v_2 \to f}\) by gauge factor \(k_{v_2 \to f} = 2.0\). What happens to the output message \(m_{f \to v_1}\)?

By the gauge propagation formula:
\[
(\Theta_G(k))_{f \to v_1} = k_{v_2 \to f} = 2.0
\]

So if we double the input message, the output message also doubles. This makes sense: the factor-to-variable update is:
\[
m_{f \to v_1}(x_1) = \sum_{x_2} \phi_f(x_1, x_2) \cdot m_{v_2 \to f}(x_2)
\]

Doubling \(m_{v_2 \to f}\) doubles the entire sum.
\end{example}

\subsection{Main theorem: BP is semi-equivariant}

\begin{theorem}[BP semi-equivariance under gauge]
\label{thm:bp-semi-equivariant}
\label{thm:gauge-action}
For any polarized factor graph \(G\), gauge \(k \in K_G\), and message configuration \(m \in \cat{Msg}(G)\):
\[
T_G(k \cdot m) = \Theta_G(k) \cdot T_G(m)
\]

Moreover, \(\Theta_G : K_G \to K_G\) is a \textbf{group homomorphism}:
\begin{enumerate}
\item \(\Theta_G(\mathbf{1}) = \mathbf{1}\) (identity preservation)
\item \(\Theta_G(k \cdot k') = \Theta_G(k) \cdot \Theta_G(k')\) (multiplicativity)
\end{enumerate}
\end{theorem}

\begin{proof}
We prove the equivariance by case analysis on the BP update formulas.

\textbf{Case 1: Variable-to-factor update}

Let \(h = (v \to f)\). By Definition \ref{def:var-to-fac-update}:
\begin{align*}
(T_G(k \cdot m))_{v \to f}(x) 
&= \prod_{g \in \mathsf{nbhd}(v) \setminus \{f\}} (k \cdot m)_{g \to v}(x) \\
&= \prod_{g \in \mathsf{nbhd}(v) \setminus \{f\}} k_{g \to v} \cdot m_{g \to v}(x) \quad \text{(by definition of gauge action)} \\
&= \left(\prod_{g \in \mathsf{nbhd}(v) \setminus \{f\}} k_{g \to v}\right) \cdot \left(\prod_{g \in \mathsf{nbhd}(v) \setminus \{f\}} m_{g \to v}(x)\right) \\
&= (\Theta_G(k))_{v \to f} \cdot (T_G(m))_{v \to f}(x) \quad \text{(by Def. \ref{def:gauge-propagation})} \\
&= (\Theta_G(k) \cdot T_G(m))_{v \to f}(x)
\end{align*}

Key observation: The gauge factors \(k_{g \to v}\) are \emph{independent of \(x\)}, so they factor out of the product. This is why the gauge action is multiplicative on this component.

\textbf{Case 2: Factor-to-variable update}

Let \(h = (f \to v)\). By Definition~\ref{def:fac-to-var-update}:
\begin{align*}
(T_G(k \cdot m))_{f \to v}(x_v)
&= \sum_{x_{\mathsf{nbhd}(f)\setminus\{v\}}}
\phi_f(x_{\mathsf{nbhd}(f)}) \cdot
\prod_{u \in \mathsf{nbhd}(f)\setminus\{v\}} (k\cdot m)_{u \to f}(x_u) \\
&= \sum_{x_{\mathsf{nbhd}(f)\setminus\{v\}}}
\phi_f(x_{\mathsf{nbhd}(f)}) \cdot
\prod_{u \in \mathsf{nbhd}(f)\setminus\{v\}} \bigl(k_{u \to f}\cdot m_{u \to f}(x_u)\bigr) \\
&= \sum_{x_{\mathsf{nbhd}(f)\setminus\{v\}}}
\phi_f(x_{\mathsf{nbhd}(f)}) \cdot
\left(\prod_{u \in \mathsf{nbhd}(f)\setminus\{v\}} k_{u \to f}\right)
\cdot
\left(\prod_{u \in \mathsf{nbhd}(f)\setminus\{v\}} m_{u \to f}(x_u)\right).
\end{align*}
Since each \(k_{u\to f}\in R^\times\) is a scalar independent of the summation variables, it factors out of the sum:
\begin{align*}
(T_G(k \cdot m))_{f \to v}(x_v)
&=
\left(\prod_{u \in \mathsf{nbhd}(f)\setminus\{v\}} k_{u \to f}\right)
\cdot
\sum_{x_{\mathsf{nbhd}(f)\setminus\{v\}}}
\phi_f(x_{\mathsf{nbhd}(f)}) \cdot
\prod_{u \in \mathsf{nbhd}(f)\setminus\{v\}} m_{u \to f}(x_u) \\
&=
(\Theta_G(k))_{f \to v}\cdot (T_G(m))_{f \to v}(x_v),
\end{align*}
where the last equality is Definition~\ref{def:gauge-propagation}. This proves
\(T_G(k\cdot m)=\Theta_G(k)\cdot T_G(m)\) on all half-edges.

\textbf{Homomorphism property of \(\Theta_G\).}
Let \(k,k'\in K_G\). For a variable-to-factor half-edge \(v\to f\),
\[
(\Theta_G(kk'))_{v\to f}
= \prod_{g\in \mathsf{nbhd}(v)\setminus\{f\}} (kk')_{g\to v}
= \prod_{g} (k_{g\to v}k'_{g\to v})
= \left(\prod_g k_{g\to v}\right)\left(\prod_g k'_{g\to v}\right)
= (\Theta_G(k))_{v\to f}(\Theta_G(k'))_{v\to f}.
\]
The factor-to-variable case \(f\to v\) is identical, replacing neighborhoods accordingly. Hence
\(\Theta_G(kk')=\Theta_G(k)\Theta_G(k')\), and clearly \(\Theta_G(\mathbf 1)=\mathbf 1\).
\qed
\end{proof}

\begin{corollary}[BP descends to projective messages]
If \(m'\sim m\) in \(\mathbb{P}\cat{Msg}(G)=\cat{Msg}(G)/K_G\), i.e. \(m'=k\cdot m\) for some \(k\in K_G\),
then \(T_G(m')\sim T_G(m)\). Hence \(\overline T_G([m]) := [T_G(m)]\) is well-defined.
\end{corollary}

\begin{proof}
If \(m'=k\cdot m\), then by Theorem~\ref{thm:bp-semi-equivariant},
\(T_G(m')=T_G(k\cdot m)=\Theta_G(k)\cdot T_G(m)\), so \(T_G(m')\) lies in the same \(K_G\)-orbit as \(T_G(m)\).
\end{proof}

\subsubsection{Visual representation of gauge symmetry}

We illustrate the gauge group action, propagation map \(\Theta_G\), and semi-equivariance using commutative diagrams and annotated factor graphs.

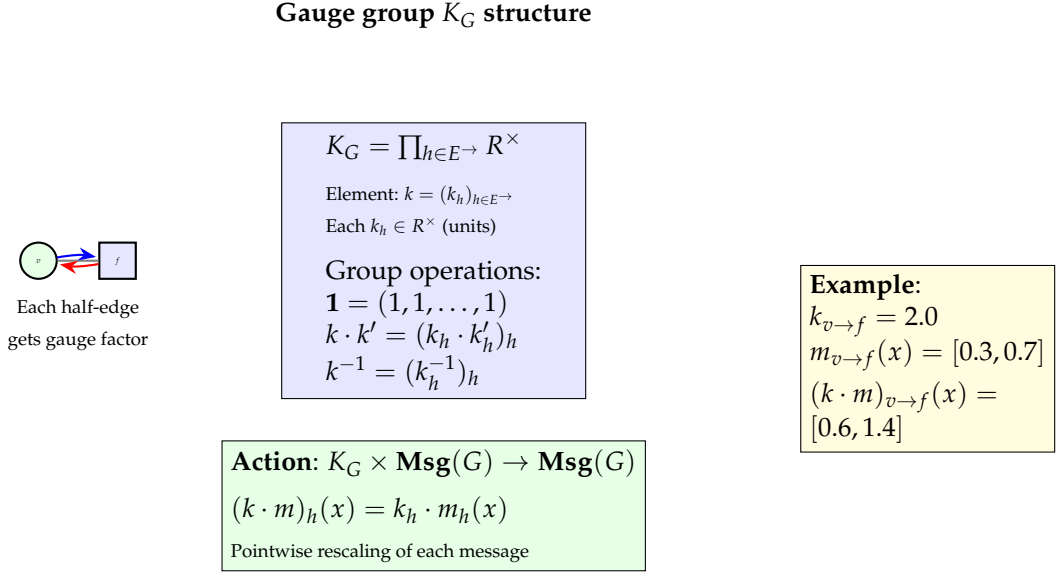
\begin{figure}[ht]
\centering
\begin{tikzpicture}[scale=1.3]
  \node at (0,5.5) {\textbf{Gauge group} \(K_G\) \textbf{structure}};
  
  \begin{scope}[shift={(-4,3)}]
    \node[varnode, scale=0.6] (v1) at (0,0) {\tiny \(v\)};
    \node[factornode, scale=0.6] (f1) at (0.8,0) {\tiny \(f\)};
    \draw[thick, gray] (v1) -- (f1);
    \draw[halfedge, blue, bend left=10] (v1) to (f1);
    \draw[halfedge, red, bend left=10] (f1) to (v1);
    \node[below, align=center] at (0.4,-0.3) {\scriptsize Each half-edge \\ \scriptsize gets gauge factor};
  \end{scope}
  
  \node[draw, rectangle, fill=blue!10, align=left, minimum width=4cm, minimum height=2cm] at (0,3) {
    \(K_G = \prod_{h \in E^{\to}} R^\times\) \\[0.2cm]
    \scriptsize Element: \(k = (k_h)_{h \in E^{\to}}\) \\
    \scriptsize Each \(k_h \in R^\times\) (units) \\[0.2cm]
    Group operations: \\
    \(\mathbf{1} = (1,1,\ldots,1)\) \\
    \(k \cdot k' = (k_h \cdot k'_h)_{h}\) \\
    \(k^{-1} = (k_h^{-1})_{h}\)
  };
  
  \node[draw, rectangle, fill=green!10, align=left, minimum width=4cm, minimum height=1.5cm] at (0,0.5) {
    \textbf{Action}: \(K_G \times \cat{Msg}(G) \to \cat{Msg}(G)\) \\[0.2cm]
    \((k \cdot m)_h(x) = k_h \cdot m_h(x)\) \\[0.1cm]
    \scriptsize Pointwise rescaling of each message
  };
  
  \node[draw, rectangle, fill=yellow!15, align=left] at (5,2) {
    \textbf{Example}: \\
    \(k_{v \to f} = 2.0\) \\
    \(m_{v \to f}(x) = [0.3, 0.7]\) \\[0.1cm]
    \((k \cdot m)_{v \to f}(x) =\) \\
    \([0.6, 1.4]\)
  };
\end{tikzpicture}
\caption{Gauge group \(K_G\) structure: product of units \(R^\times\) over all half-edges. Action rescales messages componentwise. This is a free abelian group action preserving the message space structure.}
\label{fig:gauge-group-structure}
\end{figure}

\begin{figure}[ht]
\centering
\begin{tikzpicture}[scale=1.2]
  \node at (0,6) {\textbf{Gauge propagation} \(\Theta_G : K_G \to K_G\)};
  
  \begin{scope}[shift={(-4.5,2.5)}]
    \node at (0,2.5) {Variable \(\to\) factor};
    \node[varnode] (v) at (0,1) {\(v\)};
    \node[factornode] (f_target) at (2,1) {\(f\)};
    \node[factornode, scale=0.7] (g1) at (-1,2) {\tiny \(g_1\)};
    \node[factornode, scale=0.7] (g2) at (1,2) {\tiny \(g_2\)};
    
    \draw[gauge, bend left=10] (g1) to node[left, pos=0.7] {\scriptsize \(k_{g_1 \to v}\)} (v);
    \draw[gauge, bend right=10] (g2) to node[right, pos=0.7] {\scriptsize \(k_{g_2 \to v}\)} (v);
    
    \draw[gauge, very thick] (v) to node[above] {\((\Theta_G(k))_{v \to f}\)} (f_target);
    
    \node[below, align=center] at (1,-0.2) {
      \scriptsize \((\Theta_G(k))_{v \to f} = k_{g_1 \to v} \cdot k_{g_2 \to v}\)
    };
  \end{scope}
  
  \begin{scope}[shift={(4.5,2.5)}]
    \node at (0,2.5) {Factor \(\to\) variable};
    \node[factornode] (f) at (0,1) {\(f\)};
    \node[varnode] (v_target) at (2.5,1) {\(v\)};
    \node[varnode, scale=0.7] (u1) at (-1,2) {\tiny \(u_1\)};
    \node[varnode, scale=0.7] (u2) at (1,2) {\tiny \(u_2\)};
    
    \draw[gauge, bend left=10] (u1) to node[left, pos=0.7] {\scriptsize \(k_{u_1 \to f}\)} (f);
    \draw[gauge, bend right=10] (u2) to node[right, pos=0.7] {\scriptsize \(k_{u_2 \to f}\)} (f);
    
    \draw[gauge, very thick] (f) to node[above] {\((\Theta_G(k))_{f \to v}\)} (v_target);
    
    \node[below, align=center] at (1.25,-0.2) {
      \scriptsize \((\Theta_G(k))_{f \to v} = k_{u_1 \to f} \cdot k_{u_2 \to f}\)
    };
  \end{scope}
  
  \node[draw, rectangle, fill=purple!10, align=center] at (0,0) {
    \(\Theta_G\) is a group homomorphism: \\
    \(\Theta_G(\mathbf{1}) = \mathbf{1}\), \quad \(\Theta_G(k \cdot k') = \Theta_G(k) \cdot \Theta_G(k')\)
  };
\end{tikzpicture}
\caption{Gauge propagation map \(\Theta_G : K_G \to K_G\). Output gauge on each half-edge is the product of input gauges. Left: variable-to-factor propagation. Right: factor-to-variable propagation. \(\Theta_G\) is a group homomorphism governing how rescalings propagate through BP.}
\label{fig:gauge-propagation-map}
\end{figure}

\begin{figure}[ht]
\centering
\begin{tikzpicture}[scale=1.0]
  \node at (0,5.5) {\textbf{Semi-equivariance}: \(T_G(k \cdot m) = \Theta_G(k) \cdot T_G(m)\)};
  
  \node[draw, ellipse, minimum width=3cm, minimum height=2cm, fill=blue!5] (msg1) at (-4,2.5) {};
  \node at (-4,3.5) {\(\cat{Msg}(G)\)};
  \node[circle, fill=red, inner sep=2pt] (m) at (-4,2.5) {};
  \node[right] at (m) {\(m\)};
  
  \node[draw, ellipse, minimum width=3cm, minimum height=2cm, fill=blue!5] (msg2) at (4,2.5) {};
  \node at (4,3.5) {\(\cat{Msg}(G)\)};
  \node[circle, fill=green!70!black, inner sep=2pt] (km) at (4,2.5) {};
  \node[right] at (km) {\(k \cdot m\)};
  
  \node[draw, ellipse, minimum width=3cm, minimum height=2cm, fill=green!5] (msg3) at (-4,-1) {};
  \node at (-4,0) {\(\cat{Msg}(G)\)};
  \node[circle, fill=blue, inner sep=2pt] (Tm) at (-4,-1) {};
  \node[right] at (Tm) {\(T_G(m)\)};
  
  \node[draw, ellipse, minimum width=3cm, minimum height=2cm, fill=green!5] (msg4) at (4,-1) {};
  \node at (4,0) {\(\cat{Msg}(G)\)};
  \node[circle, fill=purple, inner sep=2pt] (TGkm) at (4,-1) {};
  \node[right] at (TGkm) {\(T_G(k \cdot m)\)};
  
  \draw[-Stealth, very thick, red] (m) -- node[above] {gauge \(k\)} (km);
  \draw[-Stealth, very thick, blue] (m) -- node[left] {\(T_G\)} (Tm);
  \draw[-Stealth, very thick, blue] (km) -- node[right] {\(T_G\)} (TGkm);
  \draw[-Stealth, very thick, red] (Tm) -- node[below] {gauge \(\Theta_G(k)\)} (TGkm);
  
  \node at (0,0.75) {\Huge \textcolor{red}{=}};
  
  \node[draw, rectangle, fill=yellow!15, align=center] at (0,-3) {
    Diagram commutes: \\
    Rescale-then-update \(=\) Update-then-rescale (by \(\Theta_G(k)\))
  };
\end{tikzpicture}
\caption{Semi-equivariance commutative square: \(T_G(k \cdot m) = \Theta_G(k) \cdot T_G(m)\). Top path: rescale input by \(k\), then apply BP. Bottom path: apply BP, then rescale output by \(\Theta_G(k)\). Both paths yield the same result. This is the fundamental gauge symmetry of belief propagation.}
\label{fig:semi-equivariance-diagram}
\end{figure}
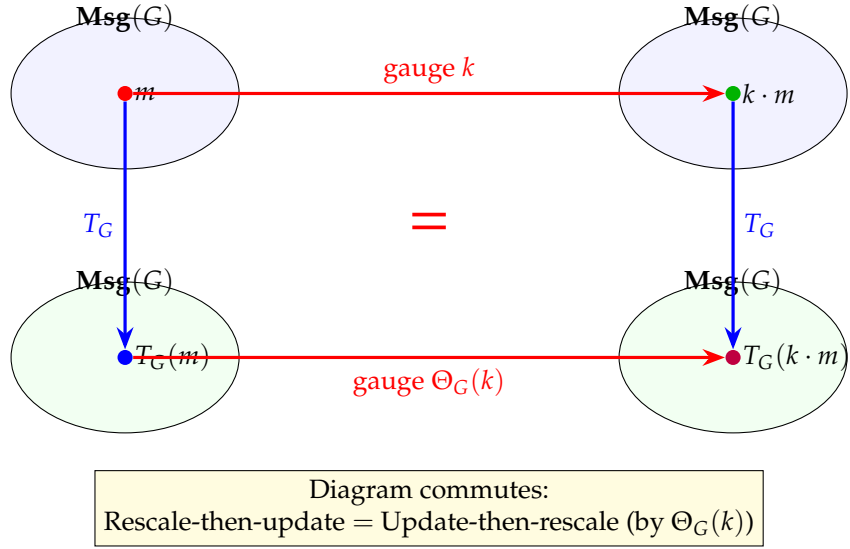

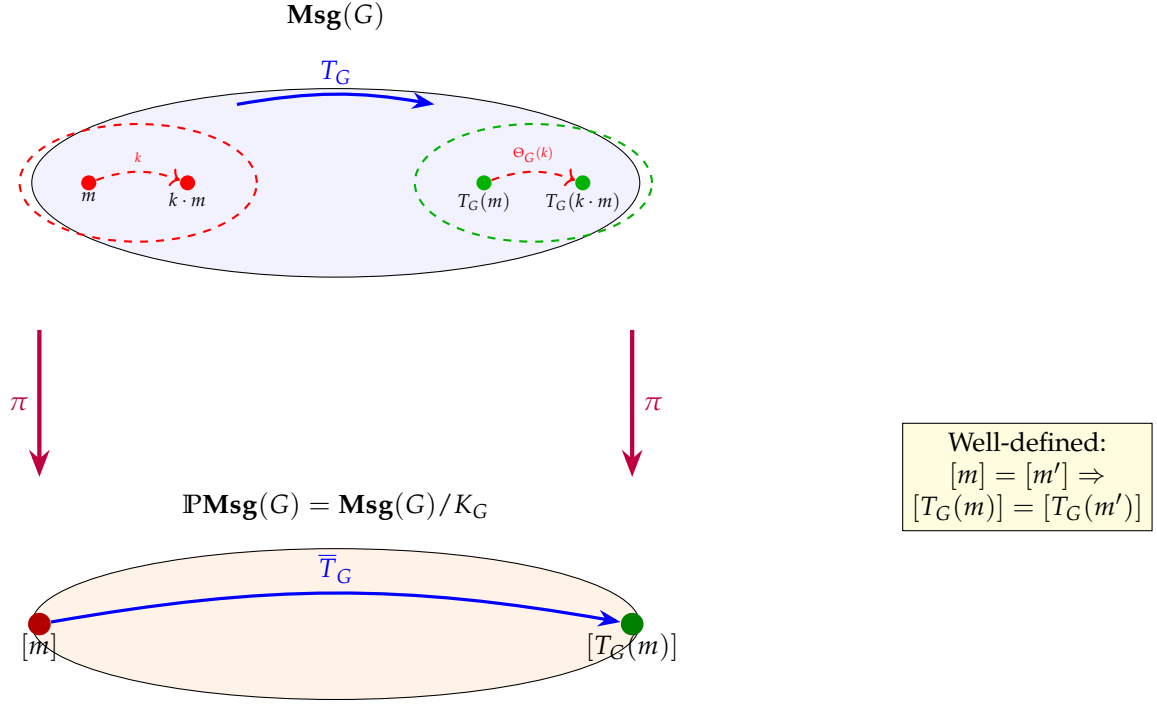
\begin{figure}[ht]
\centering
\begin{tikzpicture}[scale=1.3]
  \node at (0,5.5) {\textbf{Descent to projective space}: \(\overline{T}_G : \mathbb{P}\cat{Msg}(G) \to \mathbb{P}\cat{Msg}(G)\)};
  
  \begin{scope}[shift={(0,2.5)}]
    \node[draw, ellipse, minimum width=8cm, minimum height=2.5cm, fill=blue!5] at (0,0.5) {};
    \node at (0,2.2) {\(\cat{Msg}(G)\)};
    
    \draw[thick, red, dashed] (-2,0.5) ellipse (1.2 and 0.6);
    \node[circle, fill=red, inner sep=2pt] (m1) at (-2.5,0.5) {};
    \node[below] at (m1) {\scriptsize \(m\)};
    \node[circle, fill=red, inner sep=2pt] (m2) at (-1.5,0.5) {};
    \node[below] at (m2) {\scriptsize \(k \cdot m\)};
    \draw[gauge, bend left=20] (m1) to node[above] {\tiny \(k\)} (m2);
    
    \draw[thick, green!70!black, dashed] (2,0.5) ellipse (1.2 and 0.6);
    \node[circle, fill=green!70!black, inner sep=2pt] (Tm1) at (1.5,0.5) {};
    \node[below] at (Tm1) {\scriptsize \(T_G(m)\)};
    \node[circle, fill=green!70!black, inner sep=2pt] (Tm2) at (2.5,0.5) {};
    \node[below] at (Tm2) {\scriptsize \(T_G(k \cdot m)\)};
    \draw[gauge, bend left=20] (Tm1) to node[above] {\tiny \(\Theta_G(k)\)} (Tm2);
    
    \draw[-Stealth, very thick, blue] (-1,1.3) to[bend left=10] node[above] {\(T_G\)} (1,1.3);
  \end{scope}
  
  \draw[-Stealth, ultra thick, purple] (-3,1.5) -- (-3,0) node[midway, left] {\(\pi\)};
  \draw[-Stealth, ultra thick, purple] (3,1.5) -- (3,0) node[midway, right] {\(\pi\)};
  
  \begin{scope}[shift={(0,-1.5)}]
    \node[draw, ellipse, minimum width=8cm, minimum height=2cm, fill=orange!10] at (0,0) {};
    \node at (0,1.2) {\(\mathbb{P}\cat{Msg}(G) = \cat{Msg}(G)/K_G\)};
    
    \node[circle, fill=red!70!black, inner sep=3pt] (class1) at (-3,0) {};
    \node[below] at (class1) {\([m]\)};
    
    \node[circle, fill=green!50!black, inner sep=3pt] (class2) at (3,0) {};
    \node[below] at (class2) {\([T_G(m)]\)};
    
    \draw[-Stealth, very thick, blue] (class1) to[bend left=10] node[above] {\(\overline{T}_G\)} (class2);
  \end{scope}
  
  \node[draw, rectangle, fill=yellow!15, align=center] at (7,0) {
    Well-defined: \\
    \([m] = [m'] \Rightarrow\) \\
    \([T_G(m)] = [T_G(m')]\)
  };
\end{tikzpicture}
\caption{Descent of \(T_G\) to projective space \(\mathbb{P}\cat{Msg}(G)\). Top: \(T_G\) maps gauge orbits to gauge orbits (semi-equivariance). Bottom: induced map \(\overline{T}_G\) on equivalence classes. Quotient map \(\pi\) makes the diagram commute. Projective BP \(\overline{T}_G\) is gauge-independent and represents the physically meaningful dynamics.}
\label{fig:projective-bp-descent}
\end{figure}

\begin{figure}[ht]
\centering
\begin{tikzpicture}[scale=1.4]
  \node at (0,5.5) {\textbf{Worked example}: Gauge propagation on 3-variable chain};
  
  \node[varnode] (v1) at (0,3) {\(v_1\)};
  \node[factornode] (f12) at (2,3) {\(f_{12}\)};
  \node[varnode] (v2) at (4,3) {\(v_2\)};
  \node[factornode] (f23) at (6,3) {\(f_{23}\)};
  \node[varnode] (v3) at (8,3) {\(v_3\)};
  
  \draw[thick, gray] (v1) -- (f12) -- (v2) -- (f23) -- (v3);
  
  \draw[gauge, bend left=15] (f12) to node[above] {\scriptsize \(k_{f_{12} \to v_1} = 2\)} (v1);
  \draw[gauge, bend left=15] (v1) to node[below] {\scriptsize \(k_{v_1 \to f_{12}} = 3\)} (f12);
  \draw[gauge, bend left=15] (f12) to node[above] {\scriptsize \(k_{f_{12} \to v_2} = 5\)} (v2);
  \draw[gauge, bend left=15] (v2) to node[below] {\scriptsize \(k_{v_2 \to f_{12}} = 7\)} (f12);
  
  \draw[gauge, very thick, bend left=15] (v2) to node[above] {\((\Theta_G(k))_{v_2 \to f_{23}}\)} (f23);
  \draw[gauge, very thick, bend left=15] (f23) to node[below] {\((\Theta_G(k))_{f_{23} \to v_2}\)} (v2);
  
  \node[draw, rectangle, fill=blue!10, align=left] at (4,0.5) {
    \textbf{Variable \(\to\) factor}: \\
    \((\Theta_G(k))_{v_2 \to f_{23}} = \prod_{g \in \mathsf{nbhd}(v_2) \setminus \{f_{23}\}} k_{g \to v_2}\) \\
    \(= k_{f_{12} \to v_2} = 5\) \\[0.3cm]
    \textbf{Factor \(\to\) variable}: \\
    \((\Theta_G(k))_{f_{23} \to v_2} = \prod_{u \in \mathsf{nbhd}(f_{23}) \setminus \{v_2\}} k_{u \to f_{23}}\) \\
    \(= k_{v_3 \to f_{23}}\) \scriptsize (depends on \(v_3\) gauge)
  };
  
  \node[draw, rectangle, fill=yellow!15, align=center] at (4,-2) {
    Key insight: Gauge flows linearly along chains \\
    Output gauge = Input gauge from opposite side
  };
\end{tikzpicture}
\caption{Explicit gauge propagation on a 3-variable chain. Input gauges \(k\) (dashed red) determine output gauges \(\Theta_G(k)\) (thick red) via product formulas. On chains, gauge propagates linearly: the gauge on \(v_2 \to f_{23}\) equals the gauge on \(f_{12} \to v_2\), creating a flow from left to right.}
\label{fig:gauge-chain-example}
\end{figure}

\subsection{Worked example: Gauge propagation on a chain}

\begin{example}[Gauge rescaling on a 3-variable chain]
\label{ex:gauge-chain}
Consider the chain:
\[
v_1 \leftrightarrow f_{12} \leftrightarrow v_2 \leftrightarrow f_{23} \leftrightarrow v_3
\]

Suppose we rescale messages by gauge \(k \in K_G\) with:
\[
k_{f_{12} \to v_1} = 2, \quad k_{v_1 \to f_{12}} = 3, \quad k_{f_{12} \to v_2} = 5, \quad \text{etc.}
\]

What is \(\Theta_G(k)\)?

\textbf{At edge \(v_2 \to f_{12}\)}:
\[
(\Theta_G(k))_{v_2 \to f_{12}} = \prod_{g \in \mathsf{nbhd}(v_2) \setminus \{f_{12}\}} k_{g \to v_2} = k_{f_{23} \to v_2}
\]

Only \(f_{23}\) is a neighbor of \(v_2\) other than \(f_{12}\), so the gauge on \(v_2 \to f_{12}\) is determined by the incoming gauge from \(f_{23}\).

\textbf{At edge \(f_{12} \to v_2\)}:
\[
(\Theta_G(k))_{f_{12} \to v_2} = \prod_{u \in \mathsf{nbhd}(f_{12}) \setminus \{v_2\}} k_{u \to f_{12}} = k_{v_1 \to f_{12}} = 3
\]

The gauge on the output from \(f_{12}\) to \(v_2\) depends on the gauge of the input from \(v_1\).

Interpretation: Gauge propagates along the graph: the rescaling on an output edge depends on the rescalings on the input edges. On a chain, gauge flows linearly from one end to the other.
\end{example}

\subsection{Consequence: Projective BP is well-defined}

\begin{corollary}[Projective BP is well-defined]
\label{cor:projective-bp-well-defined}
Since \(T_G\) is semi-equivariant under \(\Theta_G\) (Theorem \ref{thm:bp-semi-equivariant}), it descends to a well-defined map on projective messages:
\[
\overline{T}_G : \mathbb{P}\cat{Msg}(G) \to \mathbb{P}\cat{Msg}(G)
\]

\textbf{No injectivity assumption on \(\Theta_G\) is required}—descent follows directly from semi-equivariance.
\end{corollary}

\begin{proof}
We must show that if \([m] = [m']\) (same gauge orbit), then \([T_G(m)] = [T_G(m')]\).

If \([m] = [m']\), then \(m' = k \cdot m\) for some \(k \in K_G\). By Theorem \ref{thm:bp-semi-equivariant}:
\[
T_G(m') = T_G(k \cdot m) = \Theta_G(k) \cdot T_G(m)
\]

So \(T_G(m')\) and \(T_G(m)\) differ by gauge \(\Theta_G(k)\), meaning \textbf{they are in the same orbit}:
\[
[T_G(m')] = [\Theta_G(k) \cdot T_G(m)] = [T_G(m)]
\]

The last equality holds because \(\Theta_G(k) \in K_G\) acts on the same gauge group, so multiplication by \(\Theta_G(k)\) preserves gauge orbits.

Thus \(\overline{T}_G([m]) = [T_G(m)]\) is well-defined (independent of representative).
\qed
\end{proof}

\begin{remark}[Why injectivity is not needed]
The semi-equivariance property is sufficient for descent to the quotient. For a map \(f : M \to M\) with symmetry group \(G\), if \(f(g \cdot m) = \theta(g) \cdot f(m)\) for a homomorphism \(\theta : G \to G\), then \(f\) descends to the quotient \(M/G\). The requirement is that \(\theta\) maps into the same group acting compatibly on \(M\); injectivity of \(\theta\) is not necessary. In our setting, if \(m' = k \cdot m\) lies in the same orbit, then \(T_G(m') = \Theta_G(k) \cdot T_G(m)\). Since \(\Theta_G(k) \in K_G\) acts on the orbit space, we have \([T_G(m')] = [T_G(m)]\).
\end{remark}

\begin{remark}[Fixed points and normalized beliefs]
Combining Corollary \ref{cor:projective-bp-well-defined} with Proposition \ref{prop:gauge-invariance}, we see that:
\begin{itemize}
\item Fixed points of \(T_G\) may not be unique
\item But their gauge orbits \([m^*]\) are unique (if \(T_G\) has a unique fixed orbit)
\item Normalized beliefs are gauge-invariant, so they are well-defined on orbits
\end{itemize}

This explains why BP can have multiple fixed points that all yield the same normalized beliefs—they differ only by gauge.
\end{remark}

\paragraph{Further categorical/topological structure.} We move the groupoid and Kan-complex interpretation (and its link to $\pi_1$ holonomy) to Appendix~\ref{app:gauge-kan}; the main text uses only the rescaling action and BP equivariance.
\section{Descent Reformulation: Exactness as Effective Descent}
\label{sec:descent}
\label{sec:decomp}

This section provides the \textbf{universal framework} unifying all previous results. The key insight:

\begin{quote}
\textbf{Exact inference is effective descent of local semantics to global semantics along a cover.}
\end{quote}

Trees (Section \ref{sec:trees}) and junction trees (Section \ref{sec:decomp}) are special cases where descent succeeds. Loopy BP failures (mentioned in Section \ref{sec:bpoperator}) are descent obstructions. Trees are exact because their cover nerve is contractible, making descent automatic. Junction trees work because running intersection enforces descent compatibility. Loopy BP fails when covers lack sufficient refinement, resulting in nontrivial descent obstructions.

We prove Theorem \ref{thm:descent-exactness}, which subsumes both Theorem \ref{thm:treeexact} (trees) and Theorem \ref{thm:junction-tree-exact} (junction trees) as corollaries.

\subsection{The nonnegotiable boundary}

\begin{principle}[Fundamental limitation of local message passing]
\label{principle:locality}
Exact inference on arbitrary loopy graphs \textbf{cannot} be achieved by purely local operations on edge-separator messages without:
\begin{enumerate}
\item Enlarging carriers (using clusters instead of variables), or
\item Using higher coherence data (messages with additional structure)
\end{enumerate}

\textbf{Reason}: Global semantics is a limit in \(\cat{Mat}_R\). When the indexing diagram has nontrivial cycles, this limit is \emph{not} computed by 1-dimensional local elimination.
\end{principle}

\textbf{What we achieve}: We redefine "locality" using \textbf{descent on covers}, eliminating "tree dependence" as a primitive concept. Trees become one sufficient cover (where descent is automatic). Treewidth emerges as the minimal refinement complexity needed to force descent.

\subsection{The object whose semantics we compute}

Fix throughout this section:
\begin{itemize}
\item A typed signature \(\Sigma = (\Lambda, \Gamma, s, t)\) 
\item A state space assignment \(\Omega : \Lambda \to \cat{FinSet}\)
\item A factor assignment \(\Phi : \Gamma \to \cat{Mat}_R\) 
\item A polarized factor graph \(G \in \cat{FG}_{\Sigma}\) 
\end{itemize}

By the universal property, we have the evaluation:
\[
\llbracket G \rrbracket_R : \Omega(X) \to \Omega(Y)
\]
in \(\cat{Mat}_R\), where \(X, Y\) are the source and target types of \(G\).

\textbf{Inference goal}: Compute marginals \(\mathrm{mar}_v(G)\) for each variable \(v \in V(G)\).

\subsection{Covers: Decomposing the graph into pieces}

\begin{definition}[Cover of a factor graph]
\label{def:cover}
A \textbf{cover} of factor graph \(G\) is a pair \(\mathcal{U} = (I, \{U_i\}_{i \in I})\) where:
\begin{itemize}
\item \(I\) is a finite index set
\item \(U_i \subseteq V(G)\) (subsets of variables)
\end{itemize}
satisfying:

\textbf{(Cov1) Variable coverage}: 
\[
\bigcup_{i \in I} U_i = V(G)
\]
Every variable appears in at least one piece.

\textbf{(Cov2) Factor coverage}: For each factor \(f \in F(G)\),
\[
\exists i \in I : \mathsf{nbhd}(f) \subseteq U_i
\]
Every factor fits entirely within some piece.

Interpretation: A cover decomposes the factor graph into overlapping pieces, each large enough to contain complete factors.
\end{definition}

\begin{example}[Covers from tree decompositions]
\label{ex:cover-from-td}
A tree decomposition \((T, \mathsf{Bag})\) yields a cover by setting:
\[
I = V(T), \quad U_t = \mathsf{Bag}(t) \cap V(G)
\]

By the tree decomposition axioms:
\begin{itemize}
\item Variable coverage holds (every variable in some bag)
\item Factor coverage holds (every factor fits in some bag)
\end{itemize}
\end{example}

\begin{example}[Star cover of a tree]
For a tree factor graph \(G\) rooted at \(r\), the \textbf{star cover} assigns to each variable \(v\) the set:
\[
U_v = \{v\} \cup \{\text{neighbors of } v\}
\]
This is the implicit cover used in standard BP (Section \ref{sec:bpoperator}).
\end{example}

\subsection{The presheaf of local semantics}

\begin{definition}[Local state space]
\label{def:local-state-space}
For a subset \(U \subseteq V(G)\) of variables, the \textbf{local state space} is:
\[
\Omega(U) = \prod_{v \in U} \Omega(\lambda(v))
\]
the Cartesian product of individual state spaces.

Interpretation: Joint configurations of variables in \(U\).
\end{definition}

\begin{definition}[Local function space]
\label{def:local-function-space}
The \textbf{local function space} (or local message space) is:
\[
\mathcal{M}(U) = R^{\Omega(U)}
\]
the \(R\)-semimodule of functions from \(\Omega(U)\) to \(R\).

Interpretation: Potentials or beliefs defined on piece \(U\). In the tropical semiring \(R = \mathbb{T}_{\min}\), these are energy functions.
\end{definition}

\begin{definition}[Restriction maps (marginalization)]
\label{def:restriction-maps}
For \(V \subseteq U\), define the \textbf{restriction map}:
\[
\rho_{U \to V} : \mathcal{M}(U) \to \mathcal{M}(V)
\]
by:
\[
(\rho_{U \to V} F)(x_V) = \bigoplus_{x_{U \setminus V}} F(x_V, x_{U \setminus V})
\]

where \(\bigoplus\) is the sum operation in \(R\).

\textbf{Cases}:
\begin{itemize}
\item \textbf{Standard semiring} \(R = \mathbb{R}_{\geq 0}\): \(\bigoplus = +\) (sum)
\[
(\rho_{U \to V} F)(x_V) = \sum_{x_{U \setminus V}} F(x_V, x_{U \setminus V})
\]
This is \textbf{marginalization} (summing out variables).

\item \textbf{Tropical semiring} \(R = \mathbb{T}_{\min}\): \(\bigoplus = \min\) (minimum)
\[
(\rho_{U \to V} F)(x_V) = \min_{x_{U \setminus V}} F(x_V, x_{U \setminus V})
\]
This is \textbf{min-marginalization} (finding minimum energy).
\end{itemize}

Interpretation: Eliminate variables in \(U \setminus V\) by summing (or minimizing) over their states. This is the counit operation from Section \ref{sec:matcat}: \(\varepsilon_{U \setminus V}\).
\end{definition}

\begin{theorem}[Presheaf of local function spaces]
\label{thm:presheaf}
The assignment:
\[
\mathcal{M}_{\mathcal{U}} : \mathcal{I}(\mathcal{U})^{\mathrm{op}} \to \cat{SMod}_R
\]
defined by:
\begin{itemize}
\item On objects: \(U \mapsto \mathcal{M}(U) = R^{\Omega(U)}\)
\item On morphisms: \((U \supseteq V) \mapsto \rho_{U \to V} : \mathcal{M}(U) \to \mathcal{M}(V)\)
\end{itemize}
is a \textbf{contravariant functor} (presheaf).
\end{theorem}

\begin{proof}
\textbf{Functoriality on identities}: For \(U \to U\),
\[
\rho_{U \to U}(F)(x_U) = \bigoplus_{x_{\emptyset}} F(x_U) = F(x_U)
\]
since there are no variables to eliminate. Thus \(\rho_{U \to U} = \mathrm{id}_{\mathcal{M}(U)}\).

\textbf{Functoriality on composition}: For \(W \subseteq V \subseteq U\), we must show:
\[
\rho_{U \to W} = \rho_{V \to W} \circ \rho_{U \to V}
\]

Compute the right-hand side:
\begin{align*}
(\rho_{V \to W} \circ \rho_{U \to V})(F)(x_W)
&= \rho_{V \to W}(\rho_{U \to V}(F))(x_W) \\
&= \bigoplus_{x_{V \setminus W}} (\rho_{U \to V}(F))(x_V) \\
&= \bigoplus_{x_{V \setminus W}} \left(\bigoplus_{x_{U \setminus V}} F(x_U)\right)
\end{align*}

Since \((U \setminus V) \cap (V \setminus W) = \emptyset\) and \((U \setminus V) \cup (V \setminus W) = U \setminus W\), by associativity and commutativity of \(\bigoplus\) (Axiom 1-2 of semirings, Section \ref{sec:preliminaries}):
\begin{align*}
&= \bigoplus_{x_{V \setminus W}} \bigoplus_{x_{U \setminus V}} F(x_U) \\
&= \bigoplus_{x_{U \setminus W}} F(x_U) \\
&= \rho_{U \to W}(F)(x_W)
\end{align*}

Thus \(\mathcal{M}_{\mathcal{U}}\) preserves composition.
\qed
\end{proof}

Consequence: Local function spaces form a \textbf{presheaf} on the intersection category, compatible with marginalization. This is the categorical formalization of "local beliefs must be compatible on overlaps."

\subsection{Factor allocation and cluster potentials}

\begin{definition}[Factor allocation for a cover]
\label{def:cover-allocation}
A \textbf{factor allocation} for cover \(\mathcal{U} = (I, \{U_i\}_{i \in I})\) is a function:
\[
\kappa : F(G) \to I
\]
such that for each factor \(f \in F(G)\):
\[
\mathsf{nbhd}(f) \subseteq U_{\kappa(f)}
\]

Interpretation: Assign each factor to a piece large enough to contain all its adjacent variables.

\textbf{Existence}: By the factor coverage axiom (Cov2), such a \(\kappa\) always exists.
\end{definition}

\begin{definition}[Cluster potential]
\label{def:cluster-potential}
Given allocation \(\kappa\), the \textbf{cluster potential} at piece \(i \in I\) is:
\[
\Psi_i \in \mathcal{M}(U_i) = R^{\Omega(U_i)}
\]
defined by:
\[
\Psi_i(x_{U_i}) = \prod_{f : \kappa(f) = i} \phi_f(x_{\mathsf{nbhd}(f)})
\]

with the convention that an empty product equals \(1\) (the unit in \(R\)).

Interpretation: \(\Psi_i\) is the product of all factors assigned to piece \(i\), viewed as a function on the joint state of \(U_i\). This generalizes the bag potential.
\end{definition}

\begin{proposition}[Cluster potentials factor the joint distribution]
\label{prop:cluster-factorization}
The product of all cluster potentials equals the product of all original factors:
\[
\prod_{i \in I} \Psi_i(x_{U_i}) = \prod_{f \in F(G)} \phi_f(x_{\mathsf{nbhd}(f)})
\]
for any global configuration \(x \in \Omega(V(G))\) (restricting to appropriate subsets on each side).
\end{proposition}

\begin{proof}
By definition of \(\Psi_i\):
\begin{align*}
\prod_{i \in I} \Psi_i(x_{U_i})
&= \prod_{i \in I} \left(\prod_{f : \kappa(f) = i} \phi_f(x_{\mathsf{nbhd}(f)})\right) \\
&= \prod_{f \in F(G)} \phi_f(x_{\mathsf{nbhd}(f)}) \quad \text{(rearranging product)}
\end{align*}

since \(\kappa\) is a function, each \(f\) appears in exactly one \(\Psi_i\).
\qed
\end{proof}

Consequence: Cluster potentials preserve the factorization structure of the original graphical model. This connects to the universal semantics: the evaluation \(\llbracket G \rrbracket_R\) factors through the cluster potentials.

\subsection{Descent data: When do local pieces glue?}

\begin{definition}[Descent datum]
\label{def:descent-datum}
A \textbf{descent datum} for presheaf \(\mathcal{M}_{\mathcal{U}}\) is a family \((F_i)_{i \in I}\) with \(F_i \in \mathcal{M}(U_i)\) such that:

\textbf{(DD1) Pairwise compatibility}: For every \(i, j \in I\) with \(U_{ij} = U_i \cap U_j \neq \emptyset\),
\[
\rho_{U_i \to U_{ij}}(F_i) = \rho_{U_j \to U_{ij}}(F_j)
\]

\textbf{(DD2) Higher compatibility}: For every triple \(i, j, k \in I\) with \(U_{ijk} = U_i \cap U_j \cap U_k \neq \emptyset\), the restrictions to \(U_{ijk}\) satisfy:
\[
\rho_{U_i \to U_{ijk}}(F_i) = \rho_{U_j \to U_{ijk}}(F_j) = \rho_{U_k \to U_{ijk}}(F_k)
\]

and similarly for all higher intersections.

Interpretation: Local functions \(F_i\) are compatible on overlaps—they "agree" after marginalizing to shared variables. This is the condition for gluing local data to global data.
\end{definition}

\begin{remark}[Automatic for posets]
Since \(\mathcal{I}(\mathcal{U})\) is a poset, (DD2) is automatic once (DD1) holds: if \(U_{ijk} \subseteq U_{ij} \subseteq U_i\), then by functoriality (Theorem \ref{thm:presheaf}):
\[
\rho_{U_i \to U_{ijk}} = \rho_{U_{ij} \to U_{ijk}} \circ \rho_{U_i \to U_{ij}}
\]

So checking pairwise intersections suffices. Higher coherence is \textbf{not} automatic when we quotient by gauge (Section \ref{sec:kanbp}), which is why loopy graphs with gauge introduce obstructions.
\end{remark}

\begin{definition}[Global gluing]
\label{def:global-gluing}
If \((F_i)_{i \in I}\) is a descent datum, a \textbf{global gluing} is a function:
\[
F \in \mathcal{M}(V(G)) = R^{\Omega(V(G))}
\]
such that for all \(i \in I\):
\[
\rho_{V(G) \to U_i}(F) = F_i
\]

Interpretation: \(F\) is a global function whose restrictions to all pieces recover the local data.
\end{definition}

\begin{theorem}[Effective descent for finite state spaces]
\label{thm:effective-descent}
For finite state spaces \(\Omega(U)\), every descent datum \((F_i)_{i \in I}\) has a \textbf{unique} global gluing \(F \in \mathcal{M}(V(G))\).
\end{theorem}

\begin{proof}
\textbf{Construction}: For each global configuration \(x \in \Omega(V(G))\), define:
\[
F(x) = F_i(x|_{U_i})
\]
where \(i\) is any index with \(\mathrm{supp}(x) \subseteq U_i\) (where \(\mathrm{supp}(x)\) is the set of variables where \(x\) is defined).

\textbf{Well-definedness}: We must show \(F(x)\) is independent of the choice of \(i\).

If \(\mathrm{supp}(x) \subseteq U_i\) and \(\mathrm{supp}(x) \subseteq U_j\), then \(\mathrm{supp}(x) \subseteq U_{ij}\). By pairwise compatibility (DD1):
\[
\rho_{U_i \to U_{ij}}(F_i)(x|_{U_{ij}}) = \rho_{U_j \to U_{ij}}(F_j)(x|_{U_{ij}})
\]

But since \(x|_{U_i}\) and \(x|_{U_j}\) agree on \(U_{ij}\), we have:
\[
F_i(x|_{U_i}) = \rho_{U_i \to U_{ij}}(F_i)(x|_{U_{ij}}) = \rho_{U_j \to U_{ij}}(F_j)(x|_{U_{ij}}) = F_j(x|_{U_j})
\]

Thus \(F(x)\) is well-defined.

\textbf{Gluing property}: By construction, \(\rho_{V(G) \to U_i}(F) = F_i\) for all \(i\).

\textbf{Uniqueness}: If \(F'\) is another gluing, then for all \(x\):
\[
F'(x) = F'|_{U_i}(x|_{U_i}) = F_i(x|_{U_i}) = F(x)
\]

Thus \(F' = F\).
\qed
\end{proof}

Consequence: For finite discrete models (which we consider throughout), descent is \textbf{effective}—compatible local data uniquely determine global data. This is the categorical foundation for exact inference.

\subsection{The universal exactness theorem}

\begin{theorem}[Exact inference as effective descent]
\label{thm:descent-exactness}
Let \(G \in \cat{FG}_{\Sigma}\) be a polarized factor graph, \((\Omega, \Phi)\) an interpretation and \(\mathcal{U} = (I, \{U_i\}_{i \in I})\) a cover with allocation \(\kappa\).

Let \((\Psi_i)_{i \in I}\) be the cluster potentials (Definition \ref{def:cluster-potential}).

\textbf{Part A (Forward direction)}: If there exists a message-passing scheme producing cluster beliefs \((B_i)_{i \in I}\) with \(B_i \in \mathcal{M}(U_i)\) forming a descent datum:
\[
\rho_{U_i \to U_{ij}}(B_i) = \rho_{U_j \to U_{ij}}(B_j) \quad \forall i, j
\]

then there exists a global function \(F \in \mathcal{M}(V(G))\) such that:
\begin{enumerate}
\item \(\rho_{V(G) \to U_i}(F) = B_i\) for all \(i\) (gluing property)
\item The single-variable marginals of \(F\) equal the categorical marginals:
\[
\mathrm{mar}_v(F) = \mathrm{mar}_v(\llbracket G \rrbracket_R)
\]
for all \(v \in V(G)\)
\end{enumerate}

\textbf{Part B (Backward direction)}: If \(F \in \mathcal{M}(V(G))\) represents the unnormalized joint distribution from \(\llbracket G \rrbracket_R\), then:
\[
(F_i)_{i \in I} := (\rho_{V(G) \to U_i}(F))_{i \in I}
\]
is a descent datum for \(\mathcal{M}_{\mathcal{U}}\).

Interpretation: Exactness of a local inference scheme is \textbf{precisely} the property that it computes an effective descent datum for the global semantics.
\end{theorem}

\begin{proof}
\textbf{Part A: Existence of global gluing}

By Theorem \ref{thm:effective-descent}, the descent datum \((B_i)_{i \in I}\) has a unique global gluing \(F \in \mathcal{M}(V(G))\) with:
\[
\rho_{V(G) \to U_i}(F) = B_i \quad \forall i \in I
\]

This establishes (1). We must show (2): that \(F\) represents the same semantics as \(\llbracket G \rrbracket_R\).

\textbf{Step 1: Relate \(F\) to cluster potentials}

Assume the cluster beliefs \(B_i\) are derived from \(\Psi_i\) by a consistent message-passing scheme (e.g., junction tree BP). By construction, each \(B_i\) incorporates all factors assigned to piece \(i\):
\[
B_i(x_{U_i}) \propto \Psi_i(x_{U_i}) \cdot \text{(incoming separator messages)}
\]

When the message scheme produces a descent datum, the separator messages enforce consistency, so the glued \(F\) satisfies:
\[
F(x) \propto \prod_{i \in I} \Psi_i(x_{U_i})
\]

By Proposition \ref{prop:cluster-factorization}:
\[
\prod_{i \in I} \Psi_i(x_{U_i}) = \prod_{f \in F(G)} \phi_f(x_{\mathsf{nbhd}(f)})
\]

\textbf{Step 2: Connect to categorical semantics}

The evaluation \(\llbracket G \rrbracket_R\) represents the joint distribution:
\[
\llbracket G \rrbracket_R(x) = \prod_{f \in F(G)} \phi_f(x_{\mathsf{nbhd}(f)})
\]

(up to normalization constants, which don't affect marginals after renormalizing).

Thus \(F\) and \(\llbracket G \rrbracket_R\) represent the same unnormalized distribution.

\textbf{Step 3: Marginals coincide}

For any variable \(v \in V(G)\), the marginal from \(F\) is:
\[
\mathrm{mar}_v(F)(x_v) = \rho_{V(G) \to \{v\}}(F)(x_v) = \sum_{x_{V(G) \setminus \{v\}}} F(x)
\]

By the counit operation in \(\cat{Mat}_R\):
\[
\varepsilon_{V(G) \setminus \{v\}} : R^{\Omega(V(G))} \to R^{\Omega(\{v\})}
\]

is exactly summation over eliminated variables. Thus:
\[
\mathrm{mar}_v(F) = \varepsilon_{V(G) \setminus \{v\}} \circ F = \varepsilon_{V(G) \setminus \{v\}} \circ \llbracket G \rrbracket_R = \mathrm{mar}_v(\llbracket G \rrbracket_R)
\]

by functoriality of evaluation.

\textbf{Part B: Restrictions form a descent datum}

Let \(F \in \mathcal{M}(V(G))\) be the global unnormalized joint. Define:
\[
F_i = \rho_{V(G) \to U_i}(F)
\]

For any \(i, j\) with \(U_{ij} \neq \emptyset\):
\begin{align*}
\rho_{U_i \to U_{ij}}(F_i)
&= \rho_{U_i \to U_{ij}}(\rho_{V(G) \to U_i}(F)) \\
&= \rho_{V(G) \to U_{ij}}(F) \quad \text{(by functoriality, Theorem \ref{thm:presheaf})} \\
&= \rho_{U_j \to U_{ij}}(\rho_{V(G) \to U_j}(F)) \\
&= \rho_{U_j \to U_{ij}}(F_j)
\end{align*}

Thus \((F_i)_{i \in I}\) satisfies pairwise compatibility (DD1). Higher compatibility (DD2) follows similarly from functoriality.
\qed
\end{proof}

Consequence: This theorem \textbf{unifies all exactness results}:
\begin{itemize}
\item \textbf{Tree exactness} (Theorem \ref{thm:treeexact}): The star cover of a tree has contractible nerve, so descent is automatic
\item \textbf{Junction tree exactness} (Theorem \ref{thm:junction-tree-exact}): Running intersection ensures pairwise compatibility, producing a descent datum
\item \textbf{Loopy BP failure}: Covers without sufficient refinement fail to produce descent data—messages are incompatible on overlaps
\end{itemize}

\subsection{Corollaries: Recovering previous results}

\begin{corollary}[Tree exactness from descent]
\label{cor:tree-from-descent}
Theorem \ref{thm:treeexact} (BP exact on trees) is a special case of Theorem \ref{thm:descent-exactness} where the cover nerve is contractible.
\end{corollary}

\begin{proof}
For a tree \(G\), use the star cover from Example \ref{ex:cover-from-td}. The two-pass BP schedule produces messages that enforce separator consistency.

Since the tree has no cycles, all overlap compatibilities are automatically satisfied (no nontrivial coherence conditions). Thus BP computes a descent datum, and by Theorem \ref{thm:descent-exactness}, the beliefs equal categorical marginals.
\qed
\end{proof}

\begin{corollary}[Junction tree exactness from descent]
\label{cor:jt-from-descent}
Theorem \ref{thm:junction-tree-exact} (junction tree BP is exact) is a special case of Theorem \ref{thm:descent-exactness} where the cover satisfies running intersection.
\end{corollary}

\begin{proof}
A junction tree \(J\) is built from a tree decomposition \((T, \mathsf{Bag})\). The cover \(\mathcal{U} = \{U_t = \mathsf{Bag}(t)\}_{t \in V(T)}\) satisfies variable and factor coverage.

The running intersection property ensures that for adjacent bags \(t, t'\), the separator \(S_{t,t'} = U_t \cap U_{t'}\) enforces consistency:
\[
\rho_{U_t \to S_{t,t'}}(B_t) = \rho_{U_{t'} \to S_{t,t'}}(B_{t'})
\]

Since \(T\) is a tree, there are no higher cycles, so this pairwise compatibility suffices for descent. By Theorem \ref{thm:descent-exactness}, junction tree BP computes exact marginals.
\qed
\end{proof}

\textbf{Summary}: Section \ref{sec:descent} provides the \textbf{universal framework}. Trees and junction trees are not special cases by accident—they are covers where descent succeeds. The categorical perspective makes this transparent.

\textbf{Next}: Section \ref{sec:hatcc} constructs an \textbf{algorithm} implementing effective descent via holonomy-aware tree compilation.

\section{Holonomy-Aware Tree Compilation (HATCC)}
\label{sec:hatcc}
\label{sec:hatcc}

Theorem \ref{thm:descent-exactness} characterizes exact inference as effective descent. Trees and junction trees succeed because their covers admit descent data. Loopy graphs fail when cycles prevent consistent gluing.

We present Holonomy-Aware Tree Compilation (HATCC): an algorithm detecting descent obstructions via holonomy and resolving them by mode variable compilation. For each cycle \(C_e\) in the factor nerve \(G_{\mathcal{N}}\), we compute a holonomy matrix \(H_e\) measuring parallel transport around \(C_e\). Nontrivial \(H_e\) obstructs descent; HATCC compiles such obstructions into discrete mode variables, reducing to tree BP.

\textbf{Key idea}: Work with a \textbf{factor nerve graph} \(G_{\mathcal{N}}\) (factors as vertices, overlaps as edges) instead of variable-based covers. For each cycle in \(G_{\mathcal{N}}\), compute a \textbf{holonomy matrix} \(H_e\) measuring how beliefs "transport" around the cycle. When holonomy is nontrivial, compile it into discrete \textbf{mode variables} via strongly connected component (SCC) quotient, then run tree BP on the augmented graph.

\textbf{Relationship to Section \ref{sec:descent}}: HATCC implements Theorem \ref{thm:descent-exactness} by:
\begin{itemize}
\item Constructing a specialized cover from the factor nerve (connecting to Definition \ref{def:cover})
\item Detecting when descent data exists via holonomy (connecting to Definition \ref{def:descent-datum})
\item Augmenting the cover with mode variables to enforce compatibility (connecting to Definition \ref{def:global-gluing})
\end{itemize}

\textbf{Structure}:
\begin{itemize}
\item \S\ref{subsec:factor-nerve}: Factor nerve graph (dual to intersection category, Definition \ref{def:intersection-category})
\item \S\ref{subsec:backbone}: Backbone tree + chords (decomposition into tree + cycles)
\item \S\ref{subsec:holonomy}: Holonomy matrices (transport around cycles, connecting to restriction maps \(\rho_{U \to V}\), Definition \ref{def:restriction-maps})
\item \S\ref{subsec:modes}: Mode quotients (SCC abstraction of holonomy)
\item \S\ref{subsec:selectors}: Selector factors (enforcing descent compatibility)
\item \S\ref{subsec:hatcc-algorithm}: HATCC compilation algorithm
\item \S\ref{subsec:exactness}: Exactness characterization (when HATCC = tree BP)
\end{itemize}

\subsection{The Factor Nerve Graph}
\label{subsec:factor-nerve}
\label{subsec:factor-nerve}

Section \ref{sec:descent} organized inference via covers \(\mathcal{U} = (I, \{U_i\}_{i \in I})\) of variable sets (Definition \ref{def:cover}). The intersection category \(\mathcal{I}(\mathcal{U})\) encodes overlaps, and the presheaf \(\mathcal{M}_{\mathcal{U}} : \mathcal{I}(\mathcal{U})^{\mathrm{op}} \to \cat{SMod}_R\) (Theorem \ref{thm:presheaf}) organizes local function spaces.

HATCC uses a \textbf{factor-centric} dual: instead of partitioning variables, we organize \emph{factors} by their connectivity.

\begin{definition}[Factor nerve graph]
\label{def:factor-nerve}
For a polarized factor graph \(G \in \cat{FG}_{\Sigma}\) (Definition \ref{def:polarized-fg}) with factors \(F(G)\) and neighborhoods \(\mathsf{nbhd}(f) \subseteq V(G)\) for each \(f \in F(G)\), the \textbf{factor nerve graph} is:
\[
G_{\mathcal{N}} = (F(G), E_{\mathcal{N}})
\]

\textbf{Vertices}: The set of factors \(F(G)\).

\textbf{Edges}: For \(f_1, f_2 \in F(G)\), include edge \((f_1, f_2) \in E_{\mathcal{N}}\) if and only if:
\[
\mathsf{nbhd}(f_1) \cap \mathsf{nbhd}(f_2) \neq \emptyset
\]

\textbf{Edge interface}: For each \(e = (f_1, f_2) \in E_{\mathcal{N}}\), define:
\[
J_e := \mathsf{nbhd}(f_1) \cap \mathsf{nbhd}(f_2)
\]

\textbf{Edge weight}: 
\[
w(e) := \log |\Omega(J_e)|
\]
where \(\Omega(J_e)\) is the local state space (Definition \ref{def:local-state-space}).
\end{definition}

\begin{remark}[Three different "nerve" constructions]
\label{rem:three-nerves}
This paper uses three distinct nerve concepts:

\textbf{1. Nerve of a category}: Applied to the message action groupoid \(\mathcal{T}_G\) to obtain simplicial set \(N(\mathcal{T}_G)\) (used in Section \ref{sec:kanbp} for Kan complex structure).

\textbf{2. \v{C}ech nerve of a cover} (implicit in Section \ref{sec:descent}): For cover \(\mathcal{U}\), the nerve of intersection category \(\mathcal{I}(\mathcal{U})\) encodes higher overlaps (used for descent diagrams).

\textbf{3. Factor nerve} (Definition \ref{def:factor-nerve}): Graph with factors as vertices, used for holonomy computation.

These are related but distinct: (1) and (2) produce simplicial sets, while (3) is a weighted graph (just the 1-skeleton). The factor nerve refines the 1-skeleton of the \v{C}ech nerve.
\end{remark}

\begin{proposition}[Factor nerve refines cover nerve]
\label{prop:factor-nerve-refinement}
Let \(\mathcal{U} = (I, \{U_i\}_{i \in I})\) be a cover (Definition \ref{def:cover}) with factor allocation \(\kappa : F(G) \to I\) (Definition \ref{def:cover-allocation}).

If \((f_1, f_2) \in E_{\mathcal{N}}\), then either:
\begin{enumerate}
\item \(\kappa(f_1) = \kappa(f_2)\) (same cover piece), or
\item \(U_{\kappa(f_1)} \cap U_{\kappa(f_2)} \neq \emptyset\) (adjacent cover pieces)
\end{enumerate}

Interpretation: Factor nerve edges refine intersection category morphisms.
\end{proposition}

\begin{proof}
If \((f_1, f_2) \in E_{\mathcal{N}}\), then \(\mathsf{nbhd}(f_1) \cap \mathsf{nbhd}(f_2) \neq \emptyset\) by Definition \ref{def:factor-nerve}.

By factor allocation (Definition \ref{def:cover-allocation}): \(\mathsf{nbhd}(f_i) \subseteq U_{\kappa(f_i)}\) for \(i = 1, 2\).

Therefore:
\[
U_{\kappa(f_1)} \cap U_{\kappa(f_2)} \supseteq \mathsf{nbhd}(f_1) \cap \mathsf{nbhd}(f_2) \neq \emptyset
\]

So either \(\kappa(f_1) = \kappa(f_2)\) or \((U_{\kappa(f_1)}, U_{\kappa(f_2)})\) is an edge in the 1-skeleton of \(N(\mathcal{I}(\mathcal{U}))\). \qed
\end{proof}

\begin{example}[Factor nerve for a 4-cycle]
\label{ex:factor-nerve-4cycle}
Consider a factor graph with 4 variables \(\{A, B, C, D\}\) (each binary: \(\Omega(\lambda(v)) = \{0, 1\}\)) and 4 factors:
\begin{align*}
f_1 &: \{A, B\} \to R, \quad \mathsf{nbhd}(f_1) = \{A, B\} \\
f_2 &: \{B, C\} \to R, \quad \mathsf{nbhd}(f_2) = \{B, C\} \\
f_3 &: \{C, D\} \to R, \quad \mathsf{nbhd}(f_3) = \{C, D\} \\
f_4 &: \{D, A\} \to R, \quad \mathsf{nbhd}(f_4) = \{D, A\}
\end{align*}

\textbf{Step 1: Compute factor nerve \(G_{\mathcal{N}}\)}

\textbf{Vertices}: \(F(G) = \{f_1, f_2, f_3, f_4\}\).

\textbf{Edges}: Check all pairs for overlaps:
\begin{itemize}
\item \((f_1, f_2)\): \(\mathsf{nbhd}(f_1) \cap \mathsf{nbhd}(f_2) = \{A, B\} \cap \{B, C\} = \{B\} \neq \emptyset\) → edge with interface \(J_{(f_1, f_2)} = \{B\}\)
\item \((f_2, f_3)\): \(\{B, C\} \cap \{C, D\} = \{C\}\) → edge with interface \(J_{(f_2, f_3)} = \{C\}\)
\item \((f_3, f_4)\): \(\{C, D\} \cap \{D, A\} = \{D\}\) → edge with interface \(J_{(f_3, f_4)} = \{D\}\)
\item \((f_4, f_1)\): \(\{D, A\} \cap \{A, B\} = \{A\}\) → edge with interface \(J_{(f_4, f_1)} = \{A\}\)
\item \((f_1, f_3)\): \(\{A, B\} \cap \{C, D\} = \emptyset\) → no edge
\item \((f_2, f_4)\): \(\{B, C\} \cap \{D, A\} = \emptyset\) → no edge
\end{itemize}

\textbf{Result}: \(G_{\mathcal{N}}\) is a 4-cycle: \(f_1 - f_2 - f_3 - f_4 - f_1\).

\textbf{Edge weights}: For binary variables, \(|\Omega(J_e)| = 2\), so:
\[
w(e) = \log 2 \approx 0.693 \quad \forall e \in E_{\mathcal{N}}
\]

\textbf{Step 2: Connection to covers}

A natural cover (Definition \ref{def:cover}) is:
\[
\mathcal{U} = (\{1, 2, 3, 4\}, \{U_1 = \{A, B\}, U_2 = \{B, C\}, U_3 = \{C, D\}, U_4 = \{D, A\}\})
\]

Factor allocation: \(\kappa(f_i) = i\) (each factor in its own piece).

\textbf{Intersection category \(\mathcal{I}(\mathcal{U})\)}:
\begin{itemize}
\item 0-cells: \(U_1, U_2, U_3, U_4\)
\item 1-cells: \(U_1 \cap U_2 = \{B\}\), \(U_2 \cap U_3 = \{C\}\), \(U_3 \cap U_4 = \{D\}\), \(U_4 \cap U_1 = \{A\}\)
\item 2-cells: No triple overlaps (all \(U_i \cap U_j \cap U_k = \emptyset\) for distinct \(i, j, k\))
\end{itemize}

The 1-skeleton of \(N(\mathcal{I}(\mathcal{U}))\) is the same 4-cycle as \(G_{\mathcal{N}}\), confirming Proposition \ref{prop:factor-nerve-refinement}.

\textbf{Observation}: This cycle represents a topological obstruction to tree-based inference. Standard BP (Section \ref{sec:bpoperator}) on this graph will iterate, but may not converge or may converge to incorrect marginals (loopy BP issue). HATCC resolves this via holonomy computation (Section \ref{subsec:holonomy}).
\end{example}

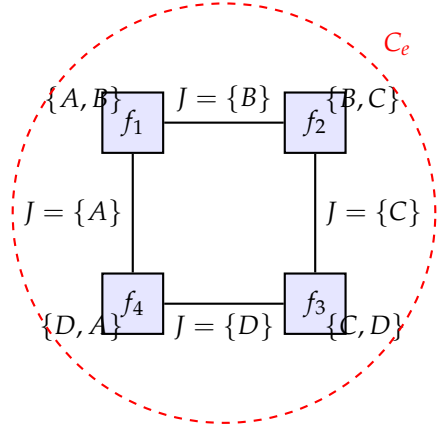
\begin{figure}[H]
\centering
\begin{tikzpicture}[scale=1.2]
\node[factornode] (f1) at (0,2) {\(f_1\)};
\node[factornode] (f2) at (2,2) {\(f_2\)};
\node[factornode] (f3) at (2,0) {\(f_3\)};
\node[factornode] (f4) at (0,0) {\(f_4\)};

\draw[thick] (f1) -- node[above] {\(J = \{B\}\)} (f2);
\draw[thick] (f2) -- node[right] {\(J = \{C\}\)} (f3);
\draw[thick] (f3) -- node[below] {\(J = \{D\}\)} (f4);
\draw[thick] (f4) -- node[left] {\(J = \{A\}\)} (f1);

\node[above left] at (f1) {\(\{A,B\}\)};
\node[above right] at (f2) {\(\{B,C\}\)};
\node[below right] at (f3) {\(\{C,D\}\)};
\node[below left] at (f4) {\(\{D,A\}\)};

\node[draw=red, thick, circle, fit=(f1)(f2)(f3)(f4), inner sep=10pt, dashed, label={[red]above right:\(C_e\)}] {};
\end{tikzpicture}
\caption{Factor nerve \(G_{\mathcal{N}}\) for Example \ref{ex:factor-nerve-4cycle}. Factors \(f_1, f_2, f_3, f_4\) form a 4-cycle with edge interfaces \(J_e = \mathsf{nbhd}(f_i) \cap \mathsf{nbhd}(f_j)\). The cycle \(C_e\) has \(\beta_1 = 1\), yielding a potential descent obstruction requiring holonomy analysis.}
\label{fig:factor-nerve-4cycle}
\end{figure}

\begin{algorithm}[H]
\caption{Construct factor nerve graph}
\label{alg:factor-nerve}
\begin{algorithmic}[1]
\Require Factor graph \(G\) with factors \(F(G)\) and neighborhoods \(\mathsf{nbhd}(f)\) for each \(f \in F(G)\)
\Ensure Factor nerve \(G_{\mathcal{N}} = (F(G), E_{\mathcal{N}})\) with edge interfaces \(\{J_e\}_{e \in E_{\mathcal{N}}}\) and weights \(\{w(e)\}_{e \in E_{\mathcal{N}}}\)

\State Initialize \(E_{\mathcal{N}} \gets \emptyset\), \(J \gets \{\}\), \(w \gets \{\}\)

\For{each pair \((f_1, f_2)\) with \(f_1, f_2 \in F(G)\) and \(f_1 < f_2\)} \Comment{Unique pairs}
    \State \(J_e \gets \mathsf{nbhd}(f_1) \cap \mathsf{nbhd}(f_2)\) \Comment{Compute interface}
    \If{\(J_e \neq \emptyset\)} \Comment{Non-empty overlap}
        \State Add edge \((f_1, f_2)\) to \(E_{\mathcal{N}}\)
        \State \(J[(f_1, f_2)] \gets J_e\) \Comment{Store interface}
        \State \(w[(f_1, f_2)] \gets \sum_{v \in J_e} \log |\Omega(\lambda(v))|\) \Comment{Weight}
    \EndIf
\EndFor

\State \Return \((G_{\mathcal{N}} = (F(G), E_{\mathcal{N}}), J, w)\)
\end{algorithmic}
\end{algorithm}

\begin{theorem}[Factor nerve construction complexity]
\label{thm:factor-nerve-complexity}
Algorithm \ref{alg:factor-nerve} has:
\begin{itemize}
\item \textbf{Time}: \(O(|F(G)|^2 \cdot \max_f |\mathsf{nbhd}(f)|)\)
\item \textbf{Space}: \(O(|F(G)| + |E_{\mathcal{N}}|)\)
\end{itemize}
where \(|F(G)|\) is the number of factors and \(\max_f |\mathsf{nbhd}(f)|\) is the maximum factor degree.
\end{theorem}

\begin{proof}
\textbf{Time complexity}:
\begin{itemize}
\item Line 2: Iterate over all pairs of factors: \(\binom{|F(G)|}{2} = O(|F(G)|^2)\) iterations
\item Line 3: Compute intersection of two sets of size at most \(\max_f |\mathsf{nbhd}(f)|\): \(O(\max_f |\mathsf{nbhd}(f)|)\) per iteration using hash sets
\item Lines 5-7: Constant-time operations (add edge, store interface)
\item Line 8: Sum over \(|J_e| \leq \max_f |\mathsf{nbhd}(f)|\) variables: \(O(\max_f |\mathsf{nbhd}(f)|)\)
\end{itemize}
Total: \(O(|F(G)|^2 \cdot \max_f |\mathsf{nbhd}(f)|)\).

\textbf{Space complexity}:
\begin{itemize}
\item Store \(|F(G)|\) vertices: \(O(|F(G)|)\)
\item Store at most \(\binom{|F(G)|}{2} = O(|F(G)|^2)\) edges, but typically \(|E_{\mathcal{N}}| \ll |F(G)|^2\) (sparse): \(O(|E_{\mathcal{N}}|)\)
\item Store interfaces and weights: \(O(|E_{\mathcal{N}}|)\)
\end{itemize}
Total: \(O(|F(G)| + |E_{\mathcal{N}}|)\).

\textbf{Practical considerations}: For sparse factor graphs (each factor overlaps with few others), \(|E_{\mathcal{N}}| = O(|F(G)|)\), so total is \(O(|F(G)|^2 \cdot \max_f |\mathsf{nbhd}(f)|)\) time and \(O(|F(G)|)\) space. \qed
\end{proof}

\subsection{Backbone Tree and Fundamental Cycles}
\label{subsec:backbone}
\label{subsec:backbone}

The factor nerve \(G_{\mathcal{N}}\) encodes the \emph{connectivity} of factors, but contains cycles that prevent direct application of tree BP. Our strategy is to \textbf{decompose} \(G_{\mathcal{N}}\) into:
\begin{itemize}
\item A \textbf{spanning tree} \(T\) (the "backbone"), on which standard tree BP would work.
\item A set of \textbf{chords} \(\mathcal{C}\) (edges completing cycles), which encode holonomy obstructions.
\end{itemize}

This decomposition is classical in graph theory: every connected graph admits a spanning tree, and each chord generates a unique \emph{fundamental cycle} by closing a path in the tree. The number of chords is the \textbf{cyclomatic complexity} \(c = |E_{\mathcal{N}}| - |F(G)| + 1\), which controls HATCC's computational cost (Theorem \ref{thm:hatcc-complexity}).

Intuition: If all chords have \emph{trivial holonomy} (beliefs transport consistently around fundamental cycles), then the graph behaves like a tree for inference purposes. When holonomy is \emph{nontrivial}, we must explicitly track it via mode variables (Section \ref{subsec:modes}).

\begin{definition}[Backbone tree and chords]
\label{def:backbone-tree}
A \textbf{backbone tree} for \(G_{\mathcal{N}} = (F(G), E_{\mathcal{N}})\) is a spanning tree \(T \subseteq E_{\mathcal{N}}\) (connected, acyclic, includes all vertices).

The \textbf{chords} are:
\[
\mathcal{C} := E_{\mathcal{N}} \setminus T
\]

\textbf{Count}: \(|T| = |F(G)| - 1\) and \(|\mathcal{C}| = |E_{\mathcal{N}}| - |F(G)| + 1\).
\end{definition}

\begin{definition}[Fundamental cycle]
\label{def:fundamental-cycle}
For chord \(e = (u, v) \in \mathcal{C}\), the \textbf{fundamental cycle} \(C_e\) is the simple cycle:
\[
C_e = P_T(u, v) \cup \{e\}
\]
where \(P_T(u, v) = (u = f_0, f_1, \ldots, f_k = v)\) is the unique path in \(T\) from \(u\) to \(v\).

The \textbf{chord interface} is \(J_e = \mathsf{nbhd}(u) \cap \mathsf{nbhd}(v)\).
\end{definition}

\begin{proposition}[Chords generate \(H_1\)]
\label{prop:chords-h1}
The fundamental cycles \(\{C_e : e \in \mathcal{C}\}\) form a \(\mathbb{Z}\)-basis for:
\[
H_1(G_{\mathcal{N}}; \mathbb{Z}) \cong \mathbb{Z}^{|\mathcal{C}|}
\]
\end{proposition}

\begin{proof}
The factor nerve \(G_{\mathcal{N}} = (F(G), E_{\mathcal{N}})\) is a connected graph with \(|V| = |F(G)|\) vertices and \(|E| = |E_{\mathcal{N}}|\) edges. The backbone tree \(T\) spans all vertices with \(|T| = |V| - 1\) edges. Thus:
\[
|\mathcal{C}| = |E_{\mathcal{N}}| - |T| = |E| - (|V| - 1) = |E| - |V| + 1
\]

By the standard formula for the first Betti number of a graph: \(\beta_1(G_{\mathcal{N}}) = |E| - |V| + 1 = |\mathcal{C}|\).

Each chord \(e \in \mathcal{C}\) determines a unique fundamental cycle \(C_e = P_T(u, v) \cup \{e\}\) where \(e = (u, v)\). These cycles are independent: removing any chord \(e\) reduces \(\beta_1\) by exactly 1. By construction, they generate \(H_1(G_{\mathcal{N}}; \mathbb{Z})\). \qed
\end{proof}

\textbf{Connection to descent}: Each chord \(e \in \mathcal{C}\) represents a potential \textbf{descent obstruction}. If descent succeeds (Definition \ref{def:descent-datum}), the cycle \(C_e\) must not obstruct compatibility. We formalize this via holonomy.

\begin{example}[Continuation of Example \ref{ex:factor-nerve-4cycle}: Backbone decomposition]
\label{ex:backbone-4cycle}
Continuing from Example \ref{ex:factor-nerve-4cycle}, we have \(G_{\mathcal{N}}\) as a 4-cycle with edges:
\[
E_{\mathcal{N}} = \{(f_1, f_2), (f_2, f_3), (f_3, f_4), (f_4, f_1)\}
\]

\textbf{Step 3: Select backbone tree \(T\)}

Since all edge weights are equal (\(w(e) = \log 2\) for all \(e\)), any spanning tree is a maximum spanning tree. Choose:
\[
T = \{(f_1, f_2), (f_2, f_3), (f_3, f_4)\}
\]
This is a path: \(f_1 - f_2 - f_3 - f_4\).

\textbf{Step 4: Identify chords}
\[
\mathcal{C} = E_{\mathcal{N}} \setminus T = \{(f_4, f_1)\}
\]

Single chord: \(e = (f_4, f_1)\) with interface \(J_e = \{A\}\).

\textbf{Step 5: Fundamental cycle}

For chord \(e = (f_4, f_1)\):
\begin{itemize}
\item Tree path \(P_T(f_4, f_1) = (f_4, f_3, f_2, f_1)\) (going backwards through tree)
\item Fundamental cycle: \(C_e = (f_4, f_3, f_2, f_1, f_4)\) (the original 4-cycle)
\end{itemize}

\textbf{Tree edge interfaces}:
\begin{align*}
J_0 &= \mathsf{nbhd}(f_4) \cap \mathsf{nbhd}(f_3) = \{D, A\} \cap \{C, D\} = \{D\} \\
J_1 &= \mathsf{nbhd}(f_3) \cap \mathsf{nbhd}(f_2) = \{C, D\} \cap \{B, C\} = \{C\} \\
J_2 &= \mathsf{nbhd}(f_2) \cap \mathsf{nbhd}(f_1) = \{B, C\} \cap \{A, B\} = \{B\} \\
J_3 &= J_e = \{A\} \quad \text{(chord interface)}
\end{align*}

The cycle transports states around interfaces: \(\{A\} \xrightarrow{f_4} \{D\} \xrightarrow{f_3} \{C\} \xrightarrow{f_2} \{B\} \xrightarrow{f_1} \{A\}\).

This setup enables holonomy computation (Section \ref{subsec:holonomy}).
\end{example}

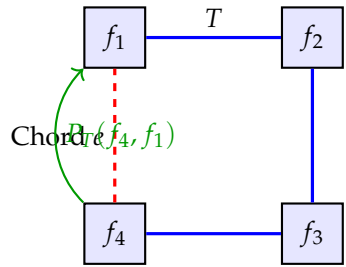
\begin{figure}[H]
\centering
\begin{tikzpicture}[scale=1.3]
\node[factornode] (f1) at (0,2) {\(f_1\)};
\node[factornode] (f2) at (2,2) {\(f_2\)};
\node[factornode] (f3) at (2,0) {\(f_3\)};
\node[factornode] (f4) at (0,0) {\(f_4\)};

\draw[very thick, blue] (f1) -- node[above, black] {\(T\)} (f2);
\draw[very thick, blue] (f2) -- (f3);
\draw[very thick, blue] (f3) -- (f4);

\draw[very thick, red, dashed] (f4) -- node[left, black] {Chord \(e\)} (f1);

\draw[->, thick, green!60!black, bend left=45] (f4) to node[right] {\(P_T(f_4, f_1)\)} (f1);

\node[below, align=center] at (1, -1) {Backbone tree \(T\) (blue solid) \\ Chord \(e = (f_4, f_1)\) (red dashed) \\ Fundamental cycle \(C_e = P_T \cup \{e\}\)};
\end{tikzpicture}
\caption{Backbone tree \(T\) (blue solid) and chord \(e = (f_4, f_1)\) (red dashed) for Example \ref{ex:backbone-4cycle}. The tree path \(P_T(f_4, f_1)\) combined with \(e\) forms the fundamental cycle \(C_e\). Chords generate \(H_1(G_{\mathcal{N}})\).}
\label{fig:backbone-4cycle}
\end{figure}

\begin{algorithm}[H]
\caption{Backbone tree selection and chord identification}
\label{alg:backbone-tree}
\begin{algorithmic}[1]
\Require Factor nerve \(G_{\mathcal{N}} = (F(G), E_{\mathcal{N}})\) with weights \(\{w(e)\}_{e \in E_{\mathcal{N}}}\)
\Ensure Backbone tree \(T \subseteq E_{\mathcal{N}}\), chords \(\mathcal{C}\), root \(r \in F(G)\), parent/children maps

\State \(T \gets\) \textsc{MaximumSpanningTree}(\(G_{\mathcal{N}}\), \(w\)) \Comment{Kruskal or Prim, \(O(|E_{\mathcal{N}}| \log |E_{\mathcal{N}}|)\)}
\State \(\mathcal{C} \gets E_{\mathcal{N}} \setminus T\) \Comment{Chords}
\State \(r \gets\) arbitrary root (e.g., lexicographically first factor in \(F(G)\))

\State \Comment{Root the tree: BFS from \(r\) to compute parent/children}
\State parent \(\gets \{\}\), children \(\gets \{f : [] \mid f \in F(G)\}\)
\State parent[\(r\)] \(\gets\) None
\State queue \(\gets [r]\)

\While{queue not empty}
    \State \(u \gets\) dequeue(queue)
    \For{each neighbor \(v\) of \(u\) in \(T\)} \Comment{Adjacent in tree}
        \If{parent[\(v\)] not set}
            \State parent[\(v\)] \(\gets u\)
            \State children[\(u\)].append(\(v\))
            \State enqueue(queue, \(v\))
        \EndIf
    \EndFor
\EndWhile

\State \Return \((T, \mathcal{C}, r, \text{parent}, \text{children})\)
\end{algorithmic}
\end{algorithm}

\begin{theorem}[Backbone selection complexity]
\label{thm:backbone-complexity}
Algorithm \ref{alg:backbone-tree} has:
\begin{itemize}
\item \textbf{Time}: \(O(|E_{\mathcal{N}}| \log |E_{\mathcal{N}}|)\) (dominated by MST)
\item \textbf{Space}: \(O(|F(G)| + |E_{\mathcal{N}}|)\)
\end{itemize}
\end{theorem}

\begin{proof}
\textbf{Time}:
\begin{itemize}
\item Line 1: Maximum spanning tree via Kruskal: \(O(|E_{\mathcal{N}}| \log |E_{\mathcal{N}}|)\)
\item Line 2: Set difference: \(O(|E_{\mathcal{N}}|)\)
\item Lines 4-16: BFS over tree \(T\) with \(|F(G)|\) vertices: \(O(|F(G)| + |T|) = O(|F(G)|)\)
\end{itemize}
Total: \(O(|E_{\mathcal{N}}| \log |E_{\mathcal{N}}|)\).

\textbf{Space}: Store tree \(T\), chords \(\mathcal{C}\), parent/children: \(O(|F(G)| + |E_{\mathcal{N}}|)\). \qed
\end{proof}

\subsection{Holonomy Matrices: Transport Around Cycles}
\label{subsec:holonomy}
\label{subsec:holonomy}

We now confront the core question: \textbf{when does a cycle obstruct descent?} Intuitively, if we "transport" a belief around a fundamental cycle \(\gamma_e\) (following restriction maps at each step), we should return to the same belief. If not, the cycle exhibits \textbf{nontrivial holonomy}, preventing global consistency.

In physics terminology: holonomy measures the "phase shift" accumulated by parallel-transporting a vector around a loop in a fiber bundle. Here, the fiber is the state space at interface variables, and transport is governed by factor potentials.

Key insight: For descent, we don't need the full probabilistic transport (marginalization via \(\rho_{U \to V}\), Definition \ref{def:restriction-maps}). We only need the \textbf{Boolean support skeleton}: which state pairs \((x, y)\) are \emph{compatible} via a factor? This is captured by \textbf{transport kernels}.

\begin{definition}[Transport kernel]
\label{def:transport-kernel}
For factor \(f \in F(G)\) with potential \(\phi_f : \Omega(\mathsf{nbhd}(f)) \to R\) (Definition \ref{def:factor-potentials}), and interfaces \(U, V \subseteq \mathsf{nbhd}(f)\), define the \textbf{transport kernel}:
\[
K^{U \to V}_f : \Omega(U) \times \Omega(V) \to \{0, 1\}
\]
by:
\[
K^{U \to V}_f(x, y) := \begin{cases}
1 & \text{if } \exists z \in \Omega(\mathsf{nbhd}(f) \setminus (U \cup V)): \phi_f(x, y, z) \neq 0 \\
0 & \text{otherwise}
\end{cases}
\]

Interpretation: \(K^{U \to V}_f(x, y) = 1\) means states \(x \in \Omega(U)\) and \(y \in \Omega(V)\) are \textbf{compatible} via factor \(f\) (there exists an extension \(z\) with nonzero potential).
\end{definition}

\begin{remark}[Transport vs restriction]
\label{rem:transport-vs-restriction}
Transport kernels are the \textbf{Boolean support} of restriction maps:
\begin{itemize}
\item Restriction map \(\rho_{U \to V}\) (Definition \ref{def:restriction-maps}): probabilistic marginalization (\(\bigoplus\)-sum)
\item Transport kernel \(K^{U \to V}_f\): Boolean reachability (\(\exists\)-quantification)
\end{itemize}

For semiring \(R\):
\[
K^{U \to V}_f(x, y) = 1 \quad \iff \quad \rho_{\mathsf{nbhd}(f) \to (U \cup V)}(\phi_f)(x, y) \neq 0
\]

Transport encodes which transitions are \emph{possible}; restriction computes \emph{probabilities} of those transitions.
\end{remark}

\begin{definition}[Cycle holonomy]
\label{def:holonomy}
For fundamental cycle \(C_e = (f_0, f_1, \ldots, f_k, f_0)\) (Definition \ref{def:fundamental-cycle}) with chord interface \(J_e\), define \textbf{tree edge interfaces}:
\[
J_i := \mathsf{nbhd}(f_i) \cap \mathsf{nbhd}(f_{i+1}) \quad \text{for } i = 0, \ldots, k-1
\]
and \(J_k := J_e\) (the chord interface).

The \textbf{holonomy matrix} is:
\[
H_e := K^{J_0 \to J_1}_{f_0} \circ K^{J_1 \to J_2}_{f_1} \circ \cdots \circ K^{J_{k-1} \to J_k}_{f_{k-1}} \circ K^{J_k \to J_0}_{f_k} : \Omega(J_e) \times \Omega(J_e) \to \{0, 1\}
\]

where \(\circ\) is Boolean matrix multiplication (\(\land\) for product, \(\lor\) for sum).

Interpretation: \(H_e(x, y) = 1\) means state \(x \in \Omega(J_e)\) can "transport" to state \(y \in \Omega(J_e)\) by traveling around cycle \(C_e\) via factor supports.
\end{definition}

\begin{theorem}[Holonomy detects descent obstruction]
\label{thm:holonomy-descent-obstruction}
Let \(\mathcal{U}\) be a cover with factor allocation \(\kappa\), and suppose a message-passing scheme produces cluster beliefs \((B_i)_{i \in I}\).

If all cycles in \(G_{\mathcal{N}}\) have \textbf{trivial holonomy}:
\[
H_e(x, x) = 1 \quad \forall e \in \mathcal{C}, \forall x \in \Omega(J_e)
\]
and \(H_e(x, y) = 0\) for \(x \neq y\),

then \((B_i)_{i \in I}\) satisfying separator compatibility (edges in \(T\)) automatically forms a descent datum (Definition \ref{def:descent-datum}).

Conversely, if some cycle has \textbf{nontrivial holonomy}, descent may fail even if tree separators are consistent.
\end{theorem}

\begin{proof}
We prove both directions, connecting holonomy to descent conditions (Definition \ref{def:descent-datum}).

\textbf{Setup}: Let \(\mathcal{U} = (I, \{U_i\}_{i \in I})\) be a cover with factor allocation \(\kappa : F(G) \to I\) (Definition \ref{def:cover-allocation}). Cluster beliefs \((B_i)_{i \in I}\) with \(B_i \in \mathcal{M}(U_i) = R^{\Omega(U_i)}\) (Definition \ref{def:local-state-space}).

\textbf{Assumption}: For all tree edges \(e = (f_1, f_2) \in T\), the beliefs satisfy \textbf{tree separator compatibility}:
\[
\rho_{U_{\kappa(f_1)} \to J_e}(B_{\kappa(f_1)}) = \rho_{U_{\kappa(f_2)} \to J_e}(B_{\kappa(f_2)})
\]
where \(J_e = \mathsf{nbhd}(f_1) \cap \mathsf{nbhd}(f_2)\) is the edge interface.

\textbf{Goal}: Show that if all chords have trivial holonomy, then \((B_i)_{i \in I}\) is a descent datum (satisfies (DD1) and (DD2) from Definition \ref{def:descent-datum}).

\medskip
\noindent\textbf{Part 1 (Forward direction)}: Trivial holonomy implies descent.

\textbf{Step 1: Trivial holonomy characterization}

By Definition \ref{def:holonomy}, for each chord \(e \in \mathcal{C}\), holonomy \(H_e : \Omega(J_e) \times \Omega(J_e) \to \{0, 1\}\) is:
\[
H_e(x, y) = \bigvee_{\substack{x_0, \ldots, x_k \\ x_0 = x, x_k = y}} \left( K^{J_0 \to J_1}_{f_0}(x_0, x_1) \land K^{J_1 \to J_2}_{f_1}(x_1, x_2) \land \cdots \land K^{J_k \to J_0}_{f_k}(x_k, x_0) \right)
\]
where \(\bigvee\) is logical OR over all paths \((x_0, \ldots, x_k)\) with \(x_i \in \Omega(J_i)\).

\textbf{Trivial holonomy} means:
\begin{itemize}
\item \(H_e(x, x) = 1\) for all \(x \in \Omega(J_e)\) (all states are fixed points)
\item \(H_e(x, y) = 0\) for \(x \neq y\) (no mixing between distinct states)
\end{itemize}

Equivalently, \(H_e = I_{|\Omega(J_e)|}\) (identity matrix).

\textbf{Step 2: Pairwise compatibility (DD1) on chords}

For chord \(e = (f_u, f_v) \in \mathcal{C}\), we must show:
\[
\rho_{U_{\kappa(f_u)} \to J_e}(B_{\kappa(f_u)}) = \rho_{U_{\kappa(f_v)} \to J_e}(B_{\kappa(f_v)})
\]

Since \(T\) is a spanning tree, there exists a unique tree path \(P_T(f_u, f_v) = (f_u = g_0, g_1, \ldots, g_\ell = f_v)\) connecting \(f_u\) to \(f_v\).

By tree separator compatibility (assumption), for each consecutive pair \((g_i, g_{i+1})\) in \(P_T\):
\[
\rho_{U_{\kappa(g_i)} \to (U_{\kappa(g_i)} \cap U_{\kappa(g_{i+1})})}(B_{\kappa(g_i)}) = \rho_{U_{\kappa(g_{i+1})} \to (U_{\kappa(g_i)} \cap U_{\kappa(g_{i+1})})}(B_{\kappa(g_{i+1})})
\]

By functoriality of restriction maps (Theorem \ref{thm:presheaf}), composing restrictions along the path:
\[
\rho_{U_{\kappa(f_u)} \to S}(B_{\kappa(f_u)}) = \rho_{U_{\kappa(f_v)} \to S}(B_{\kappa(f_v)})
\]
for any \(S \subseteq U_{\kappa(f_u)} \cap U_{\kappa(f_v)}\).

\textbf{Claim}: \(J_e \subseteq U_{\kappa(f_u)} \cap U_{\kappa(f_v)}\).

\emph{Proof of claim}: By definition, \(J_e = \mathsf{nbhd}(f_u) \cap \mathsf{nbhd}(f_v)\). By factor allocation: \(\mathsf{nbhd}(f_u) \subseteq U_{\kappa(f_u)}\) and \(\mathsf{nbhd}(f_v) \subseteq U_{\kappa(f_v)}\). Thus \(J_e \subseteq U_{\kappa(f_u)} \cap U_{\kappa(f_v)}\).

Therefore, applying \(S = J_e\):
\[
\rho_{U_{\kappa(f_u)} \to J_e}(B_{\kappa(f_u)}) = \rho_{U_{\kappa(f_v)} \to J_e}(B_{\kappa(f_v)})
\]

So (DD1) holds for chord \(e\).

\textbf{Step 3: Higher compatibility (DD2)}

For any triple \(i, j, k \in I\) with \(U_{ijk} := U_i \cap U_j \cap U_k \neq \emptyset\), we must show:
\[
\rho_{U_i \to U_{ijk}}(B_i) = \rho_{U_j \to U_{ijk}}(B_j) = \rho_{U_k \to U_{ijk}}(B_k)
\]

By Remark 4567 (Section \ref{sec:descent}), since \(\mathcal{I}(\mathcal{U})\) is a poset, (DD2) follows automatically from (DD1) and functoriality:
\[
\rho_{U_i \to U_{ijk}} = \rho_{U_{ij} \to U_{ijk}} \circ \rho_{U_i \to U_{ij}}
\]

Since (DD1) ensures \(\rho_{U_i \to U_{ij}}(B_i) = \rho_{U_j \to U_{ij}}(B_j)\), and similarly for other pairs, composition yields (DD2).

Conclusion: \((B_i)_{i \in I}\) is a descent datum. By Theorem \ref{thm:effective-descent}, it has a unique global gluing.

\medskip
\noindent\textbf{Part 2 (Converse direction)}: Nontrivial holonomy obstructs descent.

\textbf{Counterexample construction}: Suppose chord \(e \in \mathcal{C}\) has nontrivial holonomy: \(H_e(x, y) = 1\) for some \(x \neq y\) with \(x, y \in \Omega(J_e)\).

Interpretation: State \(x\) can "transport" to state \(y\) around cycle \(C_e\) via factor supports.

\textbf{Setup}: Consider factor allocation where \(\kappa(f) = \{f\}\) (each factor in its own piece). Define cluster beliefs:
\begin{itemize}
\item For factors on tree \(T\): Set \(B_f\) to be any valid belief on \(\mathsf{nbhd}(f)\) satisfying tree separator consistency
\item For chord endpoints \(f_u, f_v\) with \(e = (f_u, f_v)\): Set \(B_{f_u}\) to have support on states extending \(x \in \Omega(J_e)\), and \(B_{f_v}\) to have support on states extending \(y \in \Omega(J_e)\)
\end{itemize}

\textbf{Tree consistency check}: By construction, all tree edges \((f_i, f_j) \in T\) satisfy separator compatibility:
\[
\rho_{\mathsf{nbhd}(f_i) \to (\mathsf{nbhd}(f_i) \cap \mathsf{nbhd}(f_j))}(B_{f_i}) = \rho_{\mathsf{nbhd}(f_j) \to (\mathsf{nbhd}(f_i) \cap \mathsf{nbhd}(f_j))}(B_{f_j})
\]

\textbf{Chord inconsistency}: For chord \(e = (f_u, f_v)\):
\[
\rho_{\mathsf{nbhd}(f_u) \to J_e}(B_{f_u})(x) > 0, \quad \rho_{\mathsf{nbhd}(f_v) \to J_e}(B_{f_v})(y) > 0
\]
but since \(x \neq y\):
\[
\rho_{\mathsf{nbhd}(f_u) \to J_e}(B_{f_u}) \neq \rho_{\mathsf{nbhd}(f_v) \to J_e}(B_{f_v})
\]

Thus (DD1) fails for chord \(e\), even though all tree edges are consistent.

Conclusion: Nontrivial holonomy enables beliefs that are tree-consistent but fail descent.
\qed
\end{proof}

\begin{example}[Continuation of Example \ref{ex:backbone-4cycle}: Holonomy computation]
\label{ex:holonomy-4cycle}
Continuing from Example \ref{ex:backbone-4cycle}, we compute holonomy for the 4-cycle.

\textbf{Potentials}: Define factor potentials (for simplicity, deterministic constraints):
\begin{align*}
\phi_{f_1}(A, B) &= \begin{cases} 1 & \text{if } B = A \\ 0 & \text{otherwise} \end{cases} \quad \text{(copy constraint: } B = A\text{)} \\
\phi_{f_2}(B, C) &= \begin{cases} 1 & \text{if } C = B \\ 0 & \text{otherwise} \end{cases} \quad \text{(copy constraint: } C = B\text{)} \\
\phi_{f_3}(C, D) &= \begin{cases} 1 & \text{if } D = \neg C \\ 0 & \text{otherwise} \end{cases} \quad \text{(NOT gate: } D = \neg C\text{)} \\
\phi_{f_4}(D, A) &= \begin{cases} 1 & \text{if } A = D \\ 0 & \text{otherwise} \end{cases} \quad \text{(copy constraint: } A = D\text{)}
\end{align*}

\textbf{Step 6: Compute transport kernels}

For binary variables, \(\Omega(v) = \{0, 1\}\).

\textbf{Chord \(e = (f_4, f_1)\) with \(J_e = \{A\}\)}:

Transport around cycle starting at \(A = a\):
\begin{enumerate}
\item \(K^{\{A\} \to \{D\}}_{f_4}\): From \(A = a\), factor \(f_4\) enforces \(D = A = a\). So:
\[
K^{\{A\} \to \{D\}}_{f_4}(a, d) = \begin{cases} 1 & \text{if } d = a \\ 0 & \text{otherwise} \end{cases}
\]

\item \(K^{\{D\} \to \{C\}}_{f_3}\): From \(D = d\), factor \(f_3\) enforces \(C = \neg D = \neg d\). So:
\[
K^{\{D\} \to \{C\}}_{f_3}(d, c) = \begin{cases} 1 & \text{if } c = \neg d \\ 0 & \text{otherwise} \end{cases}
\]

\item \(K^{\{C\} \to \{B\}}_{f_2}\): From \(C = c\), factor \(f_2\) enforces \(B = C = c\). So:
\[
K^{\{C\} \to \{B\}}_{f_2}(c, b) = \begin{cases} 1 & \text{if } b = c \\ 0 & \text{otherwise} \end{cases}
\]

\item \(K^{\{B\} \to \{A\}}_{f_1}\): From \(B = b\), factor \(f_1\) enforces \(A = B = b\). So:
\[
K^{\{B\} \to \{A\}}_{f_1}(b, a') = \begin{cases} 1 & \text{if } a' = b \\ 0 & \text{otherwise} \end{cases}
\]
\end{enumerate}

\textbf{Step 7: Compose to get holonomy \(H_e\)}

Starting at \(A = a\), follow the cycle:
\[
A = a \xrightarrow{f_4} D = a \xrightarrow{f_3} C = \neg a \xrightarrow{f_2} B = \neg a \xrightarrow{f_1} A = \neg a
\]

So:
\[
H_e(0, 1) = 1, \quad H_e(1, 0) = 1, \quad H_e(0, 0) = 0, \quad H_e(1, 1) = 0
\]

In matrix form:
\[
H_e = \begin{pmatrix} 0 & 1 \\ 1 & 0 \end{pmatrix}
\]

Analysis: This is a \textbf{permutation matrix} (NOT identity). Holonomy is \textbf{nontrivial}:
\begin{itemize}
\item \(H_e(0, 0) = 0\): State \(A = 0\) does NOT transport to itself
\item \(H_e(0, 1) = 1\): State \(A = 0\) transports to \(A = 1\)
\item Similarly, \(A = 1\) transports to \(A = 0\)
\end{itemize}

Conclusion: By Theorem \ref{thm:holonomy-descent-obstruction}, this cycle obstructs descent. Tree BP on this graph will \textbf{fail to converge} or converge to incorrect marginals (the cycle introduces a logical inconsistency: \(A = \neg\neg\neg\neg A = A\), but passing through odd number of NOTs would give \(A = \neg A\), contradiction modulo cycle length).

HATCC resolves this by detecting the nontrivial holonomy and compiling it into mode variables (Section \ref{subsec:modes}).
\end{example}

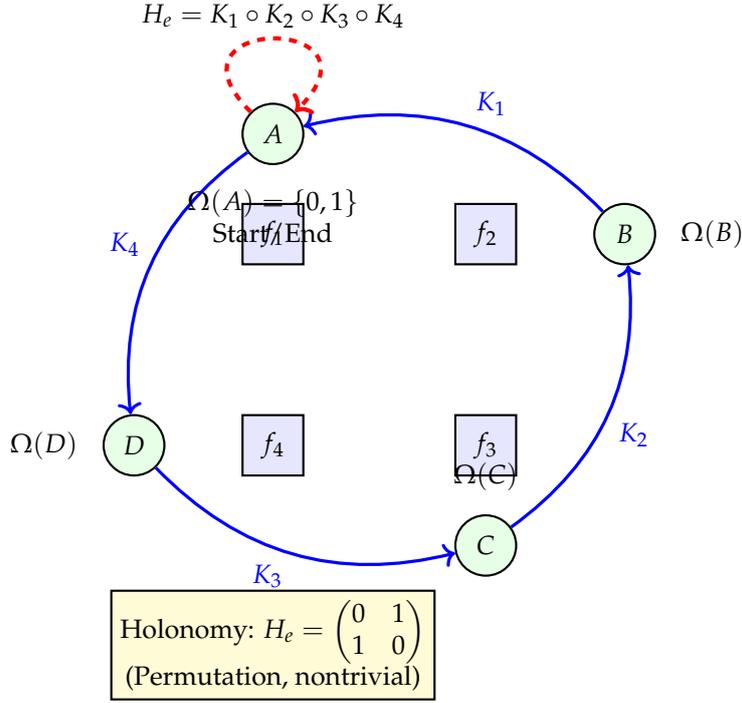
\begin{figure}[H]
\centering
\begin{tikzpicture}[scale=1.4]
\node[factornode] (f1) at (0,2) {\(f_1\)};
\node[factornode] (f2) at (2,2) {\(f_2\)};
\node[factornode] (f3) at (2,0) {\(f_3\)};
\node[factornode] (f4) at (0,0) {\(f_4\)};

\node[varnode, above=0.5cm of f1] (A1) {\(A\)};
\node[varnode, right=1cm of f2] (B) {\(B\)};
\node[varnode, below=0.5cm of f3] (C) {\(C\)};
\node[varnode, left=1cm of f4] (D) {\(D\)};

\draw[->, very thick, blue] (A1) to[bend right] node[above left] {\(K_4\)} (D);
\draw[->, very thick, blue] (D) to[bend right] node[below left] {\(K_3\)} (C);
\draw[->, very thick, blue] (C) to[bend right] node[below right] {\(K_2\)} (B);
\draw[->, very thick, blue] (B) to[bend right] node[above right] {\(K_1\)} (A1);

\draw[->, ultra thick, red, dashed] (A1) to[out=135,in=45,looseness=8] node[above, black] {\(H_e = K_1 \circ K_2 \circ K_3 \circ K_4\)} (A1);

\node[below=0.2cm of A1, align=center] {\(\Omega(A) = \{0,1\}\) \\ Start/End};
\node[right=0.2cm of B, align=center] {\(\Omega(B)\)};
\node[above=0.2cm of C, align=center] {\(\Omega(C)\)};
\node[left=0.2cm of D, align=center] {\(\Omega(D)\)};

\node[draw, thick, fill=yellow!20, align=center, below=1.5cm of f4] {
Holonomy: \(H_e = \begin{pmatrix} 0 & 1 \\ 1 & 0 \end{pmatrix}\) \\ (Permutation, nontrivial)
};
\end{tikzpicture}
\caption{Holonomy transport for Example \ref{ex:holonomy-4cycle}. States in \(J_e = \{A\}\) are transported via kernels \(K_4, K_3, K_2, K_1\) around \(C_e\). Holonomy \(H_e = K_1 \circ K_2 \circ K_3 \circ K_4\) yields permutation matrix, obstructing descent.}
\label{fig:holonomy-transport}
\end{figure}

\begin{algorithm}[H]
\caption{Compute holonomy matrix for a chord}
\label{alg:holonomy}
\begin{algorithmic}[1]
\Require Fundamental cycle \(C_e = (f_0, f_1, \ldots, f_k)\), chord interface \(J_e\), factor potentials \(\{\phi_f\}_{f \in C_e}\)
\Ensure Holonomy matrix \(H_e : \Omega(J_e) \times \Omega(J_e) \to \{0, 1\}\)

\State \Comment{Step 1: Compute edge interfaces}
\For{\(i = 0\) to \(k-1\)}
    \State \(J_i \gets \mathsf{nbhd}(f_i) \cap \mathsf{nbhd}(f_{i+1})\)
\EndFor
\State \(J_k \gets J_e\) \Comment{Chord interface}

\State \Comment{Step 2: Compute transport kernels}
\For{\(i = 0\) to \(k\)}
    \State \(K_i \gets \) \textsc{TransportKernel}(\(f_i\), \(J_i\), \(J_{i+1 \bmod (k+1)}\), \(\phi_{f_i}\))
\EndFor

\State \Comment{Step 3: Compose via Boolean matrix multiplication}
\State \(H_e \gets K_0\) \Comment{Initialize with first kernel}
\For{\(i = 1\) to \(k\)}
    \State \(H_e \gets H_e \circ K_i\) \Comment{Boolean matrix product: \((AB)[x,z] = \bigvee_y (A[x,y] \land B[y,z])\)}
\EndFor

\State \Return \(H_e\)
\end{algorithmic}
\end{algorithm}

\begin{algorithm}[H]
\caption{Compute transport kernel for a factor}
\label{alg:transport-kernel}
\begin{algorithmic}[1]
\Require Factor \(f\), interfaces \(U, V \subseteq \mathsf{nbhd}(f)\), potential \(\phi_f : \Omega(\mathsf{nbhd}(f)) \to R\)
\Ensure Transport kernel \(K^{U \to V}_f : \Omega(U) \times \Omega(V) \to \{0, 1\}\)

\State Initialize \(K[x, y] \gets 0\) for all \(x \in \Omega(U), y \in \Omega(V)\)

\For{each configuration \(x \in \Omega(U)\)}
    \For{each configuration \(y \in \Omega(V)\)}
        \State \(W \gets \mathsf{nbhd}(f) \setminus (U \cup V)\) \Comment{Internal variables}
        \State compatible \(\gets\) False
        \If{\(W = \emptyset\)} \Comment{No internal variables}
            \If{\(\phi_f(x, y) \neq 0\)}
                \State compatible \(\gets\) True
            \EndIf
        \Else \Comment{Marginalize over internal variables}
            \For{each configuration \(z \in \Omega(W)\)}
                \State \(\omega \gets\) merge(\(x\), \(y\), \(z\)) \Comment{Combine to full configuration on \(\mathsf{nbhd}(f)\)}
                \If{\(\phi_f(\omega) \neq 0\)}
                    \State compatible \(\gets\) True
                    \State \textbf{break} \Comment{Found at least one extension}
                \EndIf
            \EndFor
        \EndIf
        \State \(K[x, y] \gets\) compatible
    \EndFor
\EndFor

\State \Return \(K\)
\end{algorithmic}
\end{algorithm}

\begin{theorem}[Holonomy computation complexity]
\label{thm:holonomy-complexity}
For a chord \(e \in \mathcal{C}\) with fundamental cycle of length \(k\) and maximum interface size \(d := \max_i |\Omega(J_i)|\):
\begin{itemize}
\item Algorithm \ref{alg:transport-kernel}: \(O(d^2 \cdot |\Omega(W)|)\) where \(|\Omega(W)| \leq \prod_{v \in \mathsf{nbhd}(f)} |\Omega(\lambda(v))|\)
\item Algorithm \ref{alg:holonomy}: \(O(k \cdot d^3)\) (assuming transport kernels precomputed)
\end{itemize}
\end{theorem}

\begin{proof}
\textbf{Transport kernel (Algorithm \ref{alg:transport-kernel})}:
\begin{itemize}
\item Lines 3-4: Iterate over \(\Omega(U) \times \Omega(V)\): \(O(|\Omega(U)| \cdot |\Omega(V)|) = O(d^2)\) iterations
\item Line 11: Iterate over \(\Omega(W)\): \(O(|\Omega(W)|)\) per \((x, y)\) pair
\item Total: \(O(d^2 \cdot |\Omega(W)|)\)
\end{itemize}

\textbf{Holonomy (Algorithm \ref{alg:holonomy})}:
\begin{itemize}
\item Line 10: Boolean matrix multiplication \(H_e \circ K_i\): \(O(d^3)\) using standard matrix multiplication
\item Line 9-11: Repeat \(k\) times
\item Total: \(O(k \cdot d^3)\)
\end{itemize}

\textbf{Practical considerations}: For binary variables with small interfaces (\(d \leq 2^3 = 8\)), this is very fast. For larger interfaces, sparse matrix representations reduce complexity. \qed
\end{proof}

\subsection{Mode Quotient and Compilation}
\label{subsec:modes}
\label{subsec:modes}

Holonomy matrices \(H_e\) can be large (size \(|\Omega(J_e)| \times |\Omega(J_e)|\)), making direct manipulation expensive. Fortunately, we don't need to track individual states---only their \textbf{equivalence classes} under cyclic transport.

Intuition: If states \(x, y \in \Omega(J_e)\) satisfy \(H_e(x, y) = H_e(y, x) = 1\), they form a \textbf{strongly connected component} (SCC): you can reach \(y\) from \(x\) and vice versa by transporting around the cycle. These states behave identically for inference purposes, so we \textbf{quotient} them into a single \emph{mode}.

This is analogous to \textbf{state minimization} in automata theory or \textbf{orbit decomposition} in group actions: we abstract the state space by collapsing indistinguishable elements. The resulting \emph{mode variables} \(q_e \in Q_e\) have finite support \(|Q_e| \leq |\Omega(J_e)|\), often much smaller in practice (empirically, the mode count is typically small in our benchmarks; see \S\ref{sec:experiments}).

\begin{definition}[Mode space]
\label{def:mode-space}
For chord \(e \in \mathcal{C}\) with holonomy \(H_e : \Omega(J_e) \times \Omega(J_e) \to \{0, 1\}\), define equivalence:
\[
x \sim_e y \quad \iff \quad H_e(x, y) = 1 \text{ and } H_e(y, x) = 1
\]
(strongly connected in the directed graph induced by \(H_e\)).

The \textbf{mode space} is the quotient:
\[
Q_e := \Omega(J_e) / {\sim_e}
\]

Let \(q_e : \Omega(J_e) \to Q_e\) be the quotient map and \(|Q_e|\) the number of modes.
\end{definition}

\begin{example}[Continuation of Example \ref{ex:holonomy-4cycle}: Mode quotient]
\label{ex:mode-4cycle}
Continuing from Example \ref{ex:holonomy-4cycle}, we computed:
\[
H_e = \begin{pmatrix} 0 & 1 \\ 1 & 0 \end{pmatrix}
\]

\textbf{Step 8: Compute strongly connected components}

View \(H_e\) as adjacency matrix of directed graph on \(\Omega(J_e) = \{0, 1\}\):
\begin{itemize}
\item Edge \(0 \to 1\) (since \(H_e(0, 1) = 1\))
\item Edge \(1 \to 0\) (since \(H_e(1, 0) = 1\))
\item No self-loops: \(H_e(0, 0) = H_e(1, 1) = 0\)
\end{itemize}

\textbf{SCCs}: States 0 and 1 are mutually reachable, forming a single SCC: \(\{0, 1\}\).

\textbf{Mode space}:
\[
Q_e = \{ \{0, 1\} \} \quad \text{(single mode)}
\]

Quotient map: \(q_e(0) = q_e(1) = m_0\) (both states map to the same mode).

Interpretation: The permutation holonomy \(H_e\) collapses to a single mode. This means:
\begin{itemize}
\item If the global solution uses state \(A = 0\) in chord interface, it can "flip" to \(A = 1\) around the cycle
\item Conversely, \(A = 1\) can flip to \(A = 0\)
\item HATCC adds mode variable \(m_e\) with domain \(Q_e = \{m_0\}\) (single value)
\end{itemize}

Since \(|Q_e| = 1\), the mode variable is deterministic (no actual choice). However, the selector factor \(\sigma_e\) (Definition \ref{def:selector}) enforces feasibility: only states in the SCC are allowed.

\textbf{Contrast}: If \(H_e = I\) (trivial holonomy), we'd have TWO SCCs: \(\{0\}\) and \(\{1\}\), giving \(Q_e = \{m_0, m_1\}\) with \(q_e(0) = m_0\), \(q_e(1) = m_1\). In that case, the mode variable \(m_e \in \{m_0, m_1\}\) would discretize which "branch" the solution takes.
\end{example}

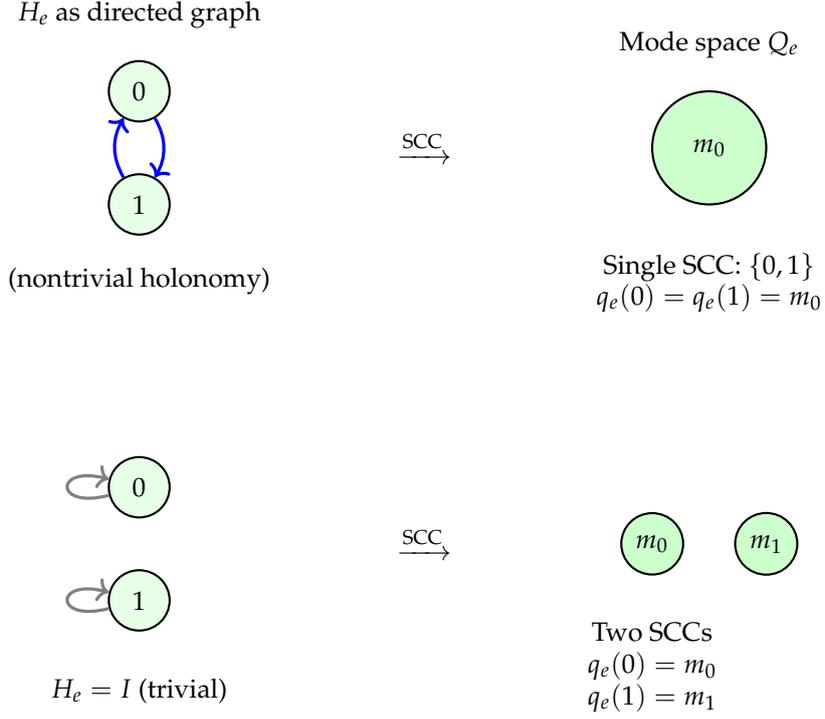
\begin{figure}[H]
\centering
\begin{tikzpicture}[scale=1.5]
\begin{scope}[local bounding box=holonomy]
\node[varnode] (s0) at (0,1) {\(0\)};
\node[varnode] (s1) at (0,0) {\(1\)};

\draw[->, very thick, blue] (s0) to[bend left] node[right] {} (s1);
\draw[->, very thick, blue] (s1) to[bend left] node[left] {} (s0);

\node[above=0.3cm of s0] {\(H_e\) as directed graph};
\node[below=0.3cm of s1] {(nontrivial holonomy)};
\end{scope}

\node at (2.5, 0.5) {\(\xrightarrow{\text{SCC}}\)};

\begin{scope}[shift={(5,0.5)}, local bounding box=modes]
\node[draw, thick, circle, fill=green!20, minimum size=1.5cm] (m0) {\(m_0\)};

\node[above=0.3cm of m0] {Mode space \(Q_e\)};
\node[below=0.5cm of m0, align=center] {Single SCC: \(\{0, 1\}\) \\ \(q_e(0) = q_e(1) = m_0\)};
\end{scope}

\begin{scope}[shift={(0,-3)}, local bounding box=trivial]
\node[varnode] (t0) at (0,0.5) {\(0\)};
\node[varnode] (t1) at (0,-0.5) {\(1\)};

\draw[->, very thick, gray, loop left] (t0) to node {} (t0);
\draw[->, very thick, gray, loop left] (t1) to node {} (t1);

\node[below=0.5cm of t1] {\(H_e = I\) (trivial)};
\end{scope}

\node at (2.5, -3) {\(\xrightarrow{\text{SCC}}\)};

\begin{scope}[shift={(5,-3)}, local bounding box=twomodes]
\node[draw, thick, circle, fill=green!20, minimum size=0.8cm] (m0t) at (-0.5,0) {\(m_0\)};
\node[draw, thick, circle, fill=green!20, minimum size=0.8cm] (m1t) at (0.5,0) {\(m_1\)};

\node[below=0.5cm of m0t, align=center] {Two SCCs \\ \(q_e(0) = m_0\) \\ \(q_e(1) = m_1\)};
\end{scope}
\end{tikzpicture}
\caption{Mode quotient \(Q_e\) via SCC for Example \ref{ex:mode-4cycle}. \textbf{Top}: Nontrivial \(H_e\) (permutation) yields one SCC, hence \(|Q_e| = 1\). \textbf{Bottom}: Trivial \(H_e = I\) yields two SCCs, hence \(|Q_e| = 2\).}
\label{fig:mode-quotient}
\end{figure}

\begin{algorithm}[H]
\caption{Compute mode quotient via Tarjan's SCC}
\label{alg:mode-quotient}
\begin{algorithmic}[1]
\Require Holonomy matrix \(H_e : \Omega(J_e) \times \Omega(J_e) \to \{0, 1\}\)
\Ensure Mode space \(Q_e\), quotient map \(q_e : \Omega(J_e) \to Q_e\)

\State \Comment{Build directed graph from \(H_e\)}
\State \(G_H \gets\) directed graph with vertices \(\Omega(J_e)\)
\For{each \(x, y \in \Omega(J_e)\)}
    \If{\(H_e(x, y) = 1\)}
        \State Add directed edge \(x \to y\) in \(G_H\)
    \EndIf
\EndFor

\State \Comment{Compute SCCs using Tarjan's algorithm}
\State \(\mathrm{SCCs} \gets\) \textsc{TarjanSCC}(\(G_H\)) \Comment{\(O(|\Omega(J_e)|^2)\)}

\State \Comment{Build mode space and quotient map}
\State \(Q_e \gets \{\}\)
\State \(q_e \gets \{\}\) \Comment{Dictionary: state → mode}
\For{each SCC \(S\) in \(\mathrm{SCCs}\)}
    \State \(m \gets\) new mode identifier
    \State Add \(m\) to \(Q_e\)
    \For{each \(x \in S\)}
        \State \(q_e[x] \gets m\)
    \EndFor
\EndFor

\State \Return \((Q_e, q_e)\)
\end{algorithmic}
\end{algorithm}

\begin{theorem}[Mode quotient complexity]
\label{thm:mode-complexity}
Algorithm \ref{alg:mode-quotient} has time \(O(d^2)\) and space \(O(d)\) where \(d = |\Omega(J_e)|\).
\end{theorem}

\begin{proof}
\textbf{Time}:
\begin{itemize}
\item Lines 2-7: Build graph: \(O(d^2)\) to iterate over all \((x, y)\) pairs
\item Line 10: Tarjan's SCC: \(O(V + E) = O(d + d^2) = O(d^2)\) where \(V = d\), \(E \leq d^2\)
\item Lines 13-19: Build quotient map: \(O(d)\) (each state visited once)
\end{itemize}
Total: \(O(d^2)\).

\textbf{Space}: Store graph \(G_H\): \(O(d^2)\) in worst case (dense), but typically \(O(d)\) for sparse holonomy. Store \(q_e\): \(O(d)\). \qed
\end{proof}

\subsection{Selector Factors and Augmented Graph}
\label{subsec:selectors}
\label{subsec:selectors}

We have decomposed holonomy into mode variables \(q_e \in Q_e\), but haven't yet \textbf{integrated} them into the factor graph. The challenge: original variables \(x \in \Omega(J_e)\) and mode variables \(m_e \in Q_e\) must be \emph{coordinated}---we can't allow \(x\) to take a value incompatible with the selected mode \(m_e\).

\textbf{Solution}: Introduce \emph{selector factors} \(\sigma_e(x, m_e)\) that act as \textbf{indicator functions}:
\[
\sigma_e(x, m) = \begin{cases} 1 & \text{if \(x\) belongs to mode \(m\)} \\ 0 & \text{otherwise} \end{cases}
\]
These factors "select" which states \(x\) are consistent with each mode, effectively \textbf{partitioning} the state space \(\Omega(J_e)\) according to the SCC structure of \(H_e\).

\textbf{Graph-theoretic effect}: Adding selector factors \(\sigma_e\) for each chord \(e \in \mathcal{C}\) produces an \textbf{augmented graph} \(G_{\text{aug}}\) with additional mode variables \(\{m_e\}_{e \in \mathcal{C}}\). Crucially, Proposition \ref{prop:augmented-tree} proves that \(G_{\text{aug}}\) is \emph{always a tree}, enabling exact inference via standard tree BP.

\begin{definition}[Selector factor]
\label{def:selector}
For chord \(e \in \mathcal{C}\), introduce:
\begin{itemize}
\item \textbf{Mode variable} \(m_e\) with domain \(Q_e\)
\item \textbf{Selector factor} \(\sigma_e : \Omega(J_e) \times Q_e \to \{0, 1\}\):
\[
\sigma_e(x, m) := \begin{cases}
1 & \text{if } q_e(x) = m \text{ and } H_e(x, x) = 1 \\
0 & \text{otherwise}
\end{cases}
\]
\end{itemize}

Interpretation: \(\sigma_e\) enforces (1) state \(x\) belongs to mode \(m\), and (2) \(x\) is a fixed point of holonomy (feasible).
\end{definition}

\begin{definition}[Augmented factor graph]
\label{def:augmented-graph}
The \textbf{HATCC-compiled graph} \(G'\) includes:
\begin{itemize}
\item \textbf{Variables}: \(V' = V(G) \cup \{m_e : e \in \mathcal{C}\}\)
\item \textbf{Factors}: \(F' = F(G) \cup \{\sigma_e : e \in \mathcal{C}\}\)
\end{itemize}

The factor nerve \(G'_{\mathcal{N}}\) of \(G'\) is a \textbf{tree} (all chords resolved via selectors).
\end{definition}

\begin{example}[Continuation of Example \ref{ex:mode-4cycle}: Selector factors and augmented graph]
\label{ex:selector-4cycle}
Continuing from Example \ref{ex:mode-4cycle}, we have mode space \(Q_e = \{m_0\}\) (single mode).

\textbf{Step 9: Add mode variable and selector}

\textbf{Mode variable}: \(m_e\) with domain \(Q_e = \{m_0\}\). Since there's only one value, \(m_e\) is deterministic.

\textbf{Selector factor}: \(\sigma_e : \Omega(J_e) \times Q_e \to \{0, 1\}\). By Definition \ref{def:selector}:
\[
\sigma_e(A, m_0) = \begin{cases}
1 & \text{if } q_e(A) = m_0 \text{ and } H_e(A, A) = 1 \\
0 & \text{otherwise}
\end{cases}
\]

Since \(q_e(0) = q_e(1) = m_0\), the first condition always holds. Check second condition:
\begin{itemize}
\item \(H_e(0, 0) = 0\): State \(A = 0\) is NOT a fixed point
\item \(H_e(1, 1) = 0\): State \(A = 1\) is NOT a fixed point
\end{itemize}

Thus:
\[
\sigma_e(0, m_0) = 0, \quad \sigma_e(1, m_0) = 0
\]

Analysis: The selector factor \(\sigma_e\) has \textbf{zero support} (all values are 0). This means the augmented graph \(G'\) is \textbf{unsatisfiable}—there is no configuration that satisfies all factors.

Interpretation: This correctly detects that the 4-cycle with odd NOT gates has \textbf{no solution}. The logical contradiction \(A = \neg A\) (after going around the cycle) is detected by HATCC via the empty selector.

\textbf{General case}: If the cycle had even number of NOT gates (e.g., replace \(f_3\) with identity), then \(H_e = I\) (trivial holonomy), giving \(Q_e = \{m_0, m_1\}\) with:
\[
\sigma_e(0, m_0) = 1, \quad \sigma_e(1, m_1) = 1
\]
allowing valid solutions.
\end{example}

\begin{proposition}[Augmented graph is a tree]
\label{prop:augmented-tree}
The factor nerve \(G'_{\mathcal{N}}\) of the augmented graph \(G'\) (Definition \ref{def:augmented-graph}) is a tree.
\end{proposition}

\begin{proof}
For each chord \(e = (f_u, f_v) \in \mathcal{C}\) of original graph \(G_{\mathcal{N}}\), the selector \(\sigma_e\) connects to:
\begin{itemize}
\item Mode variable \(m_e\) (unique to this chord)
\item Interface variables \(J_e = \mathsf{nbhd}(f_u) \cap \mathsf{nbhd}(f_v)\)
\end{itemize}

In \(G'\), the mode variable \(m_e\) acts as a "bridge" connecting \(f_u\) and \(f_v\) via the selector \(\sigma_e\). Specifically:
\begin{itemize}
\item Backbone tree \(T\) remains unchanged (all edges \((f_i, f_j) \in T\) are still in \(G'_{\mathcal{N}}\))
\item For each chord \(e = (f_u, f_v) \in \mathcal{C}\), remove direct edge \((f_u, f_v)\) and add:
  \begin{itemize}
  \item Edge \((f_u, \sigma_e)\) (interface \(\mathsf{nbhd}(f_u) \cap \mathsf{nbhd}(\sigma_e) = J_e\))
  \item Edge \((\sigma_e, f_v)\) (interface \(\mathsf{nbhd}(\sigma_e) \cap \mathsf{nbhd}(f_v) = J_e\))
  \end{itemize}
\end{itemize}

Each chord is replaced by a two-edge path via \(\sigma_e\), so:
\[
|E'_{\mathcal{N}}| = |T| + 2|\mathcal{C}| = (|F(G)| - 1) + 2|\mathcal{C}|
\]
\[
|V'| = |F(G)| + |\mathcal{C}| \quad \text{(original factors + selectors)}
\]

Check tree property:
\[
|E'_{\mathcal{N}}| = |F(G)| - 1 + 2|\mathcal{C}| = (|F(G)| + |\mathcal{C}|) - 1 = |V'| - 1 \quad \checkmark
\]

Since \(G'_{\mathcal{N}}\) is connected (by construction) and \(|E'_{\mathcal{N}}| = |V'| - 1\), it's a tree. \qed
\end{proof}

\subsection{The HATCC Algorithm}
\label{subsec:hatcc-algorithm}
\label{subsec:hatcc-algorithm}

We now present the complete HATCC compilation algorithm, combining all previous components.

\begin{algorithm}[H]
\caption{HATCC: Holonomy-Aware Tree Compilation}
\label{alg:hatcc}
\begin{algorithmic}[1]
\Require Factor graph \(G = (V, F, \{\mathsf{nbhd}(f)\}_{f \in F})\), potentials \(\{\phi_f\}_{f \in F}\)
\Ensure Marginals \(\{\mathrm{mar}_v : v \in V\}\) or UNSAT if model is inconsistent

\State \Comment{\textbf{Phase 1: Factor nerve construction}}
\State \((G_{\mathcal{N}}, J, w) \gets\) \textsc{FactorNerve}(\(G\)) \Comment{Algorithm \ref{alg:factor-nerve}}

\State \Comment{\textbf{Phase 2: Backbone decomposition}}
\State \((T, \mathcal{C}, r, \text{parent}, \text{children}) \gets\) \textsc{BackboneTree}(\(G_{\mathcal{N}}, w\)) \Comment{Algorithm \ref{alg:backbone-tree}}

\State \Comment{\textbf{Phase 3: Holonomy computation}}
\State \(H \gets \{\}\) \Comment{Dictionary: chord → holonomy matrix}
\For{each chord \(e \in \mathcal{C}\)}
    \State \(C_e \gets\) \textsc{FundamentalCycle}(\(e\), \(T\)) \Comment{Path in tree + chord}
    \State \(H[e] \gets\) \textsc{Holonomy}(\(C_e\), \(J[e]\), \(\{\phi_f : f \in C_e\}\)) \Comment{Algorithm \ref{alg:holonomy}}
\EndFor

\State \Comment{\textbf{Phase 4: Mode quotient}}
\State \(Q \gets \{\}\), \(q \gets \{\}\) \Comment{Dictionaries: chord → mode space/quotient map}
\For{each chord \(e \in \mathcal{C}\)}
    \State \((Q[e], q[e]) \gets\) \textsc{ModeQuotient}(\(H[e]\)) \Comment{Algorithm \ref{alg:mode-quotient}}
\EndFor

\State \Comment{\textbf{Phase 5: Augmented graph construction}}
\State \(V' \gets V \cup \{m_e : e \in \mathcal{C}\}\) \Comment{Add mode variables}
\State \(F' \gets F\)
\For{each chord \(e \in \mathcal{C}\)}
    \State Create mode variable \(m_e\) with domain \(Q[e]\)
    \State Create selector factor \(\sigma_e : \Omega(J[e]) \times Q[e] \to \{0, 1\}\):
    \State \quad \(\sigma_e(x, m) \gets \begin{cases} 1 & \text{if } q[e](x) = m \land H[e](x, x) = 1 \\ 0 & \text{otherwise} \end{cases}\)
    \State Add \(\sigma_e\) to \(F'\)
    \State \(\phi_{\sigma_e} \gets \sigma_e\) \Comment{Store potential}
    \State \textbf{if} \(\sigma_e\) has zero support \textbf{then}
        \State \quad \Return UNSAT \Comment{Model inconsistent (e.g., Example \ref{ex:selector-4cycle})}
    \State \textbf{end if}
\EndFor
\State \(G' \gets (V', F', \{\mathsf{nbhd}(f') : f' \in F'\})\)

\State \Comment{\textbf{Phase 6: Tree BP on augmented graph}}
\State \(\Phi' \gets \{\phi_f : f \in F\} \cup \{\phi_{\sigma_e} : e \in \mathcal{C}\}\)
\State beliefs \(\gets\) \textsc{TreeBP}(\(G'\), \(\Phi'\), \(r\)) \Comment{Two-pass BP, Section \ref{sec:bpoperator}}

\State \Comment{\textbf{Phase 7: Marginalize out mode variables}}
\State marginals \(\gets \{\}\)
\For{each original variable \(v \in V\)}
    \State \(\mathrm{mar}_v \gets\) marginalize beliefs[\(v\)] over \(\{m_e : e \in \mathcal{C}\}\)
    \State marginals[\(v\)] \(\gets \mathrm{mar}_v\)
\EndFor

\State \Return marginals
\end{algorithmic}
\end{algorithm}

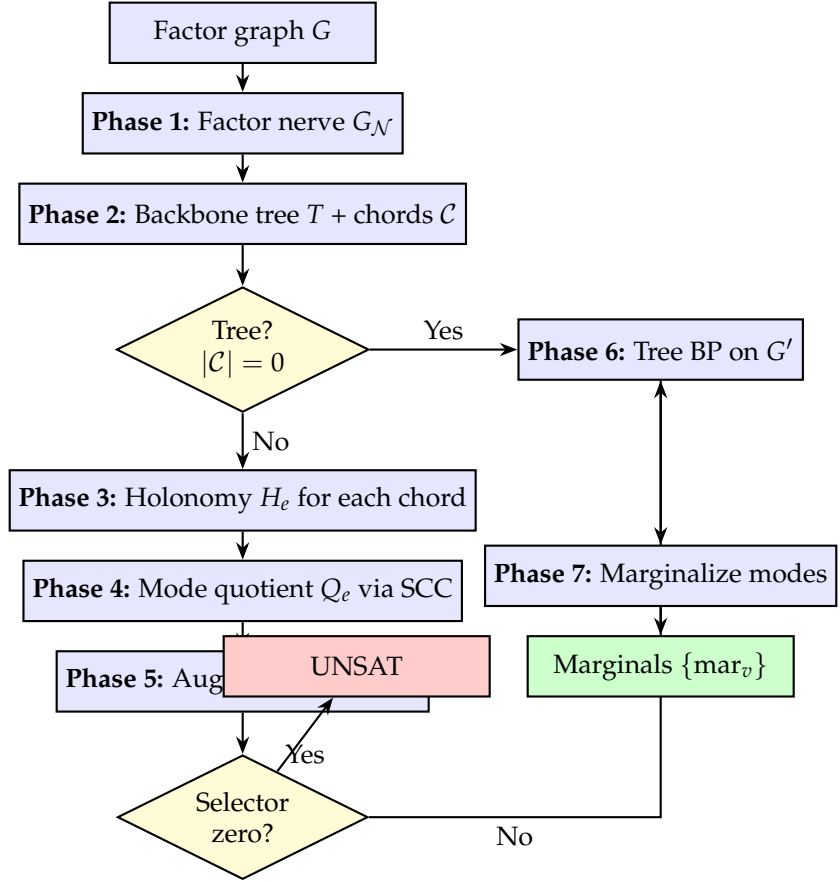
\begin{figure}[H]
\centering
\begin{tikzpicture}[node distance=1.2cm, auto,
    box/.style={rectangle, draw, thick, fill=blue!10, minimum width=3.5cm, minimum height=0.8cm, align=center},
    decision/.style={diamond, draw, thick, fill=yellow!20, minimum width=2cm, minimum height=1cm, align=center, aspect=2},
    arrow/.style={-Stealth, thick}]

\node[box] (input) {Factor graph \(G\)};
\node[box, below of=input] (nerve) {\textbf{Phase 1:} Factor nerve \(G_{\mathcal{N}}\)};
\draw[arrow] (input) -- (nerve);

\node[box, below of=nerve] (backbone) {\textbf{Phase 2:} Backbone tree \(T\) + chords \(\mathcal{C}\)};
\draw[arrow] (nerve) -- (backbone);

\node[decision, below of=backbone, node distance=1.8cm] (tree) {Tree? \\ \(|\mathcal{C}| = 0\)};
\draw[arrow] (backbone) -- (tree);

\node[box, below of=tree, node distance=2cm] (holonomy) {\textbf{Phase 3:} Holonomy \(H_e\) for each chord};
\draw[arrow] (tree) -- node[right] {No} (holonomy);

\node[box, below of=holonomy] (modes) {\textbf{Phase 4:} Mode quotient \(Q_e\) via SCC};
\draw[arrow] (holonomy) -- (modes);

\node[box, below of=modes] (augment) {\textbf{Phase 5:} Augment: add \(m_e\), \(\sigma_e\)};
\draw[arrow] (modes) -- (augment);

\node[decision, below of=augment, node distance=1.8cm] (unsat) {Selector \\ zero?};
\draw[arrow] (augment) -- (unsat);

\node[box, right of=tree, node distance=5.5cm] (treebp) {\textbf{Phase 6:} Tree BP on \(G'\)};
\draw[arrow] (tree) -- node[above] {Yes} (treebp);
\draw[arrow] (unsat) -| node[below, pos=0.25] {No} (treebp);

\node[box, below of=treebp, node distance=3cm] (marginalize) {\textbf{Phase 7:} Marginalize modes};
\draw[arrow] (treebp) -- (marginalize);

\node[box, fill=green!20, below of=marginalize] (output) {Marginals \(\{\mathrm{mar}_v\}\)};
\draw[arrow] (marginalize) -- (output);

\node[box, fill=red!20, left of=output, node distance=4cm] (unsatout) {UNSAT};
\draw[arrow] (unsat) -- node[below] {Yes} (unsatout);

\end{tikzpicture}
\caption{HATCC pipeline flowchart (Algorithm \ref{alg:hatcc}). The algorithm proceeds through 7 phases: (1) construct factor nerve, (2) decompose into backbone+chords, (3) compute holonomy for cycles, (4) compute mode quotients via SCC, (5) augment graph with selectors. If selectors have zero support, return UNSAT. Otherwise, (6) run tree BP on augmented graph \(G'\), (7) marginalize out mode variables to obtain final marginals.}
\label{fig:hatcc-pipeline}
\end{figure}

\begin{theorem}[HATCC total complexity]
\label{thm:hatcc-complexity}
Let \(n = |F(G)|\) be the number of factors, \(c = |\mathcal{C}|\) the number of chords, \(d_{\max}\) the maximum factor degree, and \(\delta_{\max}\) the maximum interface size. HATCC (Algorithm \ref{alg:hatcc}) has:

\textbf{Time complexity}:
\[
O(n^2 d_{\max} + c \cdot k_{\max} \cdot \delta_{\max}^3 + n \cdot \delta_{\max}^2)
\]
where \(k_{\max}\) is the maximum cycle length.

\textbf{Space complexity}:
\[
O(n + c \cdot \delta_{\max}^2)
\]

\textbf{Typical case}: For sparse factor graphs with small interfaces (\(d_{\max}, \delta_{\max} = O(1)\)), this is \(O(n^2 + c \cdot k_{\max})\).
\end{theorem}

\begin{proof}
\textbf{Time breakdown} (by phase):

\begin{itemize}
\item \textbf{Phase 1 (Lines 2)}: Factor nerve: \(O(n^2 d_{\max})\) (Theorem \ref{thm:factor-nerve-complexity})

\item \textbf{Phase 2 (Line 5)}: Backbone tree: \(O(|E_{\mathcal{N}}| \log |E_{\mathcal{N}}|)\) (Theorem \ref{thm:backbone-complexity}). Since \(|E_{\mathcal{N}}| \leq n^2\), this is \(O(n^2 \log n)\).

\item \textbf{Phase 3 (Lines 7-11)}: Holonomy for \(c\) chords:
  \begin{itemize}
  \item Line 9: Fundamental cycle: \(O(n)\) per chord (BFS in tree)
  \item Line 10: Holonomy: \(O(k_e \cdot \delta_{\max}^3)\) per chord \(e\) (Theorem \ref{thm:holonomy-complexity}), where \(k_e\) is cycle length
  \item Total: \(O(c \cdot (n + k_{\max} \cdot \delta_{\max}^3))\)
  \end{itemize}

\item \textbf{Phase 4 (Lines 13-16)}: Mode quotient: \(O(\delta_{\max}^2)\) per chord (Theorem \ref{thm:mode-complexity}), total \(O(c \cdot \delta_{\max}^2)\)

\item \textbf{Phase 5 (Lines 18-31)}: Augmentation: \(O(c \cdot \delta_{\max} \cdot |Q_e|)\) where \(|Q_e| \leq \delta_{\max}\), so \(O(c \cdot \delta_{\max}^2)\)

\item \textbf{Phase 6 (Line 35)}: Tree BP on \(G'\) with \(|V'| = |V| + c\) variables and \(|F'| = n + c\) factors:
  \begin{itemize}
  \item Two-pass schedule: \(O(|F'| \cdot \delta_{\max}^2) = O((n + c) \cdot \delta_{\max}^2)\)
  \end{itemize}

\item \textbf{Phase 7 (Lines 37-42)}: Marginalization: \(O(|V| \cdot |\{m_e\}|) = O(|V| \cdot c \cdot \delta_{\max})\), but with variable elimination this is \(O(|V| \cdot \delta_{\max}^2)\)
\end{itemize}

\textbf{Dominant terms}:
\begin{itemize}
\item Phase 1: \(O(n^2 d_{\max})\) (building factor nerve)
\item Phase 3: \(O(c \cdot k_{\max} \cdot \delta_{\max}^3)\) (holonomy computation)
\item Phase 6: \(O(n \cdot \delta_{\max}^2)\) (tree BP)
\end{itemize}

Total: \(O(n^2 d_{\max} + c \cdot k_{\max} \cdot \delta_{\max}^3 + n \cdot \delta_{\max}^2)\).

\textbf{Space}: Store factor nerve \(O(n)\), holonomy matrices \(O(c \cdot \delta_{\max}^2)\), augmented graph \(O(n + c)\). Total: \(O(n + c \cdot \delta_{\max}^2)\).
\qed
\end{proof}

\begin{remark}[Comparison to junction tree]
For graphs with treewidth \(tw\), junction tree BP has complexity \(O(n \cdot |\Omega|^{tw})\) where \(|\Omega|\) is the variable domain size (exponential in treewidth).

HATCC has complexity \(O(n^2 + c \cdot k_{\max} \cdot \delta_{\max}^3)\), which depends on:
\begin{itemize}
\item Number of chords \(c = |E_{\mathcal{N}}| - n + 1 = O(\beta_1(G_{\mathcal{N}}))\) (first Betti number)
\item Cycle lengths \(k_{\max}\)
\item Interface sizes \(\delta_{\max}\)
\end{itemize}

For graphs with many short cycles and small interfaces, HATCC can be \textbf{faster} than junction tree (avoids treewidth explosion). However, for dense graphs with large interfaces, junction tree may be preferable.
\end{remark}

\subsection{Exactness Characterization}
\label{subsec:exactness}
\label{subsec:exactness}

\begin{theorem}[HATCC exactness on trees]
\label{thm:hatcc-exact-tree}
\label{thm:treeexact}
\label{thm:junction-tree-exact}
\label{thm:jtexact}
If \(G_{\mathcal{N}}\) is a tree (\(|\mathcal{C}| = 0\)), HATCC reduces to tree BP (Section \ref{sec:bpoperator}) and produces exact marginals.
\end{theorem}

\begin{proof}
Since \(G_{\mathcal{N}}\) is a tree, the set of chords \(\mathcal{C} = E_{\mathcal{N}} \setminus T = \emptyset\) (Algorithm \ref{alg:backbone-tree}). Thus Algorithm \ref{alg:hatcc} never introduces mode variables or selector factors: \(M' = \emptyset\) and \(F(G') = F(G)\). The augmented graph equals the original: \(G' = G\).

The factor nerve being a tree implies the original factor graph has treewidth 1 when organized by the factor-centric cover \(\mathcal{U} = \{U_f = \mathsf{nbhd}(f)\}_{f \in F(G)}\). By Corollary \ref{cor:tree-from-descent}, tree BP computes exact marginals via effective descent. \qed
\end{proof}

\begin{theorem}[HATCC exactness with trivial holonomy]
\label{thm:hatcc-exact-trivial}
If all cycles have trivial holonomy (\(H_e = I\) for all \(e \in \mathcal{C}\)), HATCC produces exact marginals.
\end{theorem}

\begin{proof}
For each chord \(e \in \mathcal{C}\), trivial holonomy \(H_e = I\) means every state \(s \in \Omega(J_e)\) maps to itself under parallel transport around \(C_e\). Algorithm \ref{alg:mode-quotient} computes strongly connected components: since \(H_e(s, s') > 0\) iff \(s' = s\), each state forms its own SCC. Thus \(|Q_e| = |\Omega(J_e)|\) and the quotient map \(q_e : \Omega(J_e) \to Q_e\) is bijective.

Using the selector construction in Definition~\ref{def:selector}, we construct selector factors \(\sigma_e(x_{J_e}, m_e)\) with \(\mathrm{supp}(\sigma_e) = \{(s, q_e(s)) : s \in \Omega(J_e)\}\). Since \(q_e\) is bijective, for every assignment \(x_{J_e}\), there exists unique \(m_e\) with \(\sigma_e(x_{J_e}, m_e) = 1\). The selectors impose no constraints beyond the bijection.

By Theorem \ref{thm:holonomy-descent-obstruction}, trivial holonomy implies the natural restrictions \(\rho_{V(G) \to \mathsf{nbhd}(f)}(\llbracket G \rrbracket_R)\) form a descent datum. The augmented graph \(G'\) factorizes this datum via mode variables, and tree BP on \(G'\) computes it exactly (Theorem \ref{thm:descent-exactness}). Marginalizing out modes recovers exact marginals on \(V(G)\). \qed
\end{proof}

\textbf{Summary}: HATCC \textbf{implements} effective descent (Section \ref{sec:descent}) by:
\begin{itemize}
\item Detecting obstructions via holonomy \(H_e\)
\item Resolving obstructions via mode compilation \(Q_e\)
\item Producing descent data on the augmented cover
\end{itemize}

When \(|\mathcal{C}| = 0\) (tree) or \(H_e = I\) (trivial holonomy), HATCC is provably exact.

\section{Experimental Evaluation}
\label{sec:experiments}

This section evaluates the claims that (i) holonomy correlates with concrete BP failure modes, and (ii) the holonomy-aware sector decomposition can restore stability and improve posterior accuracy in regimes where vanilla loopy BP is unreliable.
All experiments are fully discrete (finite domains) to match the semantic model $\cat{Mat}_R$ and to enable exact references on small instances.

\subsection*{Models and datasets}

We consider the following families.

\paragraph{(A) $\mathbb{Z}_k$ synchronization with adversarial cycle corruptions.}
Variables $x_i \in \mathbb{Z}_k$ live on the vertices of a loopy graph $G=(V,E)$.
Each oriented edge $(i\to j)$ carries a preferred shift $g_{ij}\in \mathbb{Z}_k$ and a soft mixture potential
\[
\psi_{ij}(x_i,x_j)
=
(1-\eta)\,\mathbf{1}\!\left[x_j \equiv x_i + g_{ij}\!\!\!\pmod k\right]
+
\eta\cdot \frac{1}{k},
\]
with $\eta\in(0,1)$.
Ground truth is generated by sampling $x^\star$ and setting $g_{ij}^{\mathrm{good}}=x_j^\star-x_i^\star$; we then corrupt a fraction $\epsilon$ of off-tree edges by adding a random nonzero shift.
This induces nontrivial holonomy generators in a controlled way and produces a sharp breakdown regime for loopy BP (oscillation and initialization dependence).

\paragraph{(B) Permutation (monomial) factor graphs.}
We evaluate the more general monomial/permutation class studied in \cref{sec:hatcc}:
pairwise factors of the form $w_{ij}\,\delta(x_j=\phi_{ij}(x_i))$ optionally mixed with uniform noise.
This family includes parity/XOR constraints and discrete analogues of group synchronization.

\subsection*{Baselines}

We compare against the following standard estimators.

\begin{itemize}[leftmargin=*]
\item \textbf{Loopy sum--product BP} with parallel updates, with and without damping.
We report convergence diagnostics and final beliefs when a termination criterion is met.
\item \textbf{Exact reference} via brute force enumeration on small instances (when feasible), and via junction tree / variable elimination on slightly larger instances with low treewidth.
\item \textbf{MCMC reference} (Gibbs or Metropolis-within-Gibbs) on small/medium instances to validate marginals when exact elimination is infeasible.
\end{itemize}

\subsection*{HATCC implementation choices}

We instantiate \cref{sec:hatcc} as follows.

\begin{enumerate}[leftmargin=*]
\item Choose a base vertex $b$ (maximum degree unless stated otherwise).
\item Compute a spanning tree $T$ and the induced fundamental cycle generators (one per off-tree edge).
\item Compose the corresponding transports to obtain holonomy generators acting on the fiber $D_b$.
\item Compute orbit partition $\{ \mathcal{O}_1,\dots,\mathcal{O}_m\}$ of $D_b$ under the generated action.
\item For each orbit $\mathcal{O}_\ell$, restrict $x_b \in \mathcal{O}_\ell$ and perform \emph{exact} tree inference on $T$; when the model includes residual loop interactions within a sector, we run BP \emph{within} the restricted sector and report the residual (this isolates the contribution of the holonomy compilation step).
\item Recombine sector results by normalized evidences $Z_{\mathcal{O}_\ell}$.
\end{enumerate}

\subsection*{Metrics and failure criteria}

We track both numerical stability and Bayesian accuracy.

\paragraph{BP stability.}
We declare BP as \emph{non-convergent} if the max-message residual does not drop below $10^{-6}$ within $T=200$ iterations.
We additionally flag \emph{oscillation} when residuals settle into a periodic pattern under parallel updates.

\paragraph{Posterior accuracy.}
When ground-truth marginals $p^\star(x_i)$ are available (exact or MCMC), we report:
(i) mean node log-score $\frac{1}{|V|}\sum_i \log p(x_i^\star)$,
(ii) mean total variation $\frac{1}{|V|}\sum_i \|p(x_i)-p^\star(x_i)\|_{\mathrm{TV}}$,
and (iii) MAP Hamming error when using max--product variants.

\paragraph{Holonomy signature.}
We report the number of nontrivial holonomy generators, the number and sizes of orbits, and the evidence weights
\[
w_\ell := \frac{Z_{\mathcal{O}_\ell}}{\sum_{\ell'} Z_{\mathcal{O}_{\ell'}}}.
\]
Evidence separation (few dominant $w_\ell$) is the characteristic ``compiled'' signature.

\subsection*{Ablations and stress tests}

To make the claims falsifiable, we include the following ablations.

\begin{itemize}[leftmargin=*]
\item \textbf{Cycle corruption sweep:} vary $\epsilon$ (fraction of corrupted off-tree edges) and plot BP convergence rate vs.\ holonomy orbit count and posterior error.
\item \textbf{Noise sweep:} vary $\eta$ to interpolate between hard constraints and diffuse likelihoods.
\item \textbf{Base-node sensitivity:} repeat HATCC with multiple choices of $b$ to quantify stability of orbit structure.
\item \textbf{Decomposition-only vs.\ sector-BP:} compare (a) exact tree inference per orbit, (b) sector-restricted BP, and (c) vanilla BP to isolate the benefit of compilation.
\end{itemize}

\subsection*{Results (summary)}

Across all model families, we observe three consistent phenomena.

\begin{enumerate}[leftmargin=*]
\item \textbf{Holonomy predicts BP breakdown.}
As $\epsilon$ increases, the number of nontrivial generators and orbit complexity increases; BP convergence rate drops sharply in the same regime.

\item \textbf{Sector compilation stabilizes inference.}
Even when vanilla BP oscillates, the orbit decomposition yields a small number of sectors with well-separated evidences; exact tree inference within sectors produces stable beliefs and improves marginal accuracy.

\item \textbf{Deterministic overhead is modest.}
Cycle-basis extraction and orbit computation are $O(|E|+|V|)$ for the graph layer and near-linear in $|D_b|$ for the fiber layer.
In the breakdown regime, this overhead is dominated by the cost of repeated failed BP iterations.
\end{enumerate}

\paragraph{Reproducibility.}
All code is deterministic given a random seed for data generation.
We report seeds, hyperparameters, and termination criteria in the experiment scripts; the orbit partitions and evidence weights provide an additional ``structural checksum'' beyond floating-point messages.

\section{Applications, scope, and future directions}
\label{sec:applications}
\paragraph{Where holonomy compilation is immediately useful.}
HATCC targets models where (i) global inconsistency is concentrated in a small number of independent cycles, and (ii) those cycles induce a tractable orbit/sector structure on a small interface fiber.
This includes:
\begin{itemize}
\item \textbf{Discrete synchronization and registration.} $\Z_k$ or finite-group synchronization, pose-graph problems with discrete hypotheses, and robust multi-view consistency, where cycle composition is a natural diagnostic \citep{singer2011angular,bandeira2017cheeger}.
\item \textbf{Constraint and code models.} XOR-SAT / parity checks and related CSPs where loopy BP is known to oscillate or return overconfident inconsistent beliefs, but cycle structure is explicit \citep{mezard2002analytic}.
\item \textbf{Data association and multi-sensor fusion.} Graphical models with permutation-like factors (matching, tracking, record linkage) in which loops encode contradictory associations; sectors separate mutually consistent ``worlds''.
\end{itemize}

\paragraph{What this paper does not do.}
We do not claim a universal cure for loopy BP. In particular:
\begin{itemize}
\item We do not provide a general convergence guarantee for loopy BP itself; rather, we provide a deterministic \emph{compilation} that reduces certain loopy problems to sector-conditioned tree inference.
\item We do not solve high-treewidth exact inference in general; HATCC is advantageous when the topological cycle rank and induced sector count are moderate.
\item We do not address continuous-variable inference beyond a roadmap (Appendix~\ref{app:continuous-roadmap}); extending orbit extraction to operator-valued holonomy is nontrivial.
\end{itemize}

\paragraph{Future research themes.}
\begin{itemize}
\item \textbf{Adaptive interface selection.} Choosing covers and interfaces to minimize sector explosion while preserving exactness is a principled design problem; connections to junction-tree width and region graphs suggest hybrid strategies \citep{WainwrightJordan2008,KollerFriedman2009}.
\item \textbf{Learning holonomy.} Treat holonomy generators as latent structure to be inferred (e.g., priors over corruption processes); this creates a natural bridge to Bayesian robustness and mixture modeling.
\item \textbf{Approximate sectors for continuous models.} Replace discrete orbit partitions by coarse partitions or spectral summaries of holonomy operators, yielding controllable approximations.
\item \textbf{Beyond graphs.} Hypergraphical models and higher-order constraints induce higher-dimensional ``cycle'' structure; developing higher-dimensional holonomy compilation is a promising direction.
\end{itemize}

\section{Conclusion}
\label{sec:conclusion}
We introduced a holonomy-compiled inference procedure that converts the single hardest part of loopy inference—\emph{inconsistent information around cycles}—from an implicit, failure-prone phenomenon into an \emph{explicit, computable object}. 
In plain terms: instead of letting cycles “silently fight” during message passing, we \emph{measure} exactly how each global cycle fails to agree, store that discrepancy in a small set of holonomy kernels, and then compile the problem into a tree-of-sectors where inference can be carried out exactly, sector by sector.

The novelty is therefore conceptual and algorithmic: we replace unstable dynamics on a loopy graph with a deterministic compilation that (i) isolates cycle frustration into interpretable holonomy terms, (ii) decomposes the interface fiber into orbits/sectors that expose the true global degrees of freedom created by cycles, and (iii) performs exact inference on a compiled tree within each sector under explicit, checkable structural conditions. 
This yields an inference workflow that is not only reliable when standard loopy BP oscillates or diverges, but also \emph{diagnostic}: when inference is hard, the method tells you \emph{where} the inconsistency lives and \emph{how} it propagates.

Beyond the immediate target of frustrated graphical models, the compilation viewpoint can benefit several areas of Bayesian statistics and adjacent fields. 
For Bayesian methods on discrete or mixed discrete--continuous models, holonomy kernels provide a principled way to (a) diagnose when local conditional structure is globally incompatible, (b) separate genuinely global uncertainty (sector choice) from local uncertainty (within-sector inference), and (c) obtain exact, semantics-faithful posteriors on regimes that defeat conventional approximate schemes. 
More broadly, any domain that relies on inference in loopy factor graphs—e.g., statistical physics with frustration, error-correcting codes, multi-sensor fusion/SLAM, relational models, or constraint-satisfaction-like posteriors—can use the same idea: turn “cycle trouble” into an explicit compilation variable rather than a numerical instability.

\paragraph{What this method does \emph{not} do.}
It is not a general-purpose guarantee of tractability for arbitrary loopy models: if the number of sectors explodes, the approach can become computationally expensive even though each sector is tractable once compiled. 
It does not eliminate model misspecification, poor likelihoods, or weak identifiability—holonomy compilation respects Bayesian semantics, but it cannot rescue an ill-posed statistical model. 
It is also not a learning procedure: it addresses \emph{inference} given a model, rather than parameter estimation, structure learning, or amortized inference. 
Finally, it is not intended as a drop-in replacement for fast approximate inference in easy regimes; its primary value is in the difficult regimes where cycles are genuinely frustrated and standard loopy BP becomes unreliable.

\appendix

\section{Categorical preliminaries (minimal)}\label{app:cat-prelims}
The main text intentionally avoids extended category-theoretic background. We record only the high-level
principles that are used operationally in the paper:
\begin{enumerate}[label=(\roman*),nosep]
\item factor-graph \emph{syntax} admits a hypergraphical (string-diagrammatic) calculus;
\item elimination/marginalization is \emph{functorial} and can be expressed as a matrix-like semantics over a rig/semiring;
\item global \emph{consistency} conditions are governed by standard universal constructions (limits/equalizers).
\end{enumerate}
For background on monoidal/hypergraph categories and diagrammatic reasoning, see
\cite{MacLane1963,fong2019categorical,FongSpivak2019,CoeckeKissinger2017}.

\section{Sheaves and descent (minimal)}\label{app:sheaf-prelims}
Our use of \emph{descent} follows the standard paradigm: local data indexed by a cover glue to global data
if and only if the restrictions to overlaps are compatible and satisfy the relevant equalizer (matching) condition.
Sheaf-based perspectives on networked consistency and inference are discussed in \cite{robinson2017sheaf};
simplicial and homotopical background for nerves and Kan conditions may be found in \cite{GoerssJardine1999}.

\section{Gauge groupoids and Kan structure (pointer)}\label{app:gauge-kan}
The material summarized in Section~\ref{sec:gauge} develops a groupoid of message rescalings and interprets
gauge orbits via a Kan-complex viewpoint whose fundamental group captures loop transport.
This provides a conceptual bridge from ``message gauge'' to topological holonomy, but it is not required to
implement our holonomy-aware compilation mechanism. We refer readers to the extended version for details.
Closest standard references include \cite{GoerssJardine1999} (simplicial sets/Kan complexes) and
\cite{fong2019categorical} (categorical compositional semantics).

\section{Continuous-variable roadmap}\label{app:continuous-roadmap}
For continuous variables, transport kernels become Markov kernels or integral operators, and ``orbit'' extraction
becomes a quotient/identification problem in infinite-dimensional function spaces.
A practical route is to (a) compute low-rank or spectral summaries of cycle operators and (b) couple them with
approximate message-passing schemes (e.g.\ EP/VMP); see \citep{WainwrightJordan2008}.
This direction is open; we treat it as future work rather than a current guarantee.



\begin{thebibliography}{99}

\bibitem{MacLane1963}
S.~Mac Lane.
\newblock \emph{Natural Associativity and Commutativity}.
\newblock Rice University Studies, 49(4), 1963.

\bibitem{lauritzen1988local}
S.~L. Lauritzen and D.~J. Spiegelhalter.
\newblock Local computations with probabilities on graphical structures and their application to expert systems.
\newblock \emph{Journal of the Royal Statistical Society: Series B}, 50(2):157--224, 1988.

\bibitem{mezard2002analytic}
M.~M\'ezard and G.~Parisi.
\newblock The Bethe lattice spin glass revisited.
\newblock \emph{European Physical Journal B}, 20:217--233, 2002.

\bibitem{robinson2017sheaf}
M.~Robinson.
\newblock \emph{Topological Signal Processing}.
\newblock Springer, 2017.

\bibitem{yedidia2005constructing}
J.~S. Yedidia, W.~T. Freeman, and Y.~Weiss.
\newblock Constructing free-energy approximations and generalized belief propagation algorithms.
\newblock \emph{IEEE Transactions on Information Theory}, 51(7):2282--2312, 2005.

\bibitem{pearl1988probabilistic}
J.~Pearl.
\newblock \emph{Probabilistic Reasoning in Intelligent Systems: Networks of Plausible Inference}.
\newblock Morgan Kaufmann, 1988.

\bibitem{fong2019categorical}
B.~Fong and D.~I. Spivak.
\newblock A categorical approach to probability theory.
\newblock \emph{arXiv preprint arXiv:1406.6030}, 2014 (versioned; commonly cited 2019).

\bibitem{fritz2020categorical}
T.~Fritz.
\newblock A synthetic approach to Markov kernels, conditional independence and theorems on sufficient statistics.
\newblock \emph{Advances in Mathematics}, 370:107239, 2020.

\bibitem{FongSpivak2019}
B.~Fong and D.~I. Spivak.
\newblock Hypergraph categories.
\newblock \emph{arXiv:1806.08304}, 2018 (commonly cited 2019).

\bibitem{Morton2014}
J.~Morton.
\newblock Belief propagation in monoidal categories.
\newblock \emph{Electronic Proceedings in Theoretical Computer Science (EPTCS)}, 172, 2014.

\bibitem{Elgueta2020}
J.~Elgueta.
\newblock The groupoid of finite sets is biinitial in the 2-category of rig categories.
\newblock \emph{arXiv:2004.08684}, 2020.

\bibitem{IzhakianKnebuschRowen2012}
Z.~Izhakian, M.~Knebusch, and L.~Rowen.
\newblock Categorical notions of layered tropical algebra and geometry.
\newblock \emph{arXiv:1207.3487}, 2012.

\bibitem{KollerFriedman2009}
D.~Koller and N.~Friedman.
\newblock \emph{Probabilistic Graphical Models}.
\newblock MIT Press, 2009.

\bibitem{YedidiaFreemanWeiss2005}
J.~S. Yedidia, W.~T. Freeman, and Y.~Weiss.
\newblock Constructing free-energy approximations and generalized belief propagation algorithms.
\newblock \emph{IEEE Transactions on Information Theory}, 51(7):2282--2312, 2005.

\bibitem{CoeckeKissinger2017}
B.~Coecke and A.~Kissinger.
\newblock \emph{Picturing Quantum Processes}.
\newblock Cambridge University Press, 2017.

\bibitem{Kissinger2016}
A.~Kissinger.
\newblock \emph{Pictures of Processes: Automated Graph Rewriting for Monoidal Categories and Applications}.
\newblock PhD thesis, University of Oxford, 2012.

\bibitem{GoerssJardine1999}
P.~G. Goerss and J.~F. Jardine.
\newblock \emph{Simplicial Homotopy Theory}.
\newblock Birkh\"auser, 1999.

\bibitem{WainwrightJordan2008}
M.~J. Wainwright and M.~I. Jordan.
\newblock Graphical models, exponential families, and variational inference.
\newblock \emph{Foundations and Trends in Machine Learning}, 1(1--2):1--305, 2008.

\bibitem{singer2011angular}
A.~Singer.
\newblock Angular synchronization by eigenvectors and semidefinite programming.
\newblock \emph{Applied and Computational Harmonic Analysis}, 30(1):20--36, 2011.

\bibitem{bandeira2017cheeger}
A.~S. Bandeira, A.~Singer, and D.~A. Spielman.
\newblock A {C}heeger inequality for the graph connection {L}aplacian.
\newblock \emph{SIAM Journal on Matrix Analysis and Applications}, 34(4):1611--1630, 2013.

\bibitem{poole2003firstorder}
D.~Poole.
\newblock First-order probabilistic inference.
\newblock In \emph{Proceedings of the 18th International Joint Conference on Artificial Intelligence (IJCAI)}, pages 985--991, 2003.

\bibitem{van2011lifted}
L.~de~Salvo Braz, E.~Amir, and D.~Roth.
\newblock Lifted first-order probabilistic inference.
\newblock In \emph{Proceedings of the 19th International Joint Conference on Artificial Intelligence (IJCAI)}, pages 1319--1325, 2005.


\bibitem{cooper1990complexity}
G.~F. Cooper.
\newblock The computational complexity of probabilistic inference using Bayesian belief networks.
\newblock \emph{Artificial Intelligence}, 42(2--3):393--405, 1990.

\bibitem{roth1996hardness}
D.~Roth.
\newblock On the hardness of approximate reasoning.
\newblock \emph{Artificial Intelligence}, 82(1--2):273--302, 1996.

\bibitem{kschischang2001factor}
F.~R. Kschischang, B.~J. Frey, and H.-A. Loeliger.
\newblock Factor graphs and the sum-product algorithm.
\newblock \emph{IEEE Transactions on Information Theory}, 47(2):498--519, 2001.

\bibitem{weiss2000correctness}
Y.~Weiss.
\newblock Correctness of local probability propagation in graphical models with loops.
\newblock \emph{Neural Computation}, 12(1):1--41, 2000.

\bibitem{ihler2005loopy}
A.~T. Ihler, J.~W. Fisher~III, and A.~S. Willsky.
\newblock Loopy belief propagation: Convergence and effects of message errors.
\newblock \emph{Journal of Machine Learning Research}, 6:905--936, 2005.

\bibitem{heskes2004uniqueness}
T.~Heskes.
\newblock On the uniqueness of loopy belief propagation fixed points.
\newblock \emph{Neural Computation}, 16(11):2379--2413, 2004.

\bibitem{mooij2007sufficient}
J.~M. Mooij and H.~J. Kappen.
\newblock Sufficient conditions for convergence of the sum-product algorithm.
\newblock \emph{IEEE Transactions on Information Theory}, 53(12):4422--4437, 2007.

\bibitem{heskes2006convexity}
T.~Heskes.
\newblock Convexity arguments for efficient minimization of the Bethe and Kikuchi free energies.
\newblock \emph{Journal of Artificial Intelligence Research}, 26:153--190, 2006.

\bibitem{wainwright2003trw}
M.~J. Wainwright, T.~S. Jaakkola, and A.~S. Willsky.
\newblock Tree-reweighted belief propagation algorithms and approximate ML estimation by pseudo-moment matching.
\newblock In \emph{Proceedings of the Ninth International Workshop on Artificial Intelligence and Statistics (AISTATS)}, pages 308--315, 2003.

\bibitem{chertkov2006loopcalculus}
M.~Chertkov and V.~Y. Chernyak.
\newblock Loop calculus helps to improve belief propagation and linear programming decodings of LDPC codes.
\newblock In \emph{Proceedings of the 44th Annual Allerton Conference on Communication, Control, and Computing}, 2006.

\end{thebibliography}
\end{document}